\documentclass{article}
\usepackage{authblk}
\usepackage{amssymb}
\usepackage{subfig}
\usepackage{amsthm}
\usepackage{amsmath}
\usepackage{tikz}
\usepackage{multirow, multicol}
\usepackage{url}
\usepackage[figuresright]{rotating}
\usepackage[lined,boxed,commentsnumbered,ruled]{algorithm2e}
\usepackage{natbib}
\usepackage{abstract}
\usepackage{graphicx} 
\usepackage{lipsum}

\newtheorem{definition}{Definition}
\newtheorem{theorem}{Theorem}
\newtheorem{example}{Example}
\newtheorem{lemma}{Lemma}
\newtheorem{corollary}{Corollary}

\newcommand{\sep}{, }

\begin{document}

\title{Local Causal Discovery with Background Knowledge}

\author[1]{Qingyuan Zheng}
\author[2]{Yue Liu\thanks{Corresponding authors.}}
\author[1]{Yangbo He\thanks{Corresponding authors.}}

\affil[1]{LMAM, School of Mathematical Sciences, LMEQF, and Center of Statistical Science, Peking University.}
\affil[2]{Center for Applied Statistics and School of Statistics, Renmin University of China, Beijing, China.}

\date{}
\maketitle

\begin{abstract}
Causality plays a pivotal role in various fields of study. Based on the framework of causal graphical models, previous works have proposed identifying whether a variable is a cause or non-cause of a target in every Markov equivalent graph solely by learning a local structure. However, the presence of prior knowledge, often represented as a partially known causal graph, is common in many causal modeling applications. Leveraging this prior knowledge allows for the further identification of causal relationships. In this paper, we first propose a method for learning the local structure using all types of causal background knowledge, including direct causal information, non-ancestral information and ancestral information. Then we introduce criteria for identifying causal relationships based solely on the local structure in the presence of prior knowledge. We also apply out method to fair machine learning, and experiments involving local structure learning, causal relationship identification, and fair machine learning demonstrate that our method is both effective and efficient.
\end{abstract}


\section{Introduction}
\label{sec:intro}
Causality is crucial in artificial intelligence research. 
Recent work emphasizes that causal models should be applied in artificial intelligence systems to support explanation and understanding~\citep{lake2017building}. 
Specifically, to achieve fairness in machine learning algorithms, some approaches propose excluding all variables influenced by the sensitive variable from the set of predictors~\citep{wu2019counterfactual, zuo2022counterfactual}.
This highlights the importance of understanding the causal relationships between other variables and the sensitive variable. 

Randomized experiments are the gold standard for identifying causal relationships~\citep{robins2000marginal}.
However, randomized trials may be impractical due to cost,  ethical considerations, or infeasibility for non-manipulable variables.
As a result, inferring causal relationships from observational data has gained increasing attention~\citep{cooper1997simple, cox2018modernizing}.
Causal relationships among multiple variables in observational data can be compactly represented using causal directed acyclic graphs (DAGs)~\citep{pearl1995causal, spirtes2000causation, geng2019evaluation}.
In a causal DAG, a treatment variable is a cause of an outcome variable if and only if there exists a directed path from treatment variable to outcome variable~\citep{pearl2009causality}.
Therefore, if the underlying DAG is known, causal relationships between any two variables can be judged directly.
However, without additional distributional assumptions, we can only identify a Markov equivalence class (MEC) from observational data~\citep{chickering2002learning}. An MEC is a set of DAGs where each DAG encodes the same conditional independencies.

The DAGs in the same MEC may encode different causal relationships.
For example, in Figure~\ref{fig:intro}, $\mathcal{G}_2$ and $\mathcal{G}_3$ are two Markov equivalent DAGs. However, $A$ is a cause of $B$ in $\mathcal{G}_3$, whereas $A$ is not a cause of $B$ in $\mathcal{G}_2$.
Typically, we define the causal relationship between two nodes $X$ and $Y$ within a set of DAGs, denoted as $\mathcal{S}$, as follows~\citep{fang2022local, zuo2022counterfactual}:
\begin{itemize}
    \item $X$ is a \textit{definite cause} of $Y$ in $\mathcal{S}$ if $X$ is a cause of $Y$ in every DAG $\mathcal{G}\in\mathcal{S}$;
    \item $X$ is a \textit{definite non-cause} of $Y$ in $\mathcal{S}$ if $X$ is not a cause of $Y$ in every DAG $\mathcal{G}\in\mathcal{S}$;
    \item $X$ is a \textit{possible cause} of $Y$ in $\mathcal{S}$ otherwise; that is, $X$ is a cause of $Y$ in some DAGs in $\mathcal{S}$, and $X$ is not a cause of $Y$ in others.
\end{itemize}
Respectively, $Y$ is called a definite descendant, definite non-descendant, or possible descendant of $X$.

\begin{figure}[t]
\centering
\subfloat[$\mathcal{G}^*$]
{
\label{fig_intro:subfig1}\includegraphics[width=0.15\textwidth]{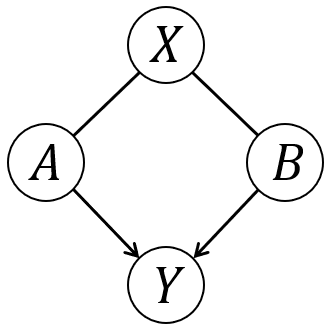}
}
\subfloat[$\mathcal{G}_1$]
{
\label{fig_intro:subfig2}\includegraphics[width=0.15\textwidth]{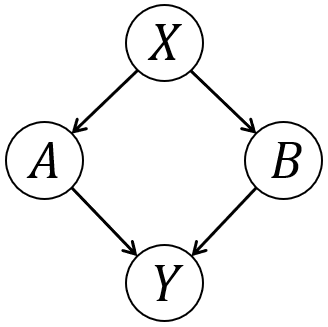}
}
\subfloat[$\mathcal{G}_2$]
{
\label{fig_intro:subfig3}\includegraphics[width=0.15\textwidth]{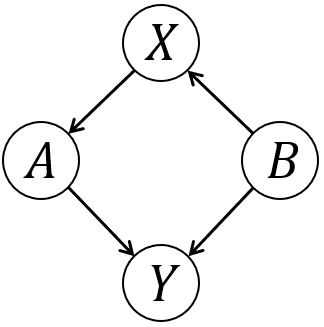}
}
\subfloat[$\mathcal{G}_3$]
{
\label{fig_intro:subfig4}\includegraphics[width=0.15\textwidth]{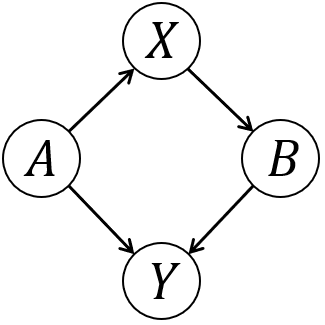}
}
\caption{An example of three Markov equivalent DAGs $\mathcal{G}_1$-$\mathcal{G}_3$ and their corresponding CPDAG $\mathcal{G}^*.$ These Markov equivalent DAGs have the same edges despite their orientation and share the same v-structure $A\to Y\leftarrow B$. The undirected edges $A-X-B$ in $\mathcal{G}^*$ indicate that these edges may have different orientations in the Markov equivalent DAGs.} 
\label{fig:intro}
\end{figure}

We can see that if $|\mathcal{S}|=1,$ $X$ is either a definite cause or a definite non-cause of $Y$ for every pair $(X,Y),$ and causal relationships are clearly identified. 
When $|\mathcal{S}|$ becomes larger, the causal relationships become more uncertain.
Therefore, if we can restrict the underlying DAG to a small set, we can obtain a better understanding of the causal relationships.

In practice, there often exists some prior knowledge about the causal structure, which encodes causal relationship among observed variables~\citep{sinha2021using}.
For example, in clinical trials, prior knowledge is often available from previous research or expert opinions~\citep{hasan2022kcrl}.
In algorithm fairness research, some works assume that the sensitive attribute set is closed under ancestral relationships~\citep{kusner2017counterfactual, zuo2022counterfactual}. By applying this prior knowledge, the MEC can be refined to a smaller set of DAGs, enabling more precise inference about causal relationships.

An MEC can be uniquely represented by a complete partially directed acyclic graph (CPDAG)~\citep{meek1995causal, andersson1997characterization, spirtes2000causation, chickering2002learning}. 
By applying some type of prior knowledge, such as pairwise direct causal relationships, tiered orderings, or non-ancestral knowledge, the refined set of MEC can be represented by a maximally partially directed acyclic graph (MPDAG)~\citep{meek1995causal,fang2020ida,fang2022representation, bang2023we}\footnote{Note that, with some type of prior knowledge such as ancestral relations, the refined set cannot be represented only by an MPDAG, but an MPDAG with a minimal residual set of direct causal clauses~\citep{fang2022representation}. We discuss this type of prior knowledge in \ref{append:algorithm}.}.
In this way, prior knowledge can be seen as a partially known part of the underlying DAG.
For example, Figure~\ref{fig:intro} shows three Markov equivalent DAGs, $\mathcal{G}_1$-$\mathcal{G}_3,$ and their corresponding CPDAG, $\mathcal{G}^*.$ The MEC is $[\mathcal{G}^*]=\{\mathcal{G}_1,\mathcal{G}_2,\mathcal{G}_3\}.$ Since $X$ is a cause of $Y$ in all DAGs, $X$ is a definite cause of $Y$ in $[\mathcal{G}^*].$ On the other hand, $A$ is a cause of $B$ in $\mathcal{G}_3$ but not in $\mathcal{G}_1$ and $\mathcal{G}_2.$ If $A\to X$ is known as prior knowledge, the MEC is restricted to $\{\mathcal{G}_3\},$ making $A$ a definite cause of $B$.

In recent years, several studies have investigated the issue of identifying causal relationships~\citep{roumpelaki2016marginal,mooij2020constraint,fang2022local,zuo2022counterfactual}. Among them,~\citet{fang2022local} proposed a method for determining causal relationships in CPDAG through a local learning approach. However, this method cannot be directly applied to equivalence classes constrained by background knowledge. \citet{zuo2022counterfactual} proposed a method for determining causal relationships in MPDAGs. Nevertheless, the performance of this method is constrained by the estimation errors inherent in learning the entire MPDAG.

In this paper, we consider the identification of causal relationships in the refined set of MEC represented by an MPDAG.
We assume no selection bias or presence of unmeasured confounders, consistent with other causal discovery methods~\citep{liu2020local, fang2022local, zuo2022counterfactual}.
Unlike existing method~\citep{zuo2022counterfactual}, we propose a local method for identifying causal relationships in an MPDAG. This approach requires only the local structure around treatment variable and some additional conditional independencies for identifying of causal relationships.
To achieve this, we introduce an algorithm for learning the local structure in an MPDAG, based on the MB-by-MB algorithm~\citep{wang2014discovering, xie2024local}.

The remainder of this paper is organized as follows. 
After reviewing existing methods in Section~\ref{sec:related works} and introducing basic concepts in Section~\ref{sec:preliminaries}, we propose a sound and complete algorithm in Section~\ref{sec:mb by mb in mpdag} for learning the local structure around the treatment variable in an MPDAG.
In Section~\ref{sec:local theorems}, we introduce criteria for judging causal relationships between the treatment and outcome given the local structure of the treatment variable.
In Section~\ref{sec:experiments}, we test our method and compare it with existing methods by experiments involving local structure learning, causal relationship identification, and fair machine learning.
Additionally, We apply our method to a real-world dataset in Section~\ref{sec:application}.

Our main contributions are summarized as follows:
\begin{itemize}
    \item We propose a sound and complete algorithm for learning the local structure around any given variable in an MPDAG.
    \item We provide sufficient and necessary conditions for causal relationships in the equivalence class represented by an MPDAG. These conditions can be assessed solely based on the local structure of the treatment variable and some additional conditional independencies.
    \item Based on these conditions, we propose a sound and complete algorithm for identifying all definite descendants, definite non-descendants, and possible descendants of the treatment variable locally in an MPDAG.
    \item Our experiments demonstrate the efficiency and efficacy of our algorithm. It shows superior performance in experiments related to local structure learning, causal relationship identification, and fair machine learning. We also apply our method to a real-world dataset.
\end{itemize}

\section{Related Works and Methods}
\label{sec:related works}
\textbf{Local structure learning.} Local structure learning refers to finding the parent-children (PC) set of a target variable and distinguishing between the parent and child identities of nodes in the PC set.
Multiple algorithms have been proposed for learning local structure in a CPDAG~\citep{yin2008partial, wang2014discovering, gao2015local, liu2020local, gupta2023local}.
Among them, the PCD-by-PCD~\citep{yin2008partial}, MB-by-MB~\citep{wang2014discovering} and CMB algorithms~\citep{gao2015local} sequentially find a local set, such as the PC set or the Markov blanket (MB)~\citep{tsamardinos2003algorithms, tsamardinos2003time, aliferis2003hiton, pena2007towards}, for nodes connected to the target node. They then find v-structures in this local set to orient edges in the local structure of the target node. This process continues until all edges around the target node are oriented or it is determined that further orientation is not possible. 
The LDECC algorithm~\citep{gupta2023local} incorporates additional conditional independence tests into the PC algorithm, which may help in learning the local structure of the target variable more efficiently.
Recently,~\citet{liu2020local} extended the MB-by-MB algorithm to learn the chain component containing the target node and the directed edges around it in a CPDAG.

To obtain the local structure around the target node in an MPDAG, we can first use the aforementioned algorithms to learn the local structure in the corresponding CPDAG, then add background knowledge and use Meek's rules~\citep{meek1995causal} for edge orientation.
However, since Meek's rules are applied on the entire graph, the completeness of such a method is not guaranteed.
Additionally, the background knowledge is not exploited in this local structure learning process.
In this paper, we propose a sound and complete algorithm based on the MB-by-MB algorithm to learn the local structure in an MPDAG, which utilizes the background knowledge during the learning process.

\textbf{Causal relationship identification.}  To judge causal relationship between a treatment variable $X$ and an outcome variable $Y$ in a set $\mathcal{S}$ of DAGs represented by an MPDAG $\mathcal{G}^*,$ an intuitive method is to enumerate all DAGs in $\mathcal{S}$~\citep{wienobst2023efficient}, and judge the causal relationship in each DAG.
However, this is generally computationally inefficient since the size of $\mathcal{S}$ is large in most cases~\citep{he2015counting, sharma2023counting}.

Another way involves estimating the causal effect of $X$ on $Y.$ 
For almost every distribution faithful to the underlying DAG, $X$ is a cause of $Y$ if and only if the causal effect of $X$ on $Y$ is not zero.
However, equivalent DAGs may encode different causal effects.
\citet{perkovic2020identifying} proposed a sufficient and necessary graphical condition for checking whether the causal effect is identifiable in an MPDAG, i.e., all DAGs represented by an MPDAG encode the same causal effect.
If the condition is met, we can estimate the causal effect using the causal identification formula given by~\citet{perkovic2020identifying} and judge the causal relation.

If the causal effect is not identified, an alternative approach involves enumerating all possible causal effects. $X$ is a definite cause (or definite non-cause) of $Y$ if and only if all possible causal effects of $X$ on $Y$ are non-zero (or zero).
This approach can be implemented through the intervention calculus when the DAG is absent (IDA) method without listing all equivalent DAGs~\citep{maathuis2009estimating, perkovic2017interpreting, nandy2017estimating, fang2020ida, liu2020collapsible, liu2020local, witte2020efficient, guo2021minimal}.
Intuitively, the IDA method enumerates all possible local structures around the treatment variable.
Each local structure represents a set of DAGs, which encode the same causal effect.
However, the computational complexity of IDA methods is exponential to the number of undirected edges incident to $X$ on a possibly causal path from $X$ to $Y$~\citep{guo2021minimal}, making it time-consuming in the worse case.

\citet{fang2022local} proposed a local method for identifying causal relations in a CPDAG. Their method focuses on learning a compact subgraph surrounding the target node and utilizes conditional independence tests to infer causal relations. However, directly applying their findings to MPDAGs is limited, as a CPDAG can be considered a special case of an MPDAG when the prior knowledge does not offer additional insights beyond observational data.

Recently,~\citet{zuo2022counterfactual} proposed a sufficient and necessary graphical condition for judging causal relations in an MPDAG.
However, judging this condition requires enumerating paths that satisfy certain conditions in an MPDAG.
Therefore, their method relies on a fully learned MPDAG.
In contrast, we propose a fully local method for identifying causal relations in an MPDAG.

\section{Preliminaries}
\label{sec:preliminaries}
We begin by clarifying our setting and listing existing results on identifying causal relations.
We state key concepts here and leave other graph terminologies and formal definitions in~\ref{append:graph}.

\subsection{Basic concepts}
Let $\mathcal{G}=(\mathbf{V},\mathbf{E})$ be a graph with node set $\mathbf{V}$ and edge set $\mathbf{E}.$
A graph $\mathcal{G}$ is a Directed Acyclic Graph (DAG) if all its edges are directed and it contains no cycles.
A path $p=\langle V_1,\dots,V_k \rangle$ is a sequence of adjacent nodes in $\mathcal{G}$.
We say that $p$ is causal if $V_i\to V_{i+1}$ for every $i.$
For two distinct nodes $X$ and $Y,$ we say $X$ is an ancestor (or cause) of $Y$ and $Y$ is a descendant of $X$, denoted by $X\in an(Y,\mathcal{G})$ and $Y\in de(X,\mathcal{G}),$ if there is a causal path from $X$ to $Y.$
By convention, $X\in an(X,\mathcal{G})$ and $X\in de(X,\mathcal{G}).$
For a non-endpoint $V_i$ on a path $p,$ $V_i$ is called a collider on $p$ if $V_{i-1}\to V_i\leftarrow V_{i+1};$ otherwise, $V_i$ is a non-collider on $p.$
A path $p$ is d-connected (or open) given a set of nodes $\mathbf{Z}$ if $\mathbf{Z}$ does not contain any of its non-colliders, and for each collider $V_i$ on $p,$ $de(V_i,\mathcal{G})\cap \mathbf{Z}\neq \emptyset.$
Otherwise, $p$ is blocked by $\mathbf{Z}.$
For distinct node sets $\mathbf{X},\mathbf{Y},\mathbf{Z},$ we say that $\mathbf{X}$ and $\mathbf{Y}$ are d-separated by $\mathbf{Z}$ if all paths from any $X\in \mathbf{X}$ to $Y\in \mathbf{Y}$ are blocked by $\mathbf{Z}.$
Let $P$ be a distribution over $\mathbf{V}.$ $P$ is Markovian and faithful with respect to $\mathcal{G}$ if conditional independencies in $P$ are equivalent to d-separations in $\mathcal{G}.$
We use $\mathbf{X}\perp \mathbf{Y} \mid \mathbf{Z}$ to denote that (i) $\mathbf{X},\mathbf{Y}$ are d-separated by $\mathbf{Z}$ in $\mathcal{G}$, and (ii) $\mathbf{X}$ and $\mathbf{Y}$ are conditionally independent given $\mathbf{Z}$ in a distribution $P$ that is Markovian and faithful with respect to $\mathcal{G}.$

For a DAG $\mathcal{G},$ its skeleton is the undirected graph obtained by replacing all directed edges in $\mathcal{G}$ with undirected edges while preserving the adjacency relations.
A v-structure in $\mathcal{G}$ is a triple $(X_{i-1},X_i,X_{i+1})$ with $X_{i-1}\to X_i\leftarrow X_{i+1}$ in $\mathcal{G}$ where $X_{i-1}$ and $X_{i+1}$ are not adjacent.
Two DAGs $\mathcal{G}_1,\mathcal{G}_2$ are Markov equivalent if they encode the same d-separation relations.
It has been shown that two DAGs are Markov equivalent if and only if they have the same skeleton and the same v-structures~\citep{verma1990equivalence}.
Let $\left[\mathcal{G}\right]$ denote the set of all DAGs that are Markov equivalent to $\mathcal{G},$ also known as the Markov equivalence class (MEC) of $\mathcal{G}.$
The MEC $\left[\mathcal{G}\right]$ can be uniquely represented by a complete partially directed acyclic graph (CPDAG) $\mathcal{C},$ which shares the same skeleton with $\mathcal{G},$ and any edge in $\mathcal{C}$ is directed if and only if it has the same direction in all DAGs in $\left[\mathcal{G}\right]$~\citep{verma1990equivalence}.
We also use $\left[\mathcal{C}\right]$ to denote the MEC represented by $\mathcal{C}.$

Let $\mathcal{B}$ be a set of background knowledge.
In this paper, we consider \textit{direct causal information}, as discussed by~\citet{zuo2022counterfactual}.
In this way, the background knowledge in $\mathcal{B}$ is of the form $X\to Y,$ meaning that $X$ is a direct cause of $Y.$
By orienting undirected edges in a CPDAG $\mathcal{C}$ according with $\mathcal{B}$ and applying Meek's rules (\citep{meek1995causal}, also see~\ref{append:graph}), we obtain a partially directed acyclic graph $\mathcal{G}^*,$ called the maximally partially directed acyclic graph (MPDAG) of $[\mathcal{C},\mathcal{B}]$~\citep{fang2022representation}.
Let $\left[\mathcal{G}^*\right]$ denote the set of DAGs in $\left[\mathcal{C}\right]$ which is consistent with $\mathcal{B}.$
It has been shown that an edge in $\mathcal{G}^*$ is directed if and only if it has the same orientation in every DAG in $\left[\mathcal{G}^*\right]$~\citep{meek1995causal}.
A path $p=\langle V_1,\dots,V_k\rangle$ is b-possibly causal in $\mathcal{G}^*$ if no edge $V_i\leftarrow V_j$ exists in $\mathcal{G}^*$ for $1\le i<j\le k.$
A path $p$ is chordless if any of its two non-consecutive nodes are not adjacent.
For distinct nodes $X$ and $Y$ in $\mathcal{G}^*,$ the critical set of $X$ with respect to $Y$ in $\mathcal{G}^*$ consists of all adjacent vertices of $X$ lying on at least one chordless b-possibly causal path from $X$ to $Y.$

In a partially directed graph $\mathcal{G}=(\mathbf{V},\mathbf{E}),$ which could be a DAG, CPDAG or MPDAG, we use $ch(X,\mathcal{G})=\{Y\in \mathbf{V}|X\to Y\mathrm{\ in\ }\mathcal{G}\},pa(X,\mathcal{G})=\{Y\in\mathbf{V}|Y\to X\mathrm{\ in\ }\mathcal{G}\},sib(X,\mathcal{G})=\{Y\in\mathbf{V}|X- Y\mathrm{\ in\ }\mathcal{G}\}$ to denote the children, parents and siblings of a node $X$ in $\mathcal{G}.$
Let $adj(X,\mathcal{G})=pa(X,\mathcal{G})\cup ch(X,\mathcal{G})\cup sib(X,\mathcal{G})$ denote the set of nodes adjacent with $X$ in $\mathcal{G}.$
For a set of nodes $\mathbf{W}\subseteq \mathbf{V},$ let $\mathcal{G}_\mathbf{W}$ denote the induced subgraph of $\mathcal{G}$ over $\mathbf{W},$ containing all and only the edges in $\mathcal{G}$ that are between nodes in $\mathbf{W}.$
For a set of nodes $\mathbf{Q}\subseteq \mathbf{V},$ we say $\mathbf{Q}$ is a clique if every pair of nodes in $\mathbf{Q}$ are adjacent in $G.$ A clique $\mathbf{Q}$ is maximal if for any clique $\mathbf{Q}',$ $\mathbf{Q}\subseteq \mathbf{Q}'$ implies $\mathbf{Q}=\mathbf{Q}'.$

\subsection{Problem formulation and existing result}

Let $\mathcal{D}$ be a set of data over variables $\mathbf{V},$ whose distribution is Markovian and faithful with respect to an underlying DAG $\mathcal{G}=\{\mathbf{V},\mathbf{E}\}.$
Let $\mathcal{B}\subseteq \mathbf{E}$ be a set of background knowledge.
Based on $\mathcal{D}$ and $\mathcal{B},$ there exists a unique MPDAG $\mathcal{G}^*$ such that $\mathcal{G}\in \left[\mathcal{G}^*\right].$
Consider a target variable $X\in\mathbf{V}$.
As defined in Section~\ref{sec:intro}, for any $Y\in\mathbf{V},$ $Y$ is a definite descendant (definite non-descendant, possible descendant) of $X$ in $\left[\mathcal{G}^*\right]$ if $Y$ is a descendant of $X$ in all (none, some) DAGs within $\left[\mathcal{G}^*\right].$
Our objective is to identify all definite descendants, definite non-descendants and possible descendants of $X$ in $\left[\mathcal{G}^*\right].$

If the true MPDAG $\mathcal{G}^*$ can be learned from $\mathcal{D}$ and $\mathcal{B}$, it has been shown that the causal relationship between any pair of nodes $(X,Y)$ can be judged by examining the critical set of $X$ with respect to $Y$ in $\mathcal{G}^*.$

\begin{lemma}
\label{lemma:zuo}
(\citep{zuo2022counterfactual}, Theorem 4.5) Let $X$ and $Y$ be two distinct vertices in an MPDAG $\mathcal{G}^*,$ and $\mathbf{C}$ be the critical set of $X$ with respect to $Y$ in $\mathcal{G}^*.$
Then $Y$ is a definite descendant of $X$ in $\left[\mathcal{G}^*\right]$ if and only if either $\mathbf{C}\cap ch(X,\mathcal{G}^*) \neq \emptyset,$ or $\mathbf{C}$ is non-empty and induces an incomplete subgraph of $\mathcal{G}^*.$
\end{lemma}

However, learning the entire MPDAG is time-consuming and may introduce more errors~\citep{chickering1996learning, chickering2004large}.
In this paper, we propose a method to judge the causal relationship between $X$ and $Y$ in $\left[\mathcal{G}^*\right]$ using only the local structure around $X,$ specifically the tuple $\left(pa(X,\mathcal{G}^*),ch(X,\mathcal{G}^*),sib(X,\mathcal{G}^*)\right)$ and the skeleton of $\mathcal{G}^*_{sib(X,\mathcal{G}^*)}$, thus avoiding the extra cost and potential errors associated with learning the entire MPDAG.

\section{Local causal structure learning with background knowledge}
\label{sec:mb by mb in mpdag}
Before presenting the theorems for identifying causal relations, we must address an important question: can we learn the local structure around $X$ without learning the entire MPDAG? In this section, we provide an affirmative answer to this question. In Section~\ref{subsec:local learning direct cause}, we first provide an algorithm to learn the local structure when background knowledge consists of only direct causal information. After that, in Section~\ref{subsec:local learn non-ans}, we incorporate non-ancestral information and ancestral information, and propose algorithms learning the local structure when background knowledge consists of direct causal information and non-ancestral information, or when all three types of background knowledge are available.

\subsection{Local structure learning with direct causal information}
\label{subsec:local learning direct cause}

We first introduce the MB-by-MB in MPDAG algorithm to find the local structure when background knowledge consists of only direct causal information. This algorithm is an extension of the MB-by-MB algorithm, which discovers the local structure of $X$ in the corresponding CPDAG~\citep{wang2014discovering, liu2020local, xie2024local}.

The main procedure of the algorithm in summarized in Algorithm~\ref{alg:mb-by-mb in MPDAG}.
Let $X$ be the target node in an MPDAG $\mathcal{G}^*.$
We want to learn $pa(X,\mathcal{G}^*),ch(X,\mathcal{G}^*),sib(X,\mathcal{G}^*)$ and the skeleton of $\mathcal{G}^*_{sib(X,\mathcal{G}^*)}$ from observational data $\mathcal{D}$ and background knowledge $\mathcal{B}$ consisting of direct causal information.
In the learning procedure, we maintain a set of conditional independencies called IndSet.
We use $(a,b,S_{ab})\in \mathrm{IndSet}$ for two distinct nodes $a,b$ and a node set $S_{ab}$ to denote that $a,b$ are conditional independent given $S_{ab}.$

\begin{algorithm}[t]
\LinesNumbered
\caption{MB-by-MB in MPDAG: find local causal structure of a certain node in MPDAG $\mathcal{G}^*$.}
\label{alg:mb-by-mb in MPDAG}
\SetKwInOut{Input}{Input}
\SetKwInOut{Output}{Output}
\small

\Input{A target $X$, observational data $\mathcal{D}$, background knowledge $\mathcal{B}$.}
\Output{A PDAG $G$ containing the local structure of $X$.}
Initialize DoneList $=\emptyset$ (list of nodes whose MBs have been found); \\
Initialize WaitList $=\{X\}$ (list of nodes whose MBs will be found); \\
Initialize $G=(\mathbf{V},\mathcal{B})$ (initial local network around $X$, with known edges in background knowledge). \\
\Repeat{$\mathrm{WaitList\ is\ empty}$}{
Pop a node $Z$ from the head of WaitList and add $Z$ to DoneList; \\
Find $\mathrm{MB}(Z)$ by a variable selection procedure such as IAMB~\citep{tsamardinos2003algorithms}; \\
Learn the marginal graph $L_Z$ over $\mathrm{MB}^+(Z)$ from data $\mathcal{D}$ or directly from learned graph. \\
Put the edges connected to $Z$ and the v-structures containing $Z$ in $L_Z$ to $G.$ \\
\Repeat{$G$ is not changed}{
For $(a\to b- c)\in G,$ if $(a,c,S_{ac})\in $ IndSet and $b\in S_{ac},$ orient $b\to c.$ \\
For $(a\to b\to c - a)\in G,$ orient $a\to c.$ \\
For $a-b,\,a-c\to b$ and $a-d\to b\in G$, if $(c,d,S_{cd})\in$ IndSet and $a\in S_{cd},$ orient $a\to b.$ \\
For $a-b,\,a-c\to b$ and $a-d\to c\in G$, if $(b,d,S_{bd})\in$ IndSet and $a\in S_{bd},$ orient $a\to b.$
}
Update WaitList $=\{U\in \mathbf{V}|U\mathrm{\ and\ }X\mathrm{\ are\ connected\ by\ an\ undirected\ path\ in\ }G\}\ \setminus$ DoneList; \\
}

\Return $G$.

\end{algorithm}

For the beginning of the algorithm, we construct a partially directed graph $G$ with edges only from background knowledge $\mathcal{B}.$
In the learning process, we gradually add undirected or directed edges into $G,$ which are guaranteed to be true edges in $\mathcal{G}^*.$
In Steps $4$ to $16,$ we iteratively consider each node $Z$ which is connected with $X$ by an undirected path in $G.$

For each $Z$ in the iteration, we find the Markov blanket of $Z$ at Step $6.$ Let $\mathrm{MB}(Z)$ denote the Markov blanket of $Z,$ which is the smallest set such that $Z\perp \mathbf{V}\setminus \mathrm{MB}(Z)\mid \mathrm{MB}(Z),$ and let $\mathrm{MB}^+(Z)=\mathrm{MB}(Z)\cup \{Z\}.$
Multiple methods have been proposed for finding the Markov blanket of a given node, such as IAMB~\citep{tsamardinos2003algorithms}, MMMB~\citep{tsamardinos2003time}, HITON-MB~\citep{aliferis2003hiton}, PCMB and KIAMB~\citep{pena2007towards}.
We choose one of such methods to find $\mathrm{MB}(Z)$ from observational data $\mathcal{D}.$

After that, we find the marginal graph $L_Z$ over $\mathrm{MB}^+(Z)$ at Step $7.$
The marginal graph $L_Z$ over $\mathrm{MB}^+(Z)$ is defined as a DAG that the marginal distribution $P_{\mathrm{MB}^+(Z)}$ over $\mathrm{MB}^+(Z)$ is Markovian and faithful to $L_Z.$
Namely, for any distinct subset $\mathbf{X}',\mathbf{Y}',\mathbf{Z}'$ of $\mathrm{MB}^+(Z),$ $\mathbf{X}'$ and $\mathbf{Y}'$ are d-separated by $\mathbf{Z}'$ in $L_Z$ if and only if $\mathbf{X}'$ and $\mathbf{Y}'$ are conditional independent given $\mathbf{Z}'$.
Therefore, we can learn the CPDAG corresponding to $L_Z$ by causal discovery algorithms such as the IC algorithm~\citep{pearl2000causality} or PC algorithm~\citep{spirtes2000causation}.
On the other hand, if $\mathrm{MB}^+(Z)\subseteq \mathrm{MB}^+(Z')$ for some $Z'\in \mathrm{DoneList},$ we directly set $L_Z=(L_{Z'})_{\mathrm{MB}^+(Z)};$ if $\mathrm{MB}(Z)\subseteq\mathrm{DoneList},$ we set $L_Z$ be the subgraph of $G$ over $\mathrm{MB}^+(Z).$
While finding MB and learning the marginal graph, we also collect discovered conditional independencies in IndSet.
Those conditional independencies can be used to check whether an edge does not exist in $\mathcal{G}^*.$

At Step $8,$ we put the edges connected to $Z$ and the v-structures containing $Z$ in $L_Z$ to $G.$
These edges are guaranteed to be existing also in $\mathcal{G}^*$ (Theorem 1,2 in~\citep{wang2014discovering}).
Then we update the orientation of edges in $G$ by Meek's rules in Step $9$ to $14$.
At Step $15,$ we update WaitList to be the nodes connected with $X$ by an undirected path in $G$ which is not considered previously.
In this way, nodes are dynamically added into or removed from WaitList as new undirected edges are added into $G$ or existing undirected edges in $G$ are oriented.

The following theorem shows that we can obtain the local structure around $X$ by running Algorithm~\ref{alg:mb-by-mb in MPDAG}:

\begin{theorem}
\label{thm:mb-by-mb}
Let $\mathcal{G}^*$ be the MPDAG under background knowledge $\mathcal{B}.$ Let $\mathcal{G}\in [\mathcal{G}^*]$ be the true underlying DAG and $\mathcal{D}$ be i.i.d. observations generated from a distribution Markovian and faithful with respect to $\mathcal{G}.$ Let $X$ be a target node in $\mathcal{G}.$ Suppose all conditional independencies are correctly checked, and let $G$ be the output of MB-by-MB in MPDAG (Algorithm~\ref{alg:mb-by-mb in MPDAG}) with input $X,\mathcal{D},\mathcal{B}$.
Then for each $Z$ connected with $X$ by an undirected path in $\mathcal{G}^*,$ including $X$ itself, we have $pa(Z,G)=pa(Z,\mathcal{G}^*),ch(Z,G)=ch(Z,\mathcal{G}^*),sib(Z,G)=sib(Z,\mathcal{G}^*).$
\end{theorem}

Theorem~\ref{thm:mb-by-mb} shows that some critical properties for local structure of $X$ are covered by the output of Algorithm~\ref{alg:mb-by-mb in MPDAG}.
The parents, children and siblings of $X$ and each sibling of $X$ are identified.
Therefore, local methods such as IDA~\citep{fang2020ida} can be applied to estimate the treatment effect of $X$ on another vertice $Y,$ which needs to check whether orienting the siblings of $X$ forms a new v-structure or directed triangle.
Our method in Section~\ref{sec:local theorems} also needs identification of $pa(X,\mathcal{G}^*),ch(X,\mathcal{G}^*),sib(X,\mathcal{G}^*)$ and $adj(Z,\mathcal{G}^*)$ for each $Z\in sib(X,\mathcal{G}^*).$
Actually, the output of Algorithm~\ref{alg:mb-by-mb in MPDAG} covers the B-component containing $X$, which is defined in~\citep{fang2022representation}, and all directed edges around the B-component.
That is similar to Corollary 4 in~\citep{liu2020local}.

As an alternative baseline method, we can first learn a part of the CPDAG $\mathcal{C}$ corresponding to $\mathcal{G},$ and infer the orientation of other edges from background knowledge and Meek's rules.
Let $ChComp(X)$ be a partially directed graph consisting with the largest undirected subgraph of $\mathcal{C}$ containing $X,$ which is also called the chain component containing $X,$ and all directed edges connected with the chain component.
We can learn $ChComp(X)$ by Algorithm 3 in~\citep{liu2020local}.
For each background knowledge $A\to B$ in $\mathcal{B},$ we either orient $A\to B$ if $A-B$ in $ChComp(X)$ or otherwise directly add this edge to $ChComp(X).$
Then we apply Meek's rules and get the local structure around $X.$

In the following theorem, we show that Algorithm~\ref{alg:mb-by-mb in MPDAG} outperforms the baseline method in the worst case:
\begin{theorem}
\label{thm:algorithm faster}
Suppose all conditional independencies are correctly checked. Let $(V_1,V_2,\dots,V_p)$ be the nodes that are sequentially considered in Step $4$-$16$ of Algorithm~\ref{alg:mb-by-mb in MPDAG}. If they are considered in the same order in the baseline method\footnote{That is, if for some $i<j,$ $V_i,V_j$ are both in the WaitQueue in Algorithm 3 in~\citep{liu2020local}, $V_j$ is never popped before $V_i.$}, the number of conditional independence test used by Algorithm~\ref{alg:mb-by-mb in MPDAG} is not larger than the baseline method.
\end{theorem}

We then show how Algorithm~\ref{alg:mb-by-mb in MPDAG} is performed by a running example. 
A variant of this example in \ref{append:example} shows that in an extreme situation, we may need more conditional independence test comparing to the baseline method.
See Section~\ref{exp:local structure} for experimental comparison between these two methods.

\begin{figure}[t]
\centering
\subfloat[$\mathcal{G}$]
{
\label{fig1:subfig1}\includegraphics[width=0.15\textwidth]{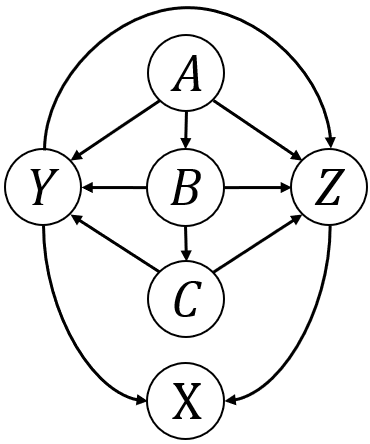}
}
\subfloat[$G_{(X)}$]
{
\label{fig1:subfig2}\includegraphics[width=0.15\textwidth]{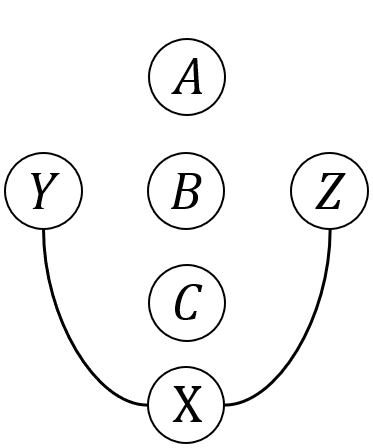}
}
\subfloat[$G_{(X,Y)}$]
{
\label{fig1:subfig3}\includegraphics[width=0.15\textwidth]{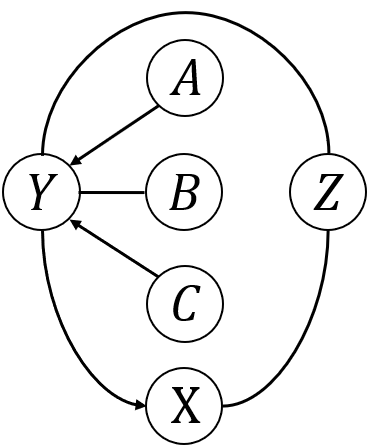}
}
\subfloat[$G_{(X,Y,Z)}$]
{
\label{fig1:subfig4}\includegraphics[width=0.15\textwidth]{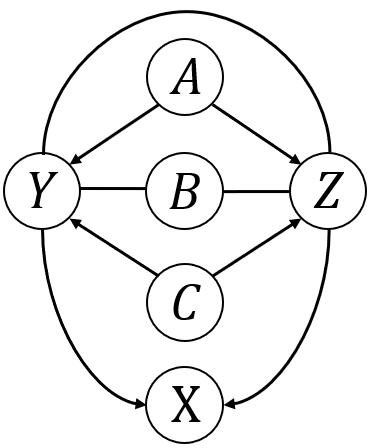}
}
\subfloat[$G^{A\to Y}_{(X)}$]
{
\label{fig1:subfig5}\includegraphics[width=0.15\textwidth]{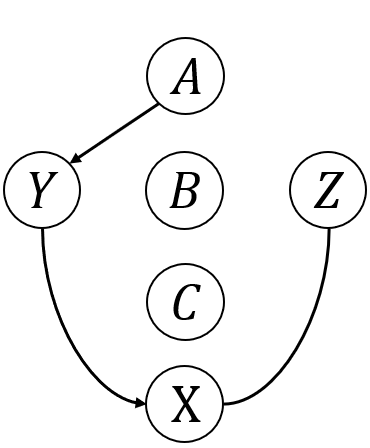}
}
\subfloat[$G^{A\to Y}_{(X,Z)}$]
{
\label{fig1:subfig6}\includegraphics[width=0.15\textwidth]{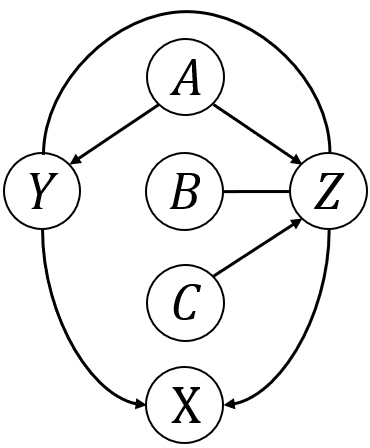}
}
\caption{An example for learning local structure. (a) The original DAG. (b)-(f) Learned structure $G$ defined at Step $3$ in Algorithm~\ref{alg:mb-by-mb in MPDAG}. $G_{D}^{\mathcal{B}}$ denote the learned graph after running Step $4$-$16$ in Algorithm~\ref{alg:mb-by-mb in MPDAG} for each node in $D$, with background knowledge $\mathcal{B}$.} 
\label{fig:subfig_1}
\end{figure}

\begin{example}
Consider the DAG $\mathcal{G}$ shown in Figure~\ref{fig1:subfig1}.
We want to learn the local structure around $X.$
Suppose first that we do not have any background knowledge about $\mathcal{G}.$
In this case, Algorithm~\ref{alg:mb-by-mb in MPDAG} degenerates to the MB-by-MB algorithm.
Firstly, we find the MB of $X,$ which is $\{Y,Z\},$ and learn the marginal graph over $\{X,Y,Z\}.$
Since there is no conditional independence over them, we remain edges $Y-X-Z$ and add them into $G,$ shown in Figure~\ref{fig1:subfig2}.
Now both $Y$ and $Z$ are connected with $X$ by an undirected path in $G.$
We then find the MB of $Y,$ which is $\{A,B,C,X,Z\}.$
By performing conditional independence test among these nodes, we find that $A\perp C|B,$ which implies that $A\to Y\leftarrow C$ is a v-structure.
By Meek's rules, we orient $Y\to X$ in $G.$
That gives Figure~\ref{fig1:subfig3}, and the only node connected with $X$ by an undirected path is $Z.$
After finding the MB of $Z,$ we know that $A$ and $C$ are both adjacent to $Z,$ so $A\perp C|B$ implies that $A\to Z\leftarrow C$ is a v-structure, and we can orient $Z\to X$ by Meek's rules, giving Figure~\ref{fig1:subfig4}.
Therefore, we conclude that the parents of $X$ in $\mathcal{G}^*$ is $\{Y,Z\},$ and it has no sibling or children.

Now suppose that background knowledge $\mathcal{B}=\{A\to Y\}$ is known in prior.
After finding $\mathrm{MB}(X)=\{Y,Z\},$ we know that $A$ and $X$ are not adjacent, so we can orient $Y\to X$ by Meek's rules, giving Figure~\ref{fig1:subfig5}.
Then we find the MB of $Z$ and learn the marginal graph over $\mathrm{MB}^+(Z).$
We orient $A\to Z\leftarrow C$ after finding $A\perp C|B$ and orient $Z\to X$ by Meek's rules, giving Figure~\ref{fig1:subfig6}.

The baseline method first follows the procedure of MB-by-MB algorithm, which also gives Figure~\ref{fig1:subfig4}.
Then it applies background knowledge $A\to Y,$ which is already learned in previous steps, and gives Figure~\ref{fig1:subfig4} as the final output.
We can see that, Algorithm~\ref{alg:mb-by-mb in MPDAG} explores less nodes than the baseline method in this example.
\end{example}

\subsection{Local structure learning with non-ancestral or ancestral information}
\label{subsec:local learn non-ans}

Non-ancestral information and ancestral information are characterized as orientation of edges between a set of nodes~\citep{fang2022representation}. For distinct nodes $X,Y$ in an MPDAG $\mathcal{G}^*$ with true underlying DAG $\mathcal{G},$ let $\mathbf{C}$ denote the critical set of $X$ respect to $Y$ in $\mathcal{G}^*.$ It has been shown that, $X$ is not an ancestor of $Y$ in $\mathcal{G}$ if and only if for all $C\in \mathbf{C},$ $C\to X$ in $\mathcal{G}.$ On the other hand, $X$ is an ancestor of $Y$ if and only if there exists $C\in \mathbf{C}$ such that $X\to C$ in $\mathcal{G}.$ Therefore, non-ancestral information can be equivalently transformed into a set of direct edges in the MPDAG, while ancestral information may not be fully represented by an MPDAG, but an MPDAG with a set of direct causal clauses (DCC). For example, $X\overset{or}{\to} \mathbf{C}$ denotes the fact that there exists $C\in \mathbf{C}$ such that $X\to C$.

However, the definition of critical set is based on the entire MPDAG: in order to check whether a node adjacent to $X$ is in the critical set of $X$ with respect to $Y$ in $\mathcal{G}^*,$ we need to list all chordless b-possibly causal paths from $X$ to $Y$ in $\mathcal{G}^*.$ Proposition 4.7 in~\citet{zuo2022counterfactual} proposed that we can list all b-possibly causal path of definite status from $X$ to $Y,$ but it still needs to learn the entire MPDAG. Instead, we want to find a local characterization of the critical set. The following lemma relates the IDA framework with the critical set.

\begin{lemma}
\label{lemma:identify critical set}
Let $\mathcal{G}^*$ be an MPDAG and $X, Y$ be two distinct nodes in $\mathcal{G}^*.$ Suppose that $X$ is not a definite cause of $Y$ in $\mathcal{G}^*.$ Let $\mathbf{C}$ be the critical set of $X$ with respect to $Y$ in $\mathcal{G}^*.$ Let $\mathcal{Q}$ be the set of all $\mathbf{Q}\subseteq sib(X,\mathcal{G}^*)$ such that orienting $\mathbf{Q}\to X$ and $X\to sib(X,\mathcal{G}^*)\setminus \mathbf{Q}$ does not introduce any v-structure collided on $X$ or any directed triangle containing $X.$ Then we have
$$\cap \left\{ \mathbf{Q}\in \mathcal{Q} \mid X\perp Y \mid pa(X,\mathcal{G}^*)\cup \mathbf{Q} \right\} = an(\mathbf{C},\mathcal{G}^*)\cap sib(X,\mathcal{G}^*).$$
\end{lemma}

Lemma~\ref{lemma:identify critical set} shows that we can  identify the ancestor of the critical set locally. If background knowledge shows that $X$ is not a cause of $Y,$ it is equivalent to $\mathbf{C}\to X,$ and by acyclic constraint it implies $an(\mathbf{C},\mathcal{G}^*) \cap sib(X,\mathcal{G}^*)\to X.$ This fact inspires Algorithm~\ref{alg:local struct:direct and non-an}, which learns the local structure around $X$ when background knowledge consists of direct causal information $\mathcal{B}_1=\{(F_i,T_i)\}_{i=1}^{k_1}$ where it is known that $F_i\to T_i$ exists in $\mathcal{G}$ for $i=1,2,\dots,k_1,$ and non-ancestral relationship $\mathcal{B}_2=\{(N_j,T_j)\}_{j=1}^{k_2},$ where it is known that $N_j$ is not a cause of $T_j$ in $\mathcal{G}$ for $j=1,2,\dots,k_2.$

\begin{algorithm}[t]
\LinesNumbered
\caption{Learning the local structure around $X$ given background knowledge consisting of direct causal information $\mathcal{B}_1$ and non-ancestral information $\mathcal{B}_2$.}
\label{alg:local struct:direct and non-an}
\SetKwInOut{Input}{Input}
\SetKwInOut{Output}{Output}
\small

\Input{Target node $X$, observational data $\mathcal{D}$, background knowledge $\mathcal{B}=\mathcal{B}_1\cup \mathcal{B}_2$.}
\Output{A PDAG $G$ containing the local structure of $X$.}
Let $G$ be the output of Algorithm~\ref{alg:mb-by-mb in MPDAG} under input $X,\mathcal{D},\mathcal{B}_1$.

\ForEach{$(N_j,T_j)\in \mathcal{B}_2$}{
\If{$N_j$ and $X$ are connected by an undirected path in $G$}{
Let $\mathbf{candC} = sib(N_j,G).$

\ForEach{$\mathbf{Q}\subseteq sib(N_j,G)$ such that orienting $\mathbf{Q}\to N_j$ does not introduce any v-structure collided on $N_j$ or any directed triangle containing $N_j$}{
\If{$N_j \perp T_j \mid pa(N_j,G)\cup \mathbf{Q}$}{
	Update $\mathbf{candC}:=\mathbf{candC}\cap \mathbf{Q}.$
}
}
Orient $\mathbf{candC}\to N_j$ and apply Meek's rules on $G.$
}
}
\Return $G$.

\end{algorithm}

In Algorithm~\ref{alg:local struct:direct and non-an}, we first learn the local structure $G$ under direct causal information $\mathcal{B}_1$ using Algorithm~\ref{alg:mb-by-mb in MPDAG}. Let $\mathcal{G}^*_1$ denote the MPDAG under background knowledge $\mathcal{B}_1.$ By Theorem~\ref{thm:mb-by-mb}, $G$ recovers the B-component including $X$ in $\mathcal{G}^*_1.$ For each $(N_j,T_j)\in \mathcal{B}_2,$ if $N_j$ and $X$ are not connected by an undirected path in $G,$ $N_j$ and $X$ are not in the same B-component in $\mathcal{G}^*_1,$ so orienting edges around $N_j$ do not affect the local structure around $X.$ On the other hand, if $N_j$ and $X$ are connected by an undirected path in $G,$ they are in the same B-component in $\mathcal{G}^*_1.$ Therefore, by Theorem~\ref{thm:mb-by-mb}, the local structure around $N_j$ in $\mathcal{G}^*_1$ is identical to the local structure around $N_j$ in $G.$ By Lemma~\ref{lemma:identify critical set}, Step 4 to 9 in Algorithm~\ref{alg:local struct:direct and non-an} finds $an(\mathbf{C},\mathcal{G}^*) \cap sib(X,\mathcal{G}^*),$ and we let them point to $N_j$ and update the local structure $G$ in Step 10. 

The following theorem shows the correctness of Algorithm~\ref{alg:local struct:direct and non-an}.

\begin{theorem}
\label{thm:correct alg non-an}
Let $\mathcal{G}^*$ be the MPDAG under background knowledge $\mathcal{B}$ consisting of direct causal information $\mathcal{B}_1$ and non-ancestral information $\mathcal{B}_2.$ Let $\mathcal{G}\in [\mathcal{G}^*]$ be the true underlying DAG and $\mathcal{D}$ be i.i.d. observations generated from a distribution Markovian and faithful with respect to $\mathcal{G}.$ Let $X$ be a target node in $\mathcal{G}.$ Suppose all conditional independencies are correctly checked, and let $G$ be the output of Algorithm~\ref{alg:local struct:direct and non-an} with input $X,\mathcal{D},\mathcal{B}$.
Then for each $Z$ connected with $X$ by an undirected path in $\mathcal{G}^*,$ including $X$ itself, we have $pa(Z,G)=pa(Z,\mathcal{G}^*),ch(Z,G)=ch(Z,\mathcal{G}^*),sib(Z,G)=sib(Z,\mathcal{G}^*).$
\end{theorem}

Then we consider ancestral information as background knowledge. Based on an learned MPDAG $\mathcal{G}^*$ with true underlying DAG $\mathcal{G}\in [\mathcal{G}^*],$ if there is extra information that $F_k$ is an ancestor of $T_k$ in $\mathcal{G},$ equivalently we have $T_k\overset{or}{\to}\mathbf{C},$ where $\mathbf{C}$ is the critical set of $T_k$ with respect to $F_k$ in $\mathcal{G}^*.$ 

Unfortunately, finding $an(\mathbf{C},\mathcal{G}^*)\cap sib(T_k,\mathcal{G}^*)$ is not enough in this case, since there is no extra information from the acyclic constraint. However, by Lemma 1 in~\citet{meek1995causal}, there is no partially directed cycle in CPDAGs. So when $\mathcal{G}^*$ is a CPDAG, we have $an(\mathbf{C},\mathcal{G}^*) \cap sib(X,\mathcal{G}^*)=\mathbf{C}.$ Therefore, in this case, we modify our algorithm to first learn a local structure of the CPDAG using MB-by-MB algorithm. Then we transform ancestral information into direct causal clauses by Lemma~\ref{lemma:identify critical set}, and last consider direct causal information and non-ancestral information and update the local structure. The main procedure of the algorithm is given in Algorithm~\ref{alg:local struct:all types} in~\ref{append:algorithm}.

\section{Local characterization of causal relations in maximal PDAG}
\label{sec:local theorems}
We now present the criteria for identifying all definite descendants, possible descendants, and definite non-descendants of the target node $X$ in the MPDAG $\mathcal{G}^*.$ These criteria rely only on the local structure learned by Algorithm~\ref{alg:mb-by-mb in MPDAG} and some additional conditional independence tests. Before presenting the theorems, we classify definite causal relationships into explicit and implicit causes, similar to the CPDAG version in~\citep{fang2022local}.

\begin{definition}
Let $X,Y$ be two distinct nodes in an MPDAG $\mathcal{G}^*$ such that $X$ is a definite cause of $Y$.
We say $X$ is an explicit cause of $Y$ if there is a common causal path from $X$ to $Y$ in every DAG in $\left[\mathcal{G}^*\right],$ otherwise $X$ is an implicit cause of $Y.$
\end{definition}

In other words, $X$ is an explicit cause of $Y$ if and only if $X\in an(Y,\mathcal{G}^*).$ For more details on explicit and implicit causal relationships, refer to Section 3 of~\citep{fang2022local}.

The first theorem gives a sufficient and necessary condition for any node to be a definite non-descendant of $X.$

\begin{theorem}
\label{thm:definite non-cause}
For any two distinct nodes $X$ and $Y$ in an MPDAG $\mathcal{G}^*,$ $Y$ is a definite non-descendant of $X$ if and only if $X\perp Y\mid pa(X,\mathcal{G}^*)$ holds.
\end{theorem}

The condition in Theorem~\ref{thm:definite non-cause}, $X\perp Y\mid pa(X,\mathcal{G}^*),$ is analogous to the Markov property $X\perp Y\mid pa(X,\mathcal{G})$ for any non-descendant $Y$ of $X$ in a DAG $\mathcal{G}$~\citep{lauritzen1990independence, lauritzen1996graphical}. The key difference is that any node in $sib(X,\mathcal{G}^*)$ could also be a parent of $X$ in $\mathcal{G}$~\citep{meek1995causal}. Notably, the set $pa(X,\mathcal{G}^*)$ is always a subset of $pa(X,\mathcal{G})$ for every $\mathcal{G}\in [\mathcal{G}^*].$ Intuitively, if a smaller set blocks all paths from $X$ to $Y,$ a stronger non-ancestral relation from $X$ to $Y$ is established. This intuition also applies to the next two theorems, which judge explicit cause and implicit cause relationships.

\begin{theorem}
\label{thm:explicit cause}
For any two distinct nodes $X$ and $Y$ in an MPDAG $\mathcal{G}^*,$ $X$ is an explicit cause of $Y$ if and only if $X \not\perp Y\mid pa(X,\mathcal{G}^*)\cup sib(X,\mathcal{G}^*)$ holds.
\end{theorem}

The condition in Theorem~\ref{thm:explicit cause}, $X \not\perp Y\mid pa(X,\mathcal{G}^*)\cup sib(X,\mathcal{G}^*),$ is similar to the local Markov property in chain graphs~\citep{frydenberg1990chain}. However, an MPDAG is not generally a chain graph, as MPDAGs may contain partially directed cycles, which are not allowed in chain graphs. Therefore, Theorem~\ref{thm:explicit cause} is not a straightforward extension of its CPDAG version. It can also serve as a technical lemma for deriving graphical properties in MPDAGs.

\begin{theorem}
\label{thm:implicit cause}
For any two distinct nodes $X$ and $Y$ in an MPDAG $\mathcal{G}^*,$ let $\mathcal{Q}$ be the set of maximal cliques of the induced subgraph of $\mathcal{G}^*$ over $sib(X,\mathcal{G}^*).$ Then, $X$ is an implicit cause of $Y$ if and only if $X\perp Y| pa(X,\mathcal{G}^*)\cup sib(X,\mathcal{G}^*)$ and $X\not\perp Y| pa(X,\mathcal{G}^*)\cup \mathbf{Q}$ for every $\mathbf{Q}\in \mathcal{Q}.$
\end{theorem}

In Theorem~\ref{thm:implicit cause}, we need to find all maximal cliques in the induced subgraph of $\mathcal{G}^*$ over $sib(X,\mathcal{G}^*).$ By Theorem~\ref{thm:mb-by-mb}, Algorithm~\ref{alg:mb-by-mb in MPDAG} can correctly discover the parents, children and sibling for each node in $sib(X,\mathcal{G}^*),$ thus fully recovering the induced subgraph of $\mathcal{G}^*$ over $sib(X,\mathcal{G}^*).$

The proof of Theorem~\ref{thm:implicit cause} is based on Lemma~\ref{lemma: maxclique is pa} in~\ref{append:proof}, which shows that for any maximal clique $\mathbf{Q}$ in $\mathcal{Q},$ there exists a DAG $\mathcal{G}\in [\mathcal{G}^*]$ such that $pa(X,\mathcal{G})=pa(X,\mathcal{G}^*)\cup \mathbf{Q}.$ In other words, any maximal clique in $\mathcal{Q}$ can serve as an additional parent set of $X$. This lemma is fundamental in IDA theory. Let $\mathcal{R}$ be the set that includes all sets of nodes $\mathbf{R}$ such that there exists a DAG $\mathcal{G}\in [\mathcal{G}^*]$ with $pa(X,\mathcal{G})=pa(X,\mathcal{G}^*)\cup \mathbf{R}.$ Lemma~\ref{lemma: maxclique is pa} divides $\mathcal{R}$ into two parts: $\mathcal{Q},$ the set of maximal cliques of the induced subgraph of $\mathcal{G}^*$ over $sib(X,\mathcal{G}^*),$ and the set of all cliques $\mathbf{R}$ that are not maximal but for which no $R\in \mathbf{R}$ and $C\in sib(X,\mathcal{G}^*)\backslash \mathbf{R}$ exist such that $C\to R$ in $\mathcal{G}^*$~\citep{fang2020ida}.

We can combine these two conditions for judging the definite cause relationship.

\begin{corollary}
\label{cor:definite cause}
For any two distinct nodes $X$ and $Y$ in an MPDAG $\mathcal{G}^*,$ let $\mathcal{Q}$ be the set of maximal cliques of the induced subgraph of $\mathcal{G}^*$ over $sib(X,\mathcal{G}^*).$ Then $X$ is an definite cause of $Y$ if and only if $X \not\perp Y|pa(X,\mathcal{G}^*)\cup sib(X,\mathcal{G}^*)$ or $X\not\perp Y| pa(X,\mathcal{G}^*)\cup \mathbf{Q}$ for every $\mathbf{Q}\in \mathcal{Q}.$
\end{corollary}

Based on these conditions, we can finally present an algorithm to identify all definite descendants, definite non-descendants and possible descendants of $X,$ given data $\mathcal{D}$ and background knowledge $\mathcal{B}.$ This algorithm is shown in Algorithm~\ref{alg:find descendants}.
We call this Algorithm by LABITER (\textbf{L}ocal \textbf{A}lgorithm with \textbf{B}ackground knowledge for \textbf{I}dentifying the \textbf{T}yp\textbf{e} of causal \textbf{R}elations). The following theorem, directly derived from Theorem~\ref{thm:definite non-cause} and Corollary~\ref{cor:definite cause}, verifies the correctness of LABITER. 

\begin{algorithm}[t]
\LinesNumbered
\caption{LABITER: Local algorithm for finding all definite descendants, definite non-descendants and possible descendants of $X$ in $\mathcal{G}^*$.}
\label{alg:find descendants}
\SetKwInOut{Input}{Input}
\SetKwInOut{Output}{Output}
\small

\Input{A target $X$, observational data $\mathcal{D}$, background knowledge $\mathcal{B}$.}
\Output{The set of definite descendants, definite non-descendants and possible descendants of $X$ in the MPDAG $\mathcal{G}^*$.}

Learn a local structure $G$ by Algorithm~\ref{alg:mb-by-mb in MPDAG}.

Let $\mathcal{Q}$ be the set of maximal cliques of the induced subgraph of $G$ over $sib(X,G).$

Initialize DefDes, DefNonDes, PosDes $=\emptyset, \emptyset, \emptyset$

\For{$Y \in \mathbf{V}\,\backslash\, \{X\}$} {
\uIf{$X\perp Y | pa(X,G)$}{
Add $Y$ to DefNonDes;
}
\uElseIf{$X \not\perp Y|pa(X,G)\cup sib(X,G)$ or $X\not\perp Y| pa(X,G)\cup \mathbf{Q}$ for every $\mathbf{Q}\in \mathcal{Q}.$}{
Add $Y$ to DefDes;
}
\Else{
Add $Y$ to PosDes;
}
}
\Return DefDes, DefNonDes, PosDes.
\end{algorithm}

\begin{theorem}
Suppose all conditional independencies are correctly checked, then Algorithm~\ref{alg:find descendants} correctly returns the definite descendants, definite non-descendants and possible descendants of $X$ in $\left[\mathcal{G}^*\right].$
\end{theorem}

\section{Experiments}
\label{sec:experiments}
\subsection{Learning local structure in MPDAG}
\label{exp:local structure}
In this section, we conduct a simulation study to test  Algorithm~\ref{alg:mb-by-mb in MPDAG}.
The MB-by-MB algorithm learns a chain component containing the target node and edges around it~\citep{liu2020local}, and it has been shown to outperforms other algorithms when learning local structure in sparse networks~\citep{wang2014discovering}.
Therefore, we compare the performance of Algorithm~\ref{alg:mb-by-mb in MPDAG} with the MB-by-MB and PC algorithms when learning local structure in a single chain component.

In sparse DAGs, the maximal size of chain component is relatively small even in large graphs~\citep{he2013reversible}.
A chain component with size $10$ is larger than $95\%$ of  chain components in a DAG with $200$ nodes and $300$ edges~\citep{he2013reversible}.
In our experiment, the size of chain components $n$ is chosen from $\{5,10,15\}.$
For each graph, we randomly select some edges as background knowledge and run three algorithms: MB-by-MB in MPDAG (Algorithm~\ref{alg:mb-by-mb in MPDAG}), the original MB-by-MB algorithm implemented on the CPDAG~\citep{wang2014discovering}, and the PC algorithm~\citep{spirtes1991algorithm}.
The proportion of background knowledge in the undirected edges ranges from $0.1$ to $0.9$.

We first use an oracle conditional independence test, implemented via d-separation in the true DAG, and collect the number of conditional independence tests used by each algorithm. The number of conditional independence tests indicates the efficiency of each algorithm without considering the estimation error. The experiment is repeated $500$ times, and the results are plotted in Figure~\ref{fig:exp1}.
Then, we draw $N=2000$ samples from a linear Gaussian model faithful to the true DAG, where weights are randomly drawn from a uniform distribution between $[0.6,1.2].$
We compare the structural Hamming distance (SHD) between the learned structure and the true local structure to evaluate the efficacy of each algorithm.
The experiment is repeated for $5000$ times, and the results are plotted in Figure~\ref{fig:exp1_2}.

\begin{figure}[t]
\centering
\subfloat[$n=5$] 
{
\label{fig2:subfig1}\includegraphics[width=0.32\textwidth]{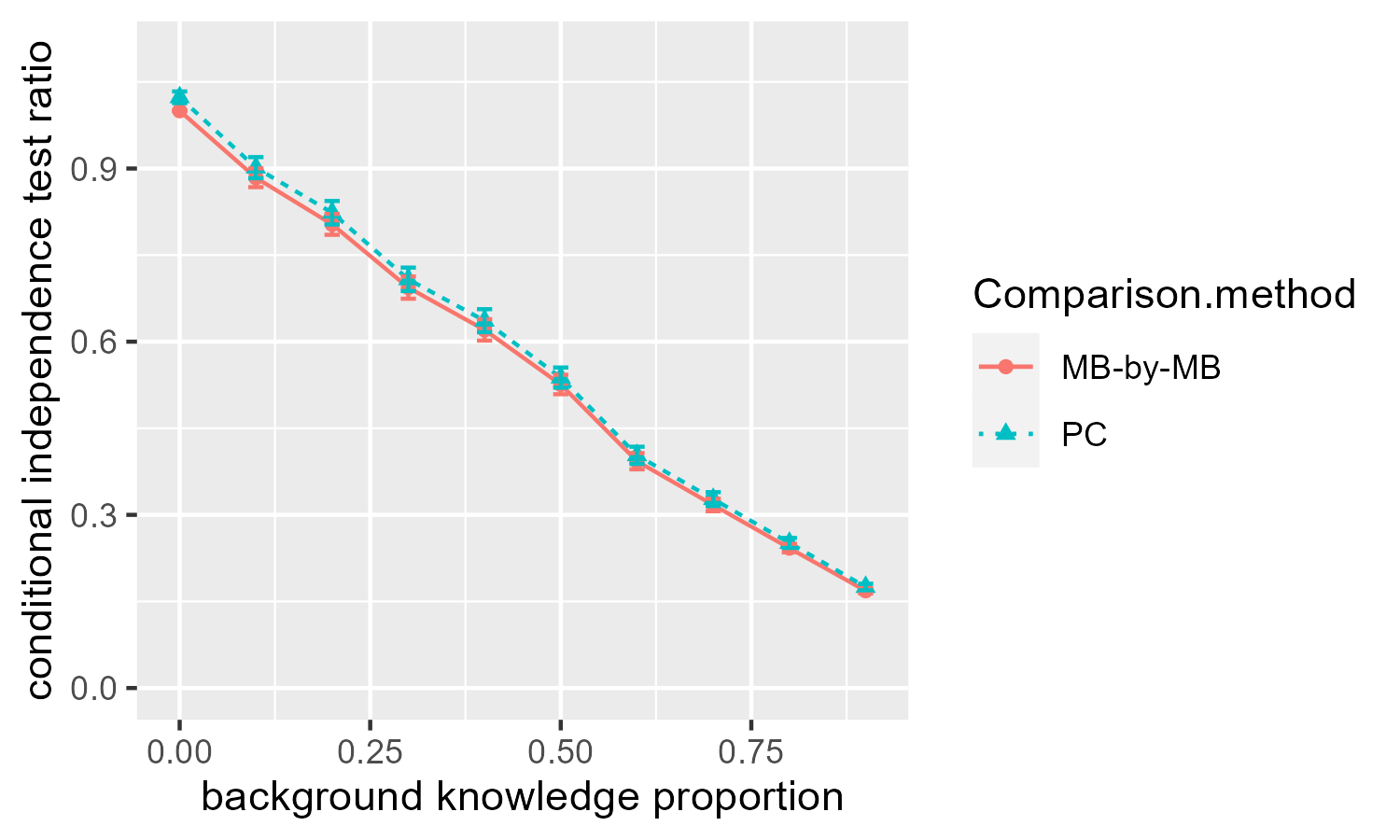}
}
\subfloat[$n=10$] 
{
\label{fig2:subfig2}\includegraphics[width=0.32\textwidth]{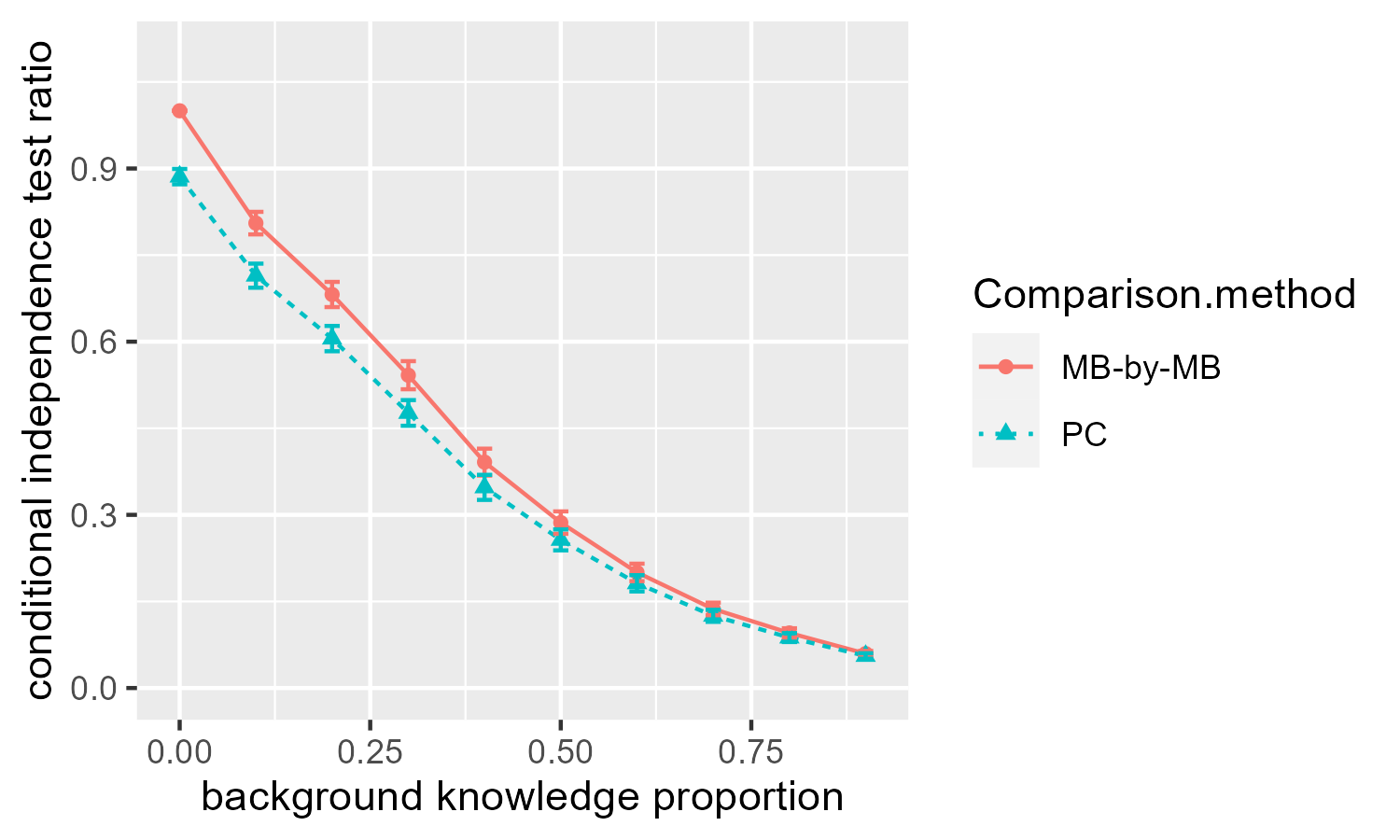}
}
\subfloat[$n=15$] 
{
\label{fig2:subfig3}\includegraphics[width=0.32\textwidth]{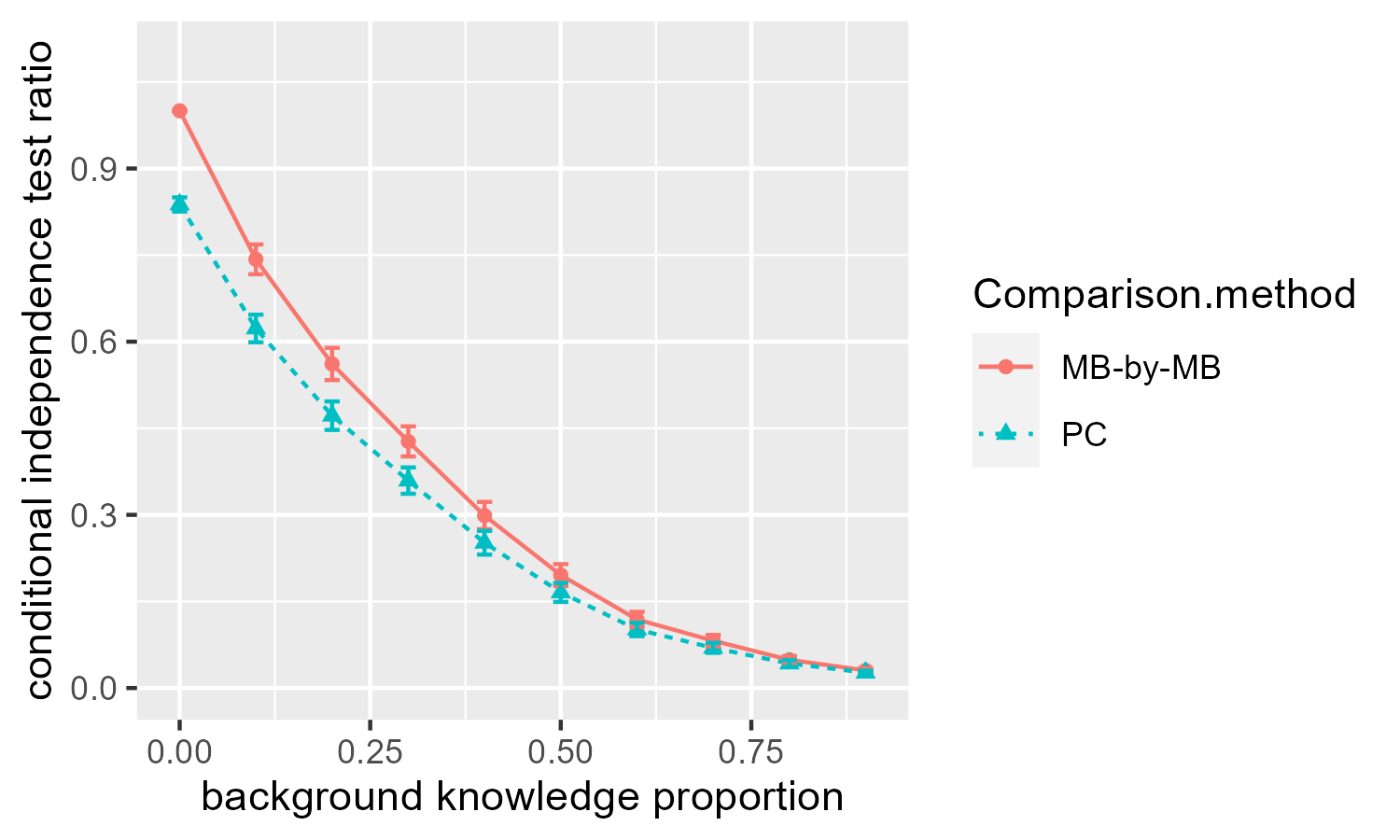}
}

\caption{Ratio of the number of conditional independence test required by MB-by-MB in MPDAG (Algorithm~\ref{alg:mb-by-mb in MPDAG}) to the comparison method (solid red line: MB-by-MB algorithm, dashed cyan line: PC algorithm) when learning local structure in a chain component.} 
\label{fig:exp1}
\end{figure}

\begin{figure}[t]
\centering
\subfloat[$n=5$] 
{
\label{fig3:subfig1}\includegraphics[width=0.32\textwidth]{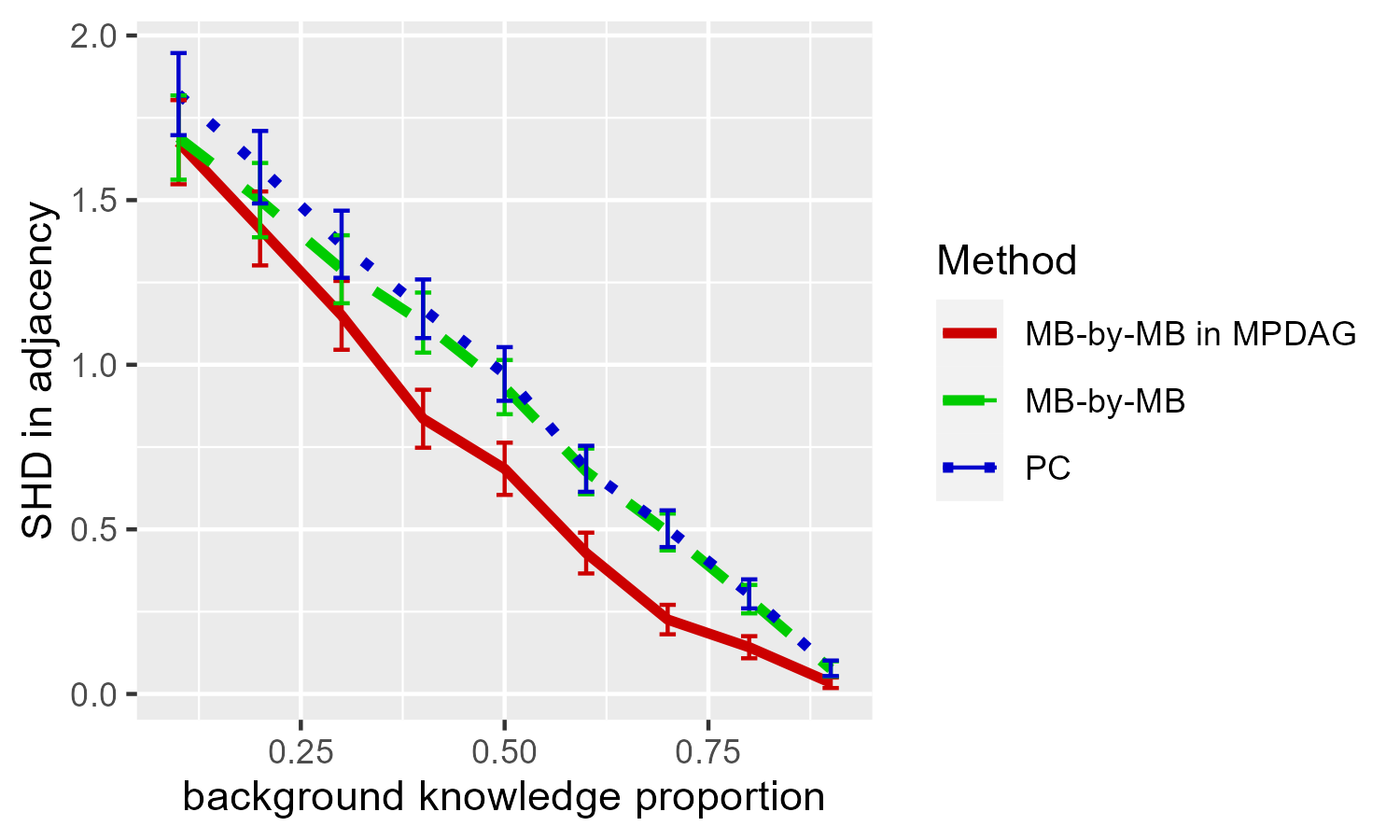}
}
\subfloat[$n=10$] 
{
\label{fig3:subfig2}\includegraphics[width=0.32\textwidth]{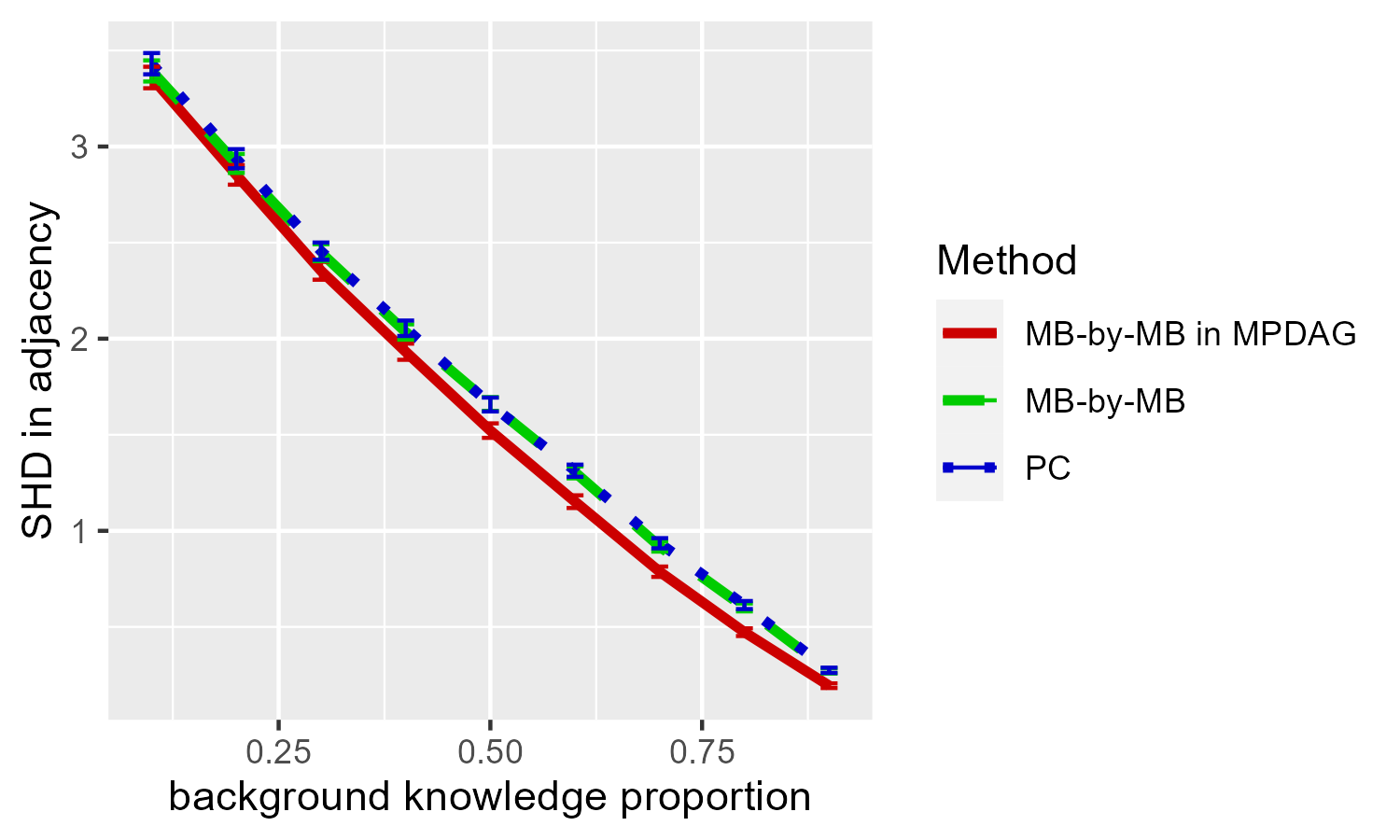}
}
\subfloat[$n=15$] 
{
\label{fig3:subfig3}\includegraphics[width=0.32\textwidth]{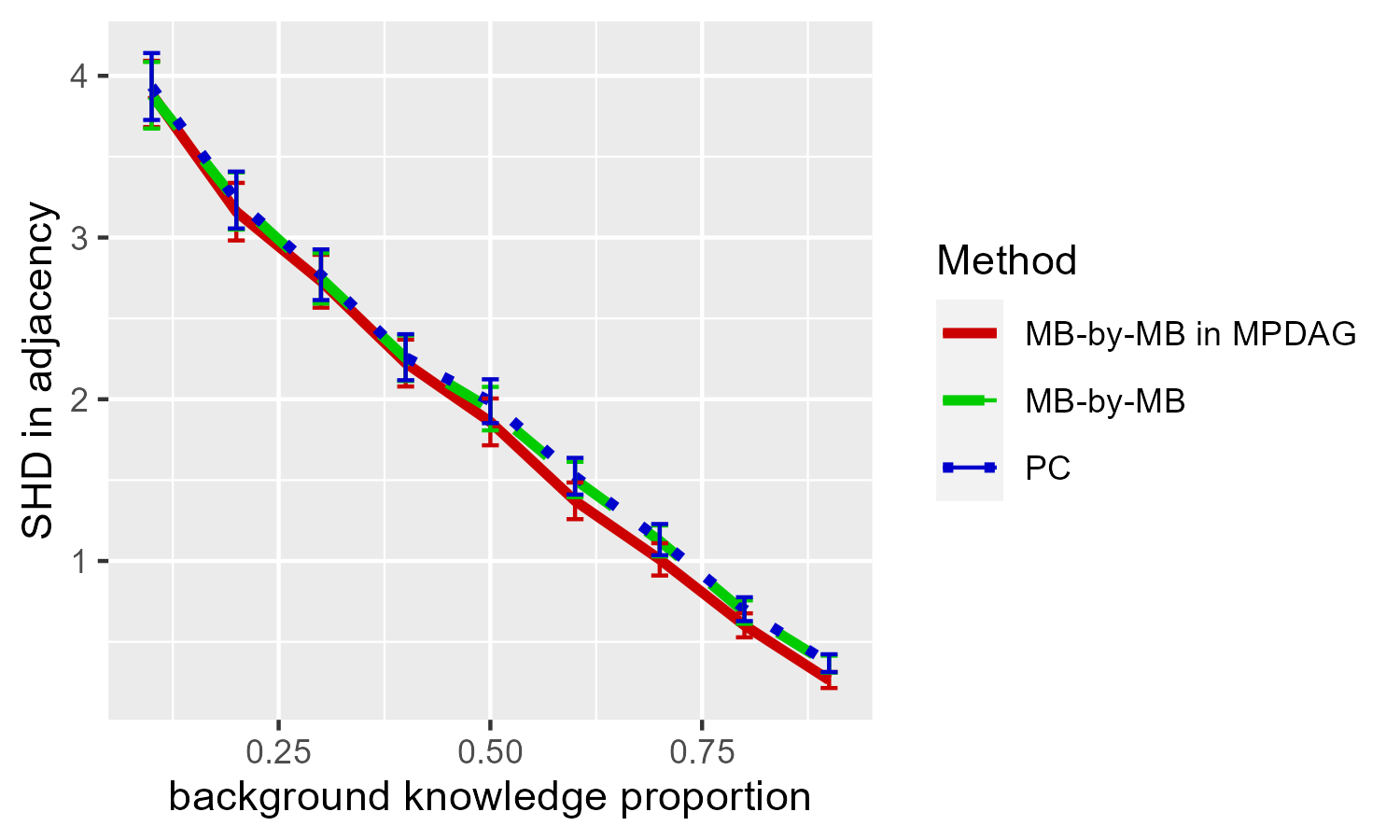}
}

\caption{Average SHD between the learned local structure and true MPDAG around the target node. Solid red line: MB-by-MB in MPDAG (Algorithm~\ref{alg:mb-by-mb in MPDAG}); dashed green line: MB-by-MB algorithm; dotted blue line: PC algorithm.} 
\label{fig:exp1_2}
\end{figure}

Figure~\ref{fig:exp1} shows the ratio of the number of conditional independence test required by Algorithm~\ref{alg:mb-by-mb in MPDAG} to that required by the comparison method.
We can see that, as the amount of background knowledge increases, the number of conditional independence tests required by the algorithm indeed decreases.
Since we focus on a single chain component instead of the whole DAG, the performance of the MB-by-MB algorithm is similar to the PC algorithm.
When background knowledge is added, the number of nodes connected to the target node by an undirected path in the MPDAG decreases.
Therefore, Algorithm~\ref{alg:mb-by-mb in MPDAG} only needs to learn a smaller local structure, leading to fewer conditional independence tests.

Figure~\ref{fig:exp1_2} shows the SHD between the learned local structure and the true local structure around the target node in the MPDAG for each algorithm.
All algorithms perform better as the proportion of background knowledge increases, since more edges around the target node are given as background knowledge.
Nevertheless, the MB-by-MB in MPDAG algorithm (Algorithm~\ref{alg:mb-by-mb in MPDAG}) outperforms other algorithms when more edges are known in prior.

\subsection{Identifying causal relations}

\begin{figure}[t!]
	\centering
	
	\subfloat[$n=50$, $N=100$ \label{fig:kappa_50_100}]{
		\begin{minipage}[t]{0.22\textwidth}
			\centering
			\includegraphics[width=\textwidth]{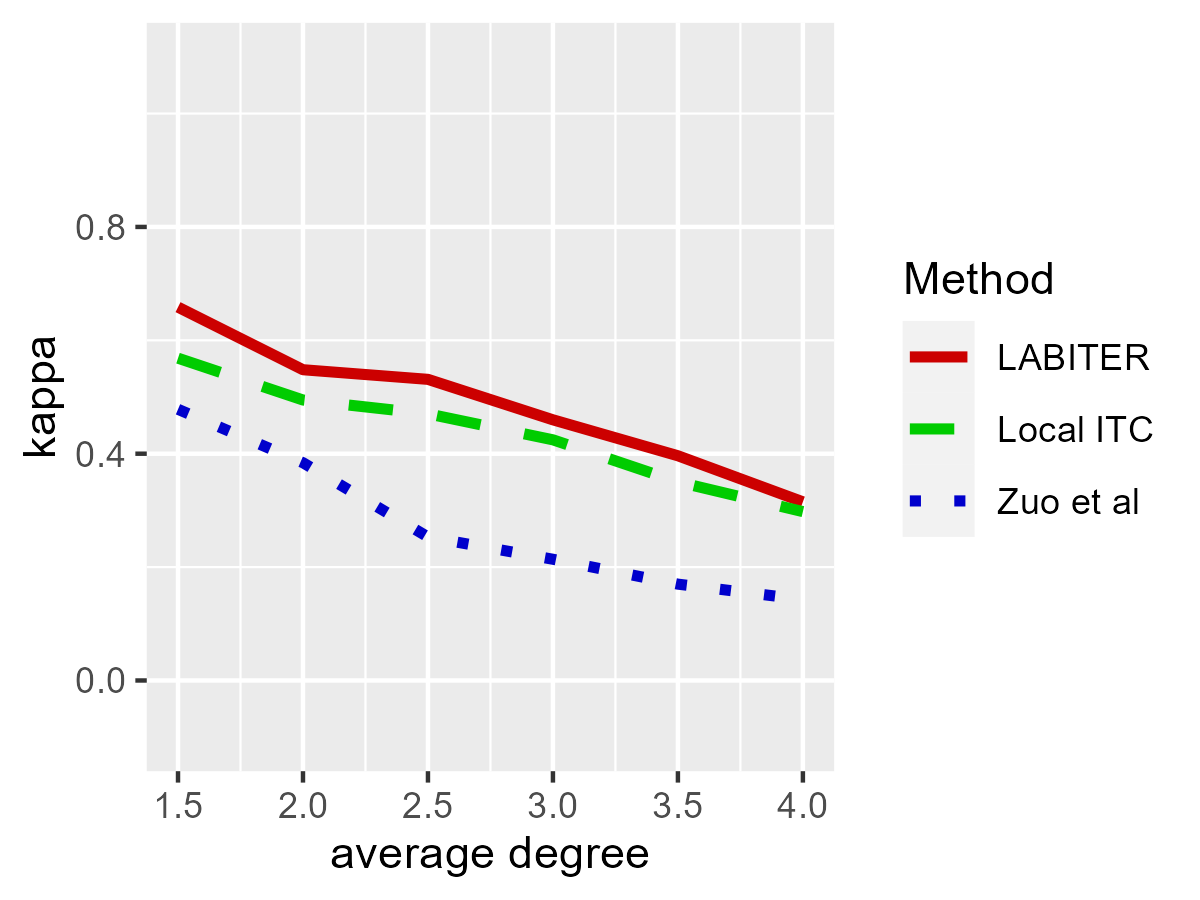}
		\end{minipage}%
	}%
	\hspace{0.01\textwidth}
	\subfloat[$n=50$, $N=200$  \label{fig:kappa_50_200}]{
		\begin{minipage}[t]{0.22\textwidth}
			\centering
			\includegraphics[width=\textwidth]{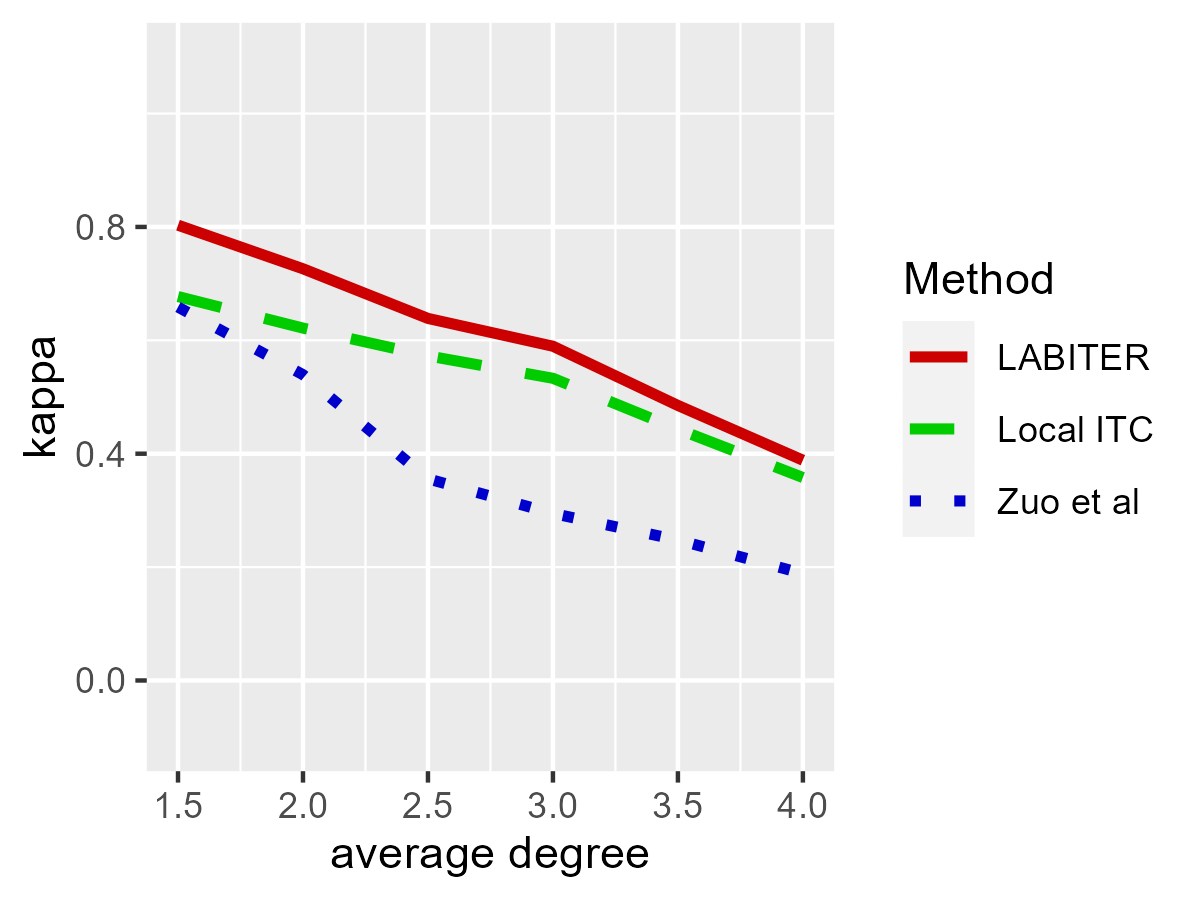}
		\end{minipage}%
	}%
	\hspace{0.01\textwidth}
	\subfloat[$n=50$, $N=500$ \label{fig:kappa_50_500}]{
		\begin{minipage}[t]{0.22\textwidth}
			\centering
			\includegraphics[width=\textwidth]{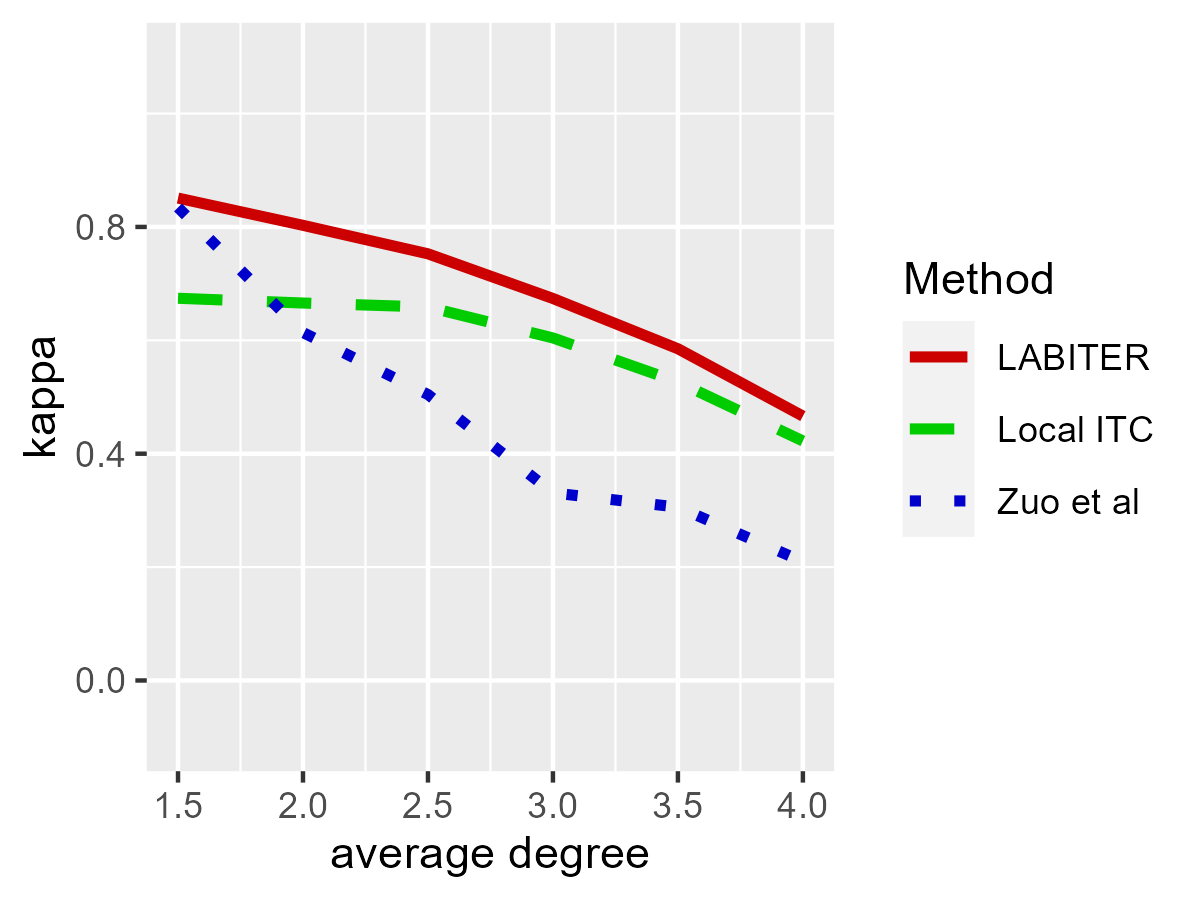}
		\end{minipage}%
	}%
	\hspace{0.01\textwidth}
	\subfloat[$n=50$, $N=1000$ \label{fig:kappa_50_1000}]{
		\begin{minipage}[t]{0.22\textwidth}
			\centering
			\includegraphics[width=\textwidth]{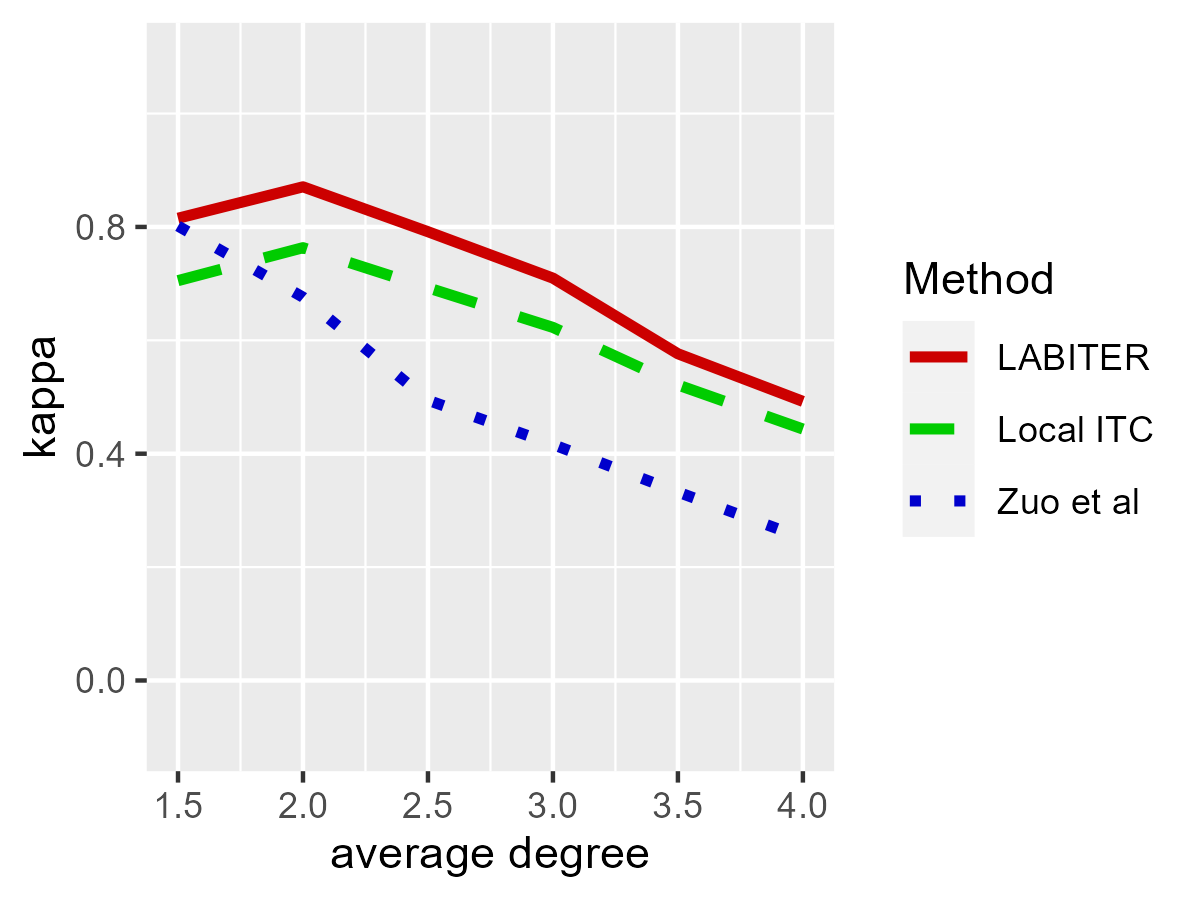}
		\end{minipage}%
	}%

    \subfloat[$n=50$, $N=100$ \label{fig:time_50_100}]{
		\begin{minipage}[t]{0.22\textwidth}
			\centering
			\includegraphics[width=\textwidth]{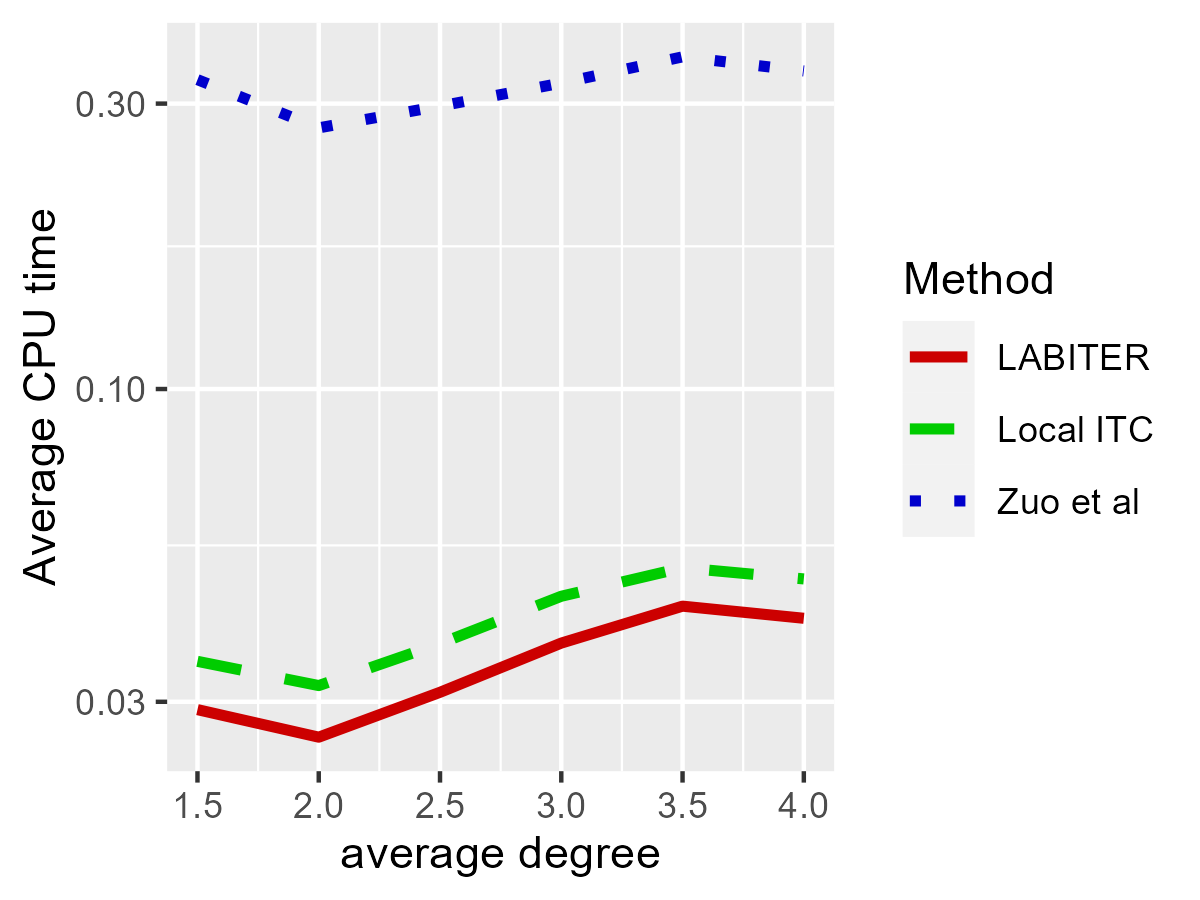}
		\end{minipage}%
	}%
	\hspace{0.01\textwidth}
	\subfloat[$n=50$, $N=200$  \label{fig:time_50_200}]{
		\begin{minipage}[t]{0.22\textwidth}
			\centering
			\includegraphics[width=\textwidth]{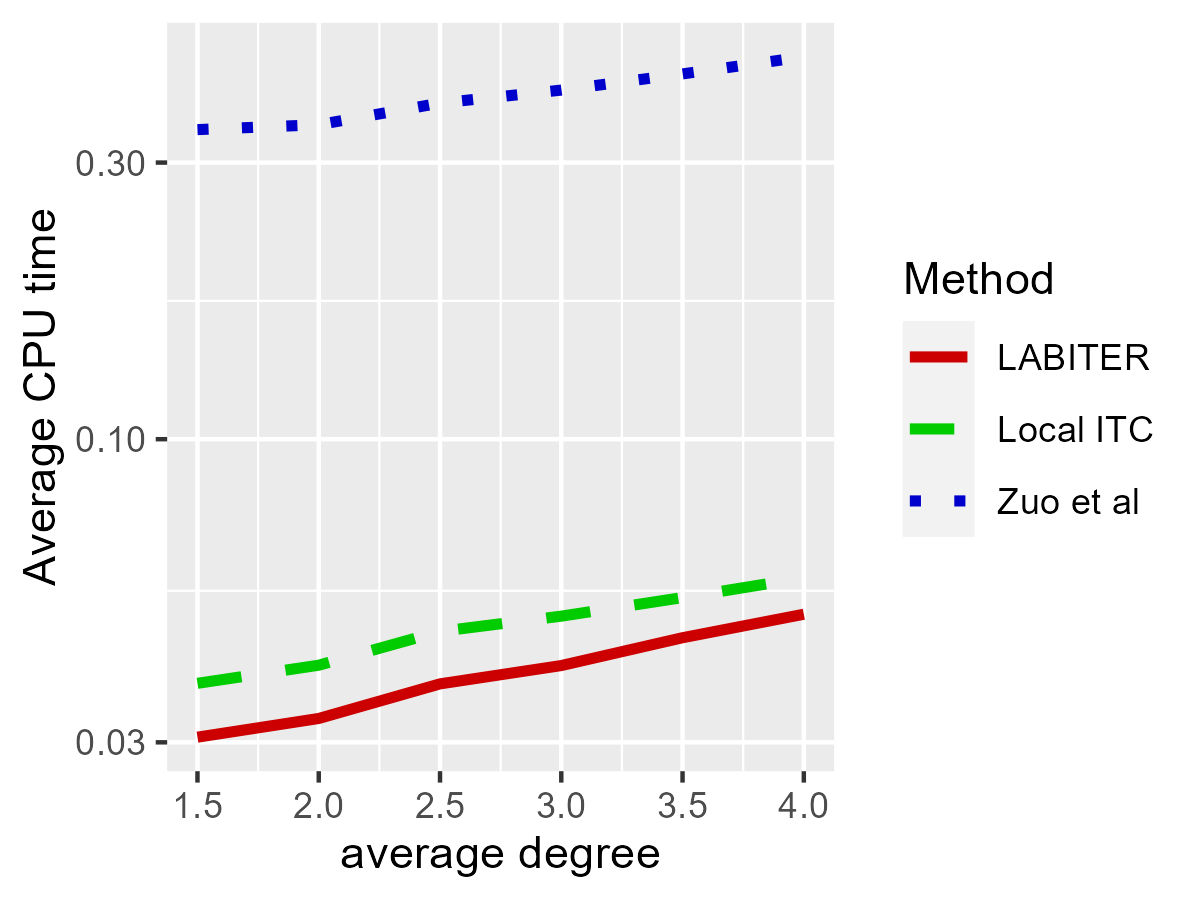}
		\end{minipage}%
	}%
	\hspace{0.01\textwidth}
	\subfloat[$n=50$, $N=500$ \label{fig:time_50_500}]{
		\begin{minipage}[t]{0.22\textwidth}
			\centering
			\includegraphics[width=\textwidth]{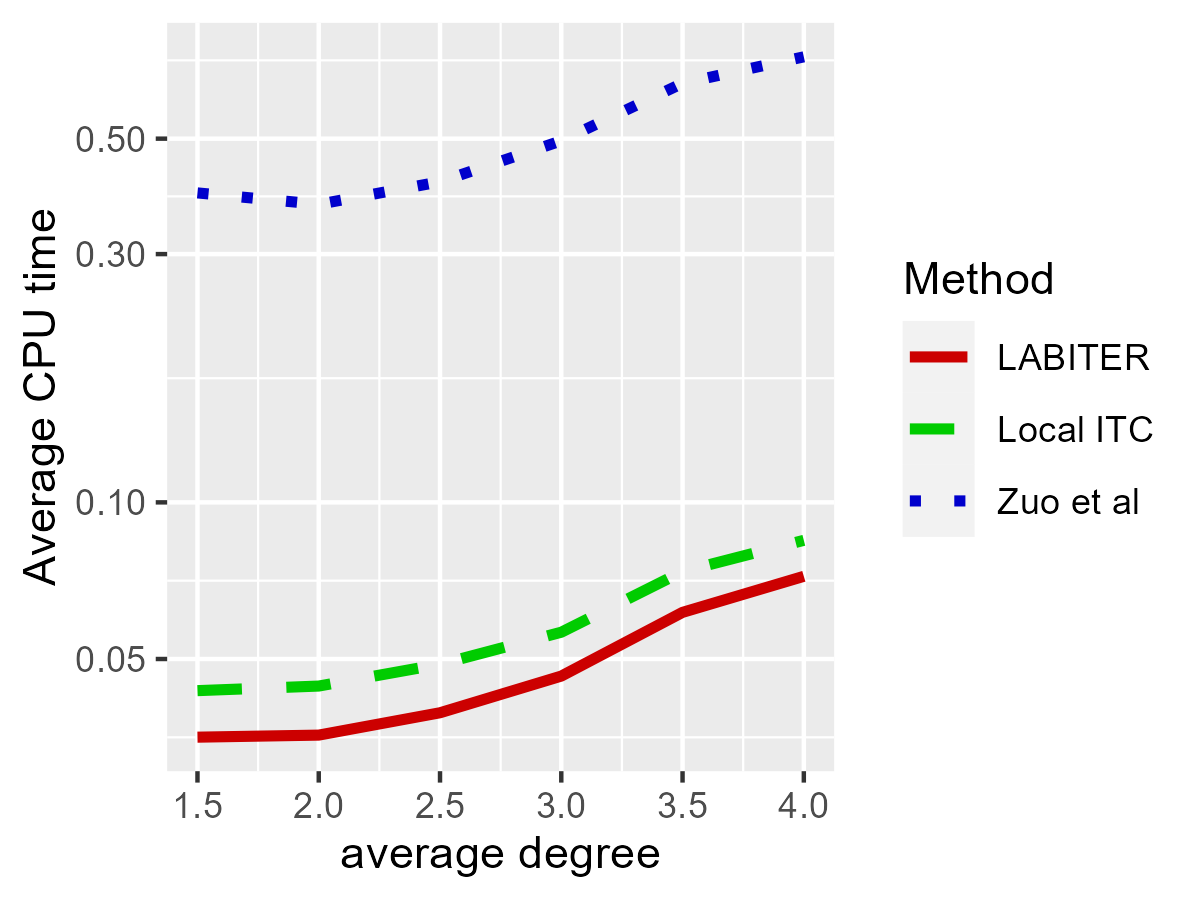}
		\end{minipage}%
	}%
	\hspace{0.01\textwidth}
	\subfloat[$n=50$, $N=1000$ \label{fig:time_50_1000}]{
		\begin{minipage}[t]{0.22\textwidth}
			\centering
			\includegraphics[width=\textwidth]{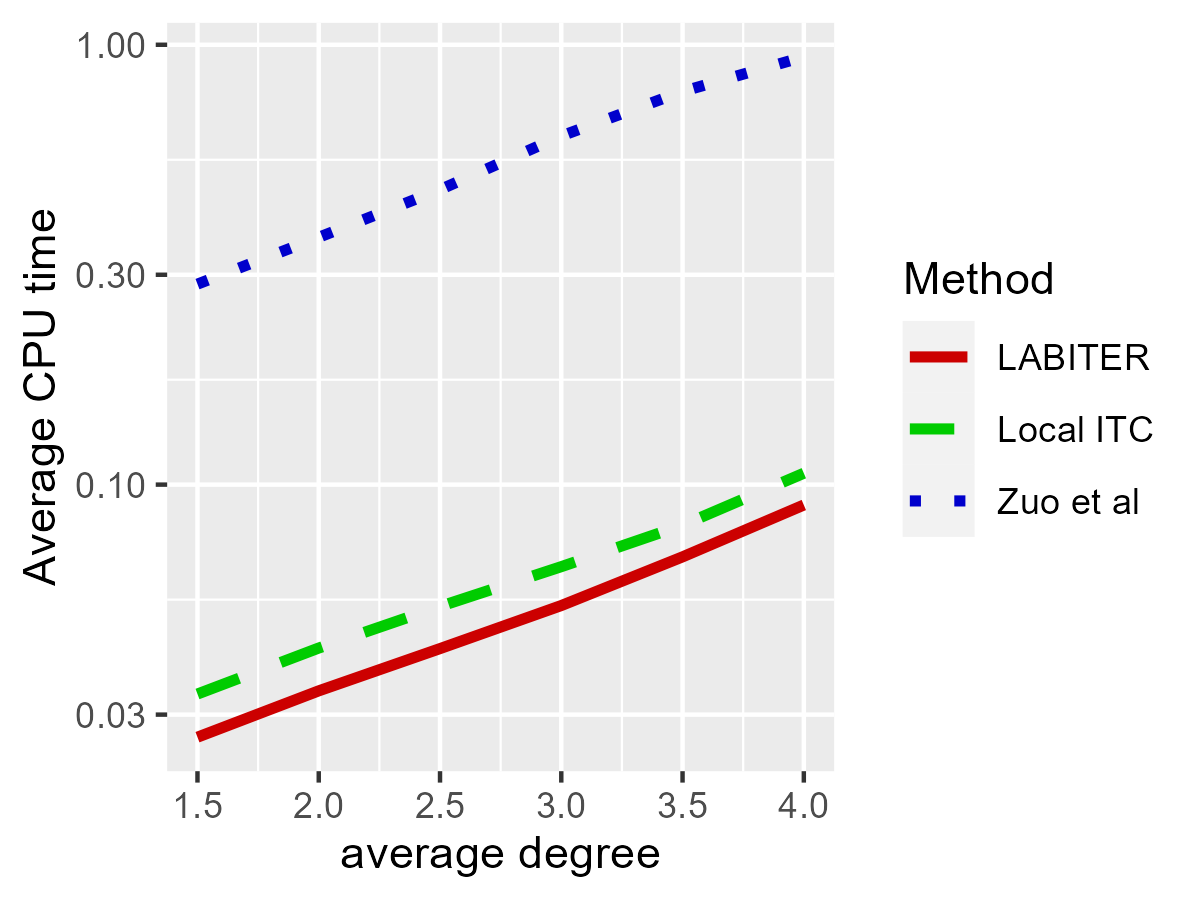}
		\end{minipage}%
	}%
	
	\caption{(a)-(d) The Kappa coefficients of different methods on random graphs with $n$ nodes and $N$ samples are drawn. (e)-(h) The average CPU time of different methods. }
	\label{fig:true:kappa}
\end{figure}

Next, we conduct an experiment to identify causal relations in an MPDAG.
We sample random DAGs $\mathcal{G}$ with $n\in \{50, 100\}$ nodes and average degrees $d\in \{1.5, 2, 2.5, 3, 3.5, 4\}.$ A sample of size $N\in \{100, 200, 500, 1000\}$ is drawn from a linear Gaussian model faithful to $\mathcal{G},$ with weights  randomly drawn from a uniform distribution between $[0.6, 1.2].$
Thirty percent of the edges in $\mathcal{G}$ are chosen as background knowledge.
We compare the performance of three algorithms for identifying causal relations: LABITER (Algorithm~\ref{alg:find descendants}), local ITC (Algorithm 1 in~\citet{fang2022local}) and Algorithm 2 in~\citet{zuo2022counterfactual}, which first learns the whole MPDAG and then identifies causal relations based on Lemma~\ref{lemma:zuo}.

In each experiment, three possible causal relations between two randomly sampled nodes, $X$ and $Y$, are considered: $X$ is a definite cause, a definite non-cause, or a possible cause of $Y.$
Let $I_{ij}(k)$ be the indicator that in the $k$-th experiment, the true causal relation is the $i$-th one and the output is the $j$-th one, where $1\le i,j\le 3$.
The experiment is repeated for $m=5000$ times.
We use the Kappa coefficient~\citep{cohen1960coefficient} as an metric for evaluating the learning of causal relations, which is also used in~\citep{fang2022local}. Denote $M_{ij}=\sum_{k=1}^{m}I_{ij}(k),$ and
$$p=\frac{\sum_{i=1}^{3}M_{ii}}{\sum_{i=1}^{3}\sum_{j=1}^{3}M_{ij}}, \quad q=\frac{\sum_{i=1}^{3}\left(\sum_{j=1}^{3}M_{ij}\right)\cdot \left(\sum_{j=1}^{3}M_{ji}\right)}{\left(\sum_{i=1}^{3}\sum_{j=1}^{3}M_{ij}\right)^2},$$
then the Kappa coefficient is given by
$$\kappa = \frac{p-q}{1-q}.$$
A higher value of $\kappa$ implies better performance for classification. We plot the kappa coefficient for three algorithms in Figure~\ref{fig:true:kappa}(a)-(d) and the corresponding average CPU time in Figure~\ref{fig:true:kappa}(e)-(h). Results for $n=100$ are shown in \ref{append:experiment}.

From the results shown in Figure~\ref{fig:true:kappa}, it is evident that our algorithm (LABITER) outperforms the other algorithms in most cases. Local ITC does not utilize the additional information from background knowledge, which could further determine some possible cause relationships as definite cause or definite non-cause relationships. Although Algorithm 2 in~\citet{zuo2022counterfactual} incorporates background knowledge, its performance is constrained by the accuracy of  CPDAG learning. In contrast, our method not only leverages background knowledge but also reduces errors through local graph learning, resulting in superior performance.

\subsection{Counterfactual fairness}

We also run an experiment on counterfactual fairness, similar to the synthetic data experiment in~\citep{zuo2022counterfactual}, comparing our method with Algorithm 2 in~\citep{zuo2022counterfactual}. We randomly sample DAGs with $n\in\{10,20,30,40\}$ nodes with $2n$ edges. Data are generated from a linear Gaussian model, with each variable uniformly discretized based on its sorted order. The sensitive variable $A$ and the target variable $Y$ are randomly selected. The range of all variables, except the binary sensitive variable $A$, are randomly sampled from the ranges of variables in the UCI Student Performance Dataset~\citep{cortez2008using}. In generating counterfactual data, $A$ is set to the other level while other variables are generated normally. Let $\hat{Y}(a)$ and $\hat{Y}(a')$ represent the estimates of $Y$ in the observed and counterfactual data, respectively. Unfairness is measured as $|\hat{Y}(a) - \hat{Y}(a')|$. According to~\citet{kusner2017counterfactual}, $\hat{Y}$ is counterfactually fair if it depends only on the non-descendants of $A$. To reduce unfairness, we should therefore exclude the descendants of $A$ from the prediction model.

We consider five baseline models as in~\citep{zuo2022counterfactual}. 1) \texttt{Full} uses all variables except $Y$ to predict $Y.$ 2) \texttt{Unaware} uses all variables except $X,Y$ to predict $Y.$ 3) \texttt{Oracle} uses all true non-descendants of $X$ to predict $Y.$ 4) \texttt{FairRelax} first learns the MPDAG, then use Lemma~\ref{lemma:zuo} to identify causal relations, and uses all definite non-descendants and possible descendants of $X$ to predict $Y.$ 5) \texttt{Fair} also identify causal relations as \texttt{FairRelax} does, and uses all definite non-descendants of $X$ to predict $Y.$ We propose two methods: \texttt{LFairRelax} uses LABITER to identify causal relations and uses all definite non-descendants and possible descendants of $X$ to predict $Y,$ and \texttt{LFair} uses LABITER to identify causal relations and uses all definite non-descendants of $X$ to predict $Y.$

\begin{table}[h]
\centering
\label{tab:my-table}
\small
\resizebox{\textwidth}{!}{
\begin{tabular}{ccccccccc}
\hline
\multirow{2}{*}{} & \multirow{2}{*}{Node} & \multicolumn{7}{c}{Method} \\ \cline{3-9} 
 &  & Full & Unaware & Oracle & FairRelax & Fair & LFairRelax & LFair \\ \hline
\multirow{4}{*}{\rotatebox{90}{Unfairness}} & 10 & $0.251\pm0.37$ & $0.118\pm0.288$ & $0.000\pm0.000$ & $0.069\pm0.173$ & $0.058\pm0.162$ & $0.012\pm0.041$ & $0.012\pm0.041$ \\ 
 & 20 & $0.165\pm0.267$ & $0.085\pm0.211$ & $0.000\pm0.000$ & $0.079\pm0.199$ & $0.067\pm0.152$ & $0.035\pm0.11$ & $0.034\pm0.109$ \\ 
 & 30 & $0.099\pm0.156$ & $0.044\pm0.103$ & $0.000\pm0.000$ & $0.055\pm0.281$ & $0.053\pm0.279$ & $0.028\pm0.157$ & $0.029\pm0.166$ \\ 
 & 40 & $0.098\pm0.18$ & $0.054\pm0.154$ & $0.000\pm0.000$ & $0.054\pm0.17$ & $0.026\pm0.077$ & $0.025\pm0.118$ & $0.025\pm0.122$ \\ \hline
\multirow{4}{*}{\rotatebox{90}{RMSE}} & 10 & $0.431\pm0.189$ & $0.441\pm0.193$ & $0.513\pm0.268$ & $0.491\pm0.232$ & $0.518\pm0.251$ & $0.591\pm0.312$ & $0.59\pm0.312$ \\ 
 & 20 & $0.529\pm0.434$ & $0.537\pm0.438$ & $0.628\pm0.643$ & $0.623\pm0.559$ & $0.69\pm0.646$ & $0.794\pm0.919$ & $0.794\pm0.919$ \\ 
 & 30 & $0.732\pm0.745$ & $0.731\pm0.744$ & $0.911\pm1.118$ & $0.816\pm0.801$ & $0.892\pm0.95$ & $0.891\pm0.91$ & $0.891\pm0.91$ \\ 
 & 40 & $0.703\pm0.72$ & $0.708\pm0.724$ & $0.849\pm1.001$ & $0.774\pm0.819$ & $0.843\pm0.912$ & $0.947\pm1.092$ & $0.947\pm1.093$ \\ \hline
\end{tabular}
}
\caption{Average RMSE and unfairness for each method in the counterfactual fairness experiment.}
\end{table}

We draw $N=1000$ samples in each experiment and use all methods to obtain $\hat{Y}$ from observational data. The experiment is repeated for $m=100$ times. We collect the RMSE on the test set and unfairness for $\hat{Y}$ for each method, and show them in Table 1. We can see that as RMSE is similar across different methods, out method can achieve a lower unfairness. 

\section{Application}
\label{sec:application}

We apply our method to the bank marketing dataset\footnote{Data download and data information are available at \url{https://archive.ics.uci.edu/ml/datasets/Bank+Marketing}.}~\citep{moro2014data}. The Bank dataset pertains to direct marketing campaigns conducted by a Portuguese banking institution, primarily through phone calls. These campaigns often required multiple contacts with the same client to determine if they would subscribe to a bank term deposit. The dataset's classification goal is to predict whether a client will subscribe to a term deposit, indicated by the binary variable \texttt{y}. Existing studies have recognized marital status as a protected attribute~\citep{chierichetti2017fair, backurs2019scalable, bera2019fair, hu2020fairnn, ziko2021variational}. Therefore, to predict \texttt{y} under the premise of counterfactual fairness, we should exclude descendant variables of marital status from the predictive variables. Due to the absence of counterfactual data to measure unfairness, we merely identify causal relationships between other variables and marital status and provide discussion accordingly.

\begin{figure}[t]
\centering
\includegraphics[width=0.5\textwidth]{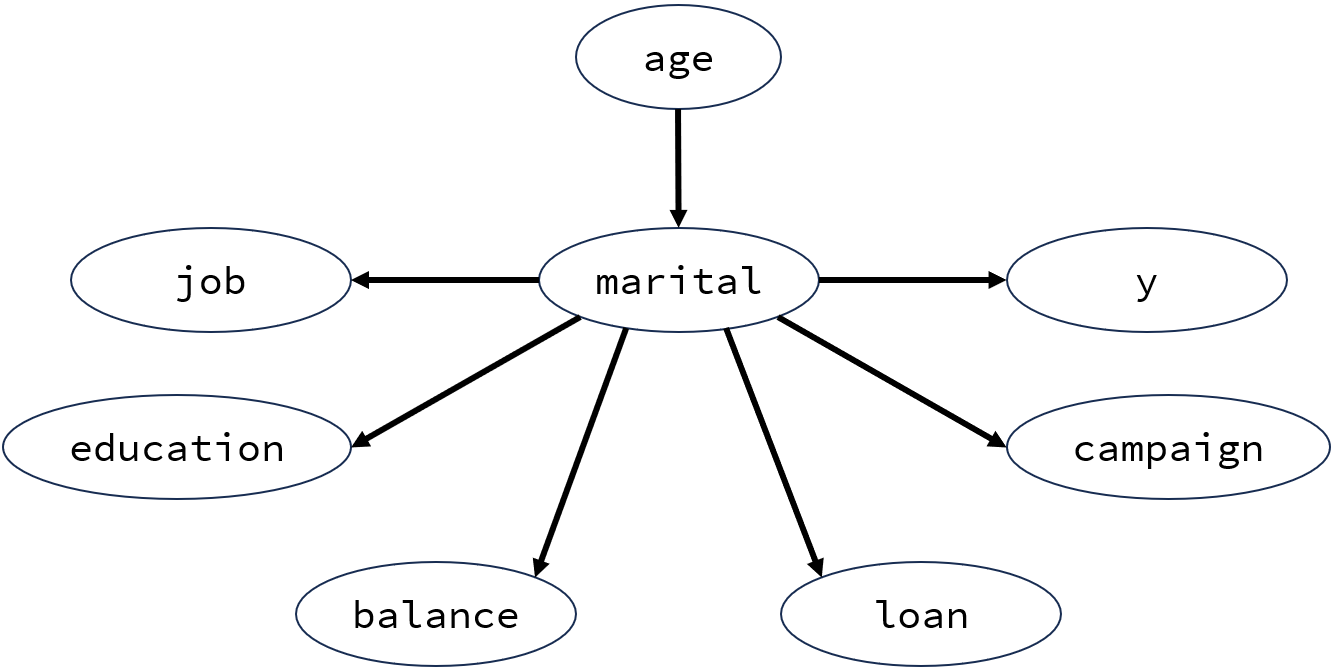}

\caption{Learned local structure from the Bank dataset. Marital status is directly affected by age, and it directly affects job, education, balance, personal loan, campaign (number of contacts performed during this campaign and for this client) and \texttt{y}.} 
\label{fig:application}
\end{figure}

\begin{table}[t]
    \centering
    \resizebox{\textwidth}{!}{
    \begin{tabular}{|c|c|c|c|c|}
\hline
\multicolumn{2}{|c|}{LABITER} & \multicolumn{2}{|c|}{Zuo et al.} \\
\hline
Definite descendant & Definite non-descendant & Definite descendant & Definite non-descendant \\
\hline
\textbf{job} & age & education & age \\
education & default & balance & \textbf{job} \\
balance & day & housing & default \\
housing & duration & loan & day \\
loan & previous & contact & \textbf{month} \\
contact & poutcome & y & duration \\
\textbf{month} & & & \textbf{campaign} \\
\textbf{campaign} & & & \textbf{pdays} \\
\textbf{pdays} & & & previous \\
y & & & poutcome \\
\hline
\end{tabular}
}
    \caption{Identified causal relations for marital status by running LABITER (left two columns) and Algorithm 2 in~\citet{zuo2022counterfactual} with learned MPDAG (right two columns). Differences are marked with bold.}
    \label{tab:application}
\end{table}

Among the variables in the dataset, age is not influenced by other variables, which we consider as background knowledge. Under this premise, we use Algorithm~\ref{alg:mb-by-mb in MPDAG} to learn the local structure around \texttt{marital}, with results shown in Figure~\ref{fig:application}. After that, we apply two algorithms, LABITER (Algorithm~\ref{alg:find descendants}) and Algorithm 2 in~\citet{zuo2022counterfactual}, to find definite descendants and definite non-descendants of marital status, with results shown in Table~\ref{tab:application}. Note that both algorithms do not identify any variable as a possible descendant of marital status.

We can see that, LABITER identifies job, month, campaign and pdays as definite descendants of marital status, while Algorithm 2 in~\citet{zuo2022counterfactual} identifies them as definite non-descendants. Specifically, it is plausible to argue that marital status can significantly influence one's occupation or job type. Therefore, our algorithm's recognition of job as a descendant of marital status is more in line with real-world causal dynamics.

\section{Discussion}
\label{sec:discussion}
In this paper, we propose a local learning approach for finding definite descendants, possible descendants and definite non-descendants of a certain variable in an MPDAG. We first propose an algorithm learning the local structure around any variable in the MPDAG. The algorithm find the adjacent nodes of the variable, distinguish them by parents, children, and siblings in the MPDAG, and provide similar results for its sibling nodes. Based on the local structure, we propose a criterion that identifies the definite descendants, definite non-descendants and possible descendants of the target variable. Different from similar results in~\citep{zuo2022counterfactual} which needs to learn the whole MPDAG and check paths on it, our criterion only need the local structure and some conditional independent tests. Moreover, our criterion can further distinguish definite cause relationship into explicit cause and implicit cause relationship, defined as whether the causal path between two nodes remains the same for every DAG represented by the MPDAG.

Our work has made some progress in the theoretical understanding of MPDAG. Combining with previous studies~\citep{fang2020ida,zuo2022counterfactual}, it shows that even if MPDAG is not chain graph in general, most results for CPDAG can also be extended to MPDAG. MPDAG has a wider use than CPDAG in practice, especially for scenarios where most part of causal graph is known as prior from expert knowledge. Our results for judging causal relations can also be applied to fairness machine learning, which alleviates the issue of estimation errors for learning causal graphs by learning a local structure around the target variable.

A possible direction for future work is to consider unmeasured confounding variables and selection bias, which can be characterized using partial ancestral graphs~\citep{richardson2003causal, zhang2008completeness}. Recently,~\citet{xie2024local} proposed an algorithm to learn the local structure in a partial ancestral graph. Based on that algorithm, there may be similar criteria for identification of causal relationships in a partial ancestral graph.

\bibliographystyle{apalike}
\bibliography{ref}

\appendix

\section{Graph terminologies and definitions}
\label{append:graph}
In this section, we review the graph terminologies and give the detailed definitions. Some definitions used in proofs but not in the main text are also listed.

A graph $G=(\mathbf{V},\mathbf{E})$ is a tuple with node set $\mathbf{V}$ and edge set $\mathbf{E}\subseteq \mathbf{V}\times \mathbf{V},$ where $\times$ denotes the Cartesian product. For two distinct nodes $X,Y\in \mathbf{V},$ we say that there is a directed edge from $X$ to $Y$ in $G$ if $(X,Y)\in \mathbf{E}$ and $(Y,X)\not\in \mathbf{E},$ denoted by $X\to Y\in G.$ If $(X,Y)\in \mathbf{E}$ and $(Y,X)\in \mathbf{E},$ we say that there is a undirected edge between $X$ and $Y$ in $G,$ denoted by $X-Y\in G.$ Two nodes $X$ and $Y$ are adjacent in $G$ if $X\to Y$ or $Y\to X$ or $X-Y$ in $G.$ 

A path $p=\langle V_1,\dots,V_k \rangle$ is a sequence of nodes where $V_i$ and $V_{i+1}$ are adjacent in $G$ for $i=1,2,\dots,k-1.$ We say that $p$ is causal (or directed) if $V_i\to V_{i+1}\in G$ for each $i.$ We say that $p$ is partially directed if there is no $i$ such that $V_i\leftarrow V_{i+1}\in G.$
We call $p$ a cycle if $V_1=V_k.$ If a cycle is causal, we call it a directed cycle. A graph $G$ is directed (or undirected) if all of its edges are directed (or undirected). A directed acyclic graph (DAG) is a directed graph which does not contain any directed cycle. 
For two paths $p=\langle V_1,\dots,V_k \rangle$ and $q=\langle W_1,\dots,W_l\rangle,$ let $p\oplus q=\langle V_1,\dots,V_k,W_1,\dots,W_l\rangle$ denote the connected path of $p$ and $q.$

For two distinct nodes $X,Y$ in a graph $G,$ $X$ is a parent (children, sibling, ancestor, descendant) of $Y$ if there is a directed edge from $X$ to $Y$ (directed edge from $Y$ to $X$, undirected edge between $X$ and $Y,$ directed path from $X$ to $Y$, directed path from $Y$ to $X$). They are adjacent if there is a directed or undirected edges between $X$ and $Y$. We use the convention that each node is an ancestor and descendant of itself. The set of parents, children, siblings, ancestors, descendants and adjacent nodes of $X$ are denoted as $pa(X,G),ch(X,G),sib(X,G), \\an(X,G),de(X,G),adj(X,G).$ A collider is a triple $(X,Y,Z)$ where $X\to Y\leftarrow Z,$ and it is also called a v-structure if $X$ and $Z$ are not adjacent. In a path $p,$ $V_i$ is a collider if $V_{i-1}\to V_i\leftarrow V_{i+1},$ otherwise $V_i$ is called a non-collider. A path $p$ is open given node set $\mathbf{S}$ if none of its non-collider is in $\mathbf{S}$ and all its colliders have a descendant in $\mathbf{S}.$ Otherwise, $p$ is blocked by $\mathbf{S}.$ Two nodes $X,Y$ is d-separated by $\mathbf{S}$ if all paths between $X$ and $Y$ are blocked by $\mathbf{S}.$ Otherwise, they are d-connected given $\mathbf{S}.$ Let $X\perp Y |\mathbf{S}$ denote that $X,Y$ are d-separated by $\mathbf{S}.$

Different DAGs may encode same d-separations. For example, $X\to Y$ and $X\leftarrow Y$ both imply that $X,Y$ are d-connected given the empty set. Two DAGs that encode the same d-separations are called to be Markov equivalent. Let $\mathcal{G}$ be a DAG. Let $[\mathcal{G}]$ denote the set of DAGs that are Markov equivalent with $\mathcal{G},$ which is also called the Markov equivalence class (MEC) of $\mathcal{G}.$ The MEC $[\mathcal{G}]$ can be represented by a complete partially directed acyclic graph (CPDAG) $\mathcal{C},$ which shares the same skeleton with $\mathcal{G},$ and any edge in $\mathcal{C}$ is directed if and only if it has the same direction in all DAGs in $[\mathcal{G}].$ 

Under a set of background knowledge $\mathcal{B}$ consisting of direct causal information with form $V_i\to V_j,$ the MEC is restricted to a smaller set which contains Markov equivalent DAGs that are consistent with background knowledge $\mathcal{B},$ that is, all background knowledge $V_i\to V_j$ are correctly oriented in all DAGs in this set. Just like CPDAG, this restricted set can be represented by a partially directed graph $\mathcal{G}^*,$ which is called the maximally partially directed acyclic graph (MPDAG) of $\mathcal{G}$ with background knowledge $\mathcal{B}.$ 

\section{Proofs}
\label{append:proof}
\subsection{Existing results and lemmas}

\begin{lemma}
\label{lemma:mpdag def}
(Theorem 4 in~\citet{meek1995causal}) Let $\mathcal{G}^*$ be an MPDAG and $[\mathcal{G}^*]$ be the restricted Markov equivalence class represented by $\mathcal{G}^*.$ Let $X,Y$ be two distinct nodes in $\mathcal{G}^*,$ then $X\to Y\in \mathcal{G}^*$ if and only if $X\to Y\in \mathcal{G}$ for any $\mathcal{G}\in [\mathcal{G}^*],$ and $X-Y\in \mathcal{G}^*$ if and only if there exists two DAGs $\mathcal{G}_1,\mathcal{G}_2\in [\mathcal{G}^*]$ such that $X\to Y\in\mathcal{G}_1$ and $X\leftarrow Y\in \mathcal{G}_2.$
\end{lemma}

\begin{lemma}
\label{lemma: perkovic}
(Lemma B.1 in \citep{perkovic2017interpreting}) Let $p=\langle V_1,V_2,\dots,V_k\rangle$ be a b-possibly causal definite path in a MPDAG $M.$ If there is $i\in \{1,2,\dots,n-1\}$ such that $V_i\to V_{i+1}$ in $M,$ then $p(V_i,V_k)$ is causal in $M.$
\end{lemma}

\begin{lemma}
\label{lemma:partial cycle}
Let $X,Y$ be two distinct nodes in an MPDAG $\mathcal{G}^*.$ If $X \in an(Y,\mathcal{G}^*)$ and $X,Y$ are adjacent, then $X\to Y\in \mathcal{G}^*.$ 
\end{lemma}
\begin{proof}
Let $\mathcal{G}$ be any DAG in $[\mathcal{G}^*]$ By Lemma~\ref{lemma:mpdag def}, $X\in an(Y,\mathcal{G})$ and $X,Y$ are adjacent in $\mathcal{G}.$ If $X\leftarrow Y\in \mathcal{G},$ it will form a directed cycle, which contradicts with the definition of DAG. Therefore, $X\to Y\in \mathcal{G},$ and by Lemma~\ref{lemma:mpdag def} we have $X\to Y\in \mathcal{G}^*.$ 
\end{proof}
\begin{corollary}
\label{cor:partial cycle}
Let $X,Y$ be two distinct nodes in an DAG $\mathcal{G}.$ If $X\in an(Y,\mathcal{G})$ and $X,Y$ are adjacent, then $X\to Y\in \mathcal{G}.$ \qedsymbol
\end{corollary}

\begin{lemma}
\label{lemma:shortest+collider}
Let $\mathbf{X}, \mathbf{Y}$ and $\mathbf{S}$ be disjoint node sets in a DAG $\mathcal{G}$ such that $\mathbf{X}\perp \mathbf{Y} | \mathbf{S}$ in $\mathcal{G}.$ Let $p$ be the shortest d-connecting path between $\mathbf{X}$ and $\mathbf{Y}$ given $\mathbf{S}.$ Then there is no shielded collider on $p.$ Namely, there does not exist $A,B,C$ in $p$ such that $A\to B\leftarrow C$ in $\mathcal{G}$ and $A,C$ are adjacent in $\mathcal{G}.$
\end{lemma}	
\begin{proof}
Suppose such $A,B,C$ in $p$ exists, write $p=p(L,A)\oplus (A,B,C)\oplus p(C,R),$ where $L\in\mathbf{X}$ and $R\in\mathbf{Y}.$ We claim that $q=p(L,A)\oplus (A,C)\oplus p(C,R)$ is also d-connecting given $\mathbf{S},$ which leads to a contradiction with the fact that $p$ is the shortest. It is clear that other nodes except $A,C$ on $q$ have the same collider/non-collider status on $q$ and $p,$ so we just need to consider $A,C$ on $q.$ Since $B$ is a collider on $p$ and $p$ is d-connecting given $\mathbf{S},$ we have $B\in an(\mathbf{S},\mathcal{G})$ and hence $A\in an(\mathbf{S},\mathcal{G})$ by $A\to B$ in $\mathcal{G}.$ Therefore, if $A$ is a collider on $q,$ it cannot block $q.$ Since $A$ is a non-collider on $p$ and $p$ is d-connecting given $\mathbf{S},$ we have $A\not\in\mathbf{S},$ so if $A$ is a non-collider on $q,$ it also cannot block $q.$ The same reasoning can be applied to conclude that $C$ cannot block $q.$ Therefore, $q$ is d-connecting given $\mathbf{S},$ which leads to a contradiction.
\end{proof}

\subsection{Proof for Theorem~\ref{thm:mb-by-mb}}

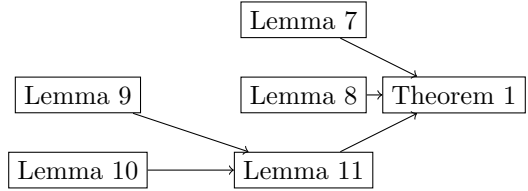
\begin{figure}[ht]
    \centering
    \begin{tikzpicture}[node distance=1.5cm]
        \node at (3, 2) [draw] (lemma4) {Lemma~\ref{lemma:adj}};
        \node at (3, 1) [draw] (lemma5) {Lemma~\ref{lemma:sound}};
        \node at (0, 1) [draw] (lemma6) {Lemma~\ref{lemma:exist DEGS}};
        \node at (0, 0) [draw] (lemma7) {Lemma~\ref{lemma:complete induction}};
        \node at (3, 0) [draw] (lemma8) {Lemma~\ref{lemma:complete}};
        \node at (5, 1) [draw] (theorem1) {Theorem~\ref{thm:mb-by-mb}};
        
        \draw[->] (lemma4) -- (theorem1);
        \draw[->] (lemma5) -- (theorem1);
        \draw[->] (lemma6) -- (lemma8);
        \draw[->] (lemma7) -- (lemma8);
        \draw[->] (lemma8) -- (theorem1);
    \end{tikzpicture}
    \caption{Proof structure of Theorem~\ref{thm:mb-by-mb}.}
    \label{fig:proof for thm1}
\end{figure}

Figure~\ref{fig:proof for thm1} shows how lemmas fit together to prove Theorem~\ref{thm:mb-by-mb}. Theorem~\ref{thm:mb-by-mb} directly follows from Lemma~\ref{lemma:adj},\ref{lemma:sound},\ref{lemma:complete}. Then we proof the lemmas shown in this structure. The first two lemmas show the soundness of Algorithm~\ref{alg:mb-by-mb in MPDAG}.

\begin{lemma}
\label{lemma:adj}
Suppose the joint distribution is markovian and faithful with the underlying DAG , and all independence tests are correct, then the MB-by-MB in MPDAG algorithm can correctly discover the edges which are connected to each node in DoneList without their orientation.
\end{lemma}
\begin{proof}
According to Theorem 1 in \citep{wang2014discovering}, for each node $X$ in DoneList, with a correct $\mathrm{MB}(X)$ such that $X \perp (V\setminus \mathrm{MB}^+(X))|\mathrm{MB}(X),$ for any $Y\in V,$ $Y$ and $X$ are d-separated by a subset of $V$ if and only if $Y$ and $X$ are d-separated by a subset of $\mathrm{MB}(X).$ So with conditional independence set learned by observed data of $\mathrm{MB}^+(X),$ whether $X$ and $Y$ are connected could be correctly discovered.
\end{proof}

\begin{lemma}
\label{lemma:sound}
Suppose the joint distribution is markovian and faithful with the underlying DAG , and all independence tests are correct, then each directed edge in the output graph of MB-by-MB in MPDAG algorithm is correct, i.e. they are also directed in the true MPDAG with the same direction.
\end{lemma}
\begin{proof}
Denote $M$ to be the true MPDAG and $G$ the output PDAG of the algorithm. We show that each directed edge in output $G$ is correct, in the sense that it is also a directed edge in $M$ and has the same orientation in $G$ and $M.$ There are three sources of directed edges in $G$: background knowledge, learned by a v-structure or oriented by one of four Meek rules. Since $G$ is initialized by background knowledge, each directed edge from background knowledge in $G$ is correct. According to Theorem 2 in \citep{wang2014discovering}, each directed edge learned by a v-structure is correct.

Consider a directed edge in $G$ oriented by one of four Meek rules. By induction, we can assume the directed edges needed by a Meek rule is correct since they finally comes from a v-structure or a background knowledge. (1) For rule R1, if $a\to b-c \in G,$ we know that $a\to b\in M$, and $b,c$ are adjacent in $M.$ Since $(a,c,S_{ac})\in\,$IndSet and $b\in S_{ac},$ we know that $a,c$ are not adjacent in $M,$ and $a\to b\leftarrow c\not\in M.$ So $b\to c\in M$ by R1. (2) For rule R2, if $a\to b\to c-a\in G,$ we know that $a\to b\to c\in M$, and $a,c$ are adjacent in $M.$ So $a\to c\in M$ by R2. (3) For rule R3, if $a-b,\,a-c\to b,\,a-d\to b\in G,$ we know that $c\to b,\,d\to b\in M,$ and $(a,b),(a,c),(a,d)$ are adjacent pairs in $M.$ Since $(c,d,S_{cd})\in\,$IndSet and $a\in S_{cd},$ we know that $c,d$ are not adjacent in $M,$ and $c\to a\leftarrow d\not\in M.$ Enumerate all structure about $(c,a,d)$ in $M:$ (3.1) $a\to c,$ then $a\to b$ by R2; (3.2) $a\to d,$ then $a\to b$ by R2; (3.3) $c-a-d,$ then $a\to b$ by R3. So $a\to b\in M.$ (4) For rule R4, if $a-b,\,a-c\to b,\,a-d\to c\in G,$ we know that $d\to c\to b\in M,$ and $(a,b),(a,c),(a,d)$ are adjacent pairs in $M.$ Since $(b,d,S_{bd})\in\,$IndSet and $a\in S_{bd},$ we know that $b,d$ are not adjacent in $M$ and $b\to a\leftarrow d\not\in M.$ Enumerate all structure about $(c,a,d)$ in $M:$ (4.1) $c\to a,$ then $d\to a$ by R2 and $a\to b$ by R1; (4.2) $a\to c,$ then $a\to b$ by R2; (4.3) $c-a-d,$ then $a\to b$ by R4; (4.4) $c-a\to d$, which cannot exist since $a\to c$ by R2; (4.5) $c-a\leftarrow d,$ then $a\to b$ by R1. So $a\to b\in M.$ Until now we verified that all directed edges oriented by one of four Meek rules in MB-by-MB algorithm is a correct directed edge in $M.$ By induction we know all directed edges in $G$ is correct in $M.$
\end{proof}

Then we define the directed edge generation sequence of an MPDAG, which represents the order and way for each directed edge in an MPDAG to be oriented. As in the following definition, each directed edge is marked as one of the following labels: "V", "B", "R1", "R2", "R3", "R4", which respectively implies that this edge is oriented in a v-structure, by background knowledge, or by Meek's rules R1-R4. 

The motivation of this definition is that we need to prove that all directed edges around the target node are discovered by Algorithm~\ref{alg:mb-by-mb in MPDAG}. Actually, this fact can be explained by simple words: since all directed edges are oriented by v-structure, background knowledge, or one of four Meek's rules, and all nodes connect with the target node by an undirected path are explored by the Algorithm, continuing the Algorithm for exploring other nodes, which are blocked from the target node by directed edges, cannot help to orient the direction for undirected edges. If the reader believes that these words proves that all directed edges around the target node are discovered, there is no need to read the following proof carefully. 

The following lemmas lies in the intuition that if an directed edge in the MPDAG is not discovered in Algorithm~\ref{alg:mb-by-mb in MPDAG}, then another directed edge, which participated in the orientation of this edge, is neither discovered. By iteration, at last an directed edge which is in a v-structure in which all nodes are explored, or is itself a background knowledge, is not discovered, which conflicts with previous lemmas. Therefore, we need to define the "orientation procedure" of directed edges in an MPDAG.

\begin{definition}
(Directed edge generation sequence(DEGS) of an MPDAG) Let $M=(V,E_M)$ be an MPDAG, and let $E^d_M$ be its directed edges and $k=|E^d_M|$. A directed edge generation sequence of an MPDAG $M$ is a sequence of tuples $S=\langle(d_1,p_1,r_1),\dots,(d_k,p_k,r_k) \rangle.$ For each $i=1,2,\dots,k,$ $d_i\in E^d_M$ is a directed edge in $M$ and $E^d_M=\{d_1,d_2,\dots,d_k\}.$ $p_i$ is a set of directed edges, and $r_i$ is a label which takes one of these values: $\{"V","B","R1","R2","R3","R4"\}.$
\end{definition}

\begin{algorithm}[t]
\LinesNumbered
\caption{Check the consistency of a DEGS corresponding to an MPDAG.}
\label{alg:check DEGS}
\SetKwInOut{Input}{Input}
\SetKwInOut{Output}{Output}
\small

\Input{An MPDAG $M=(V,E_M)$, a DEGS $S=\langle (d_1,p_1,r_1),\dots,(d_k,p_k,r_k) \rangle$, a set of background knowledge $B\subset E_M$}
\Output{A bool variable indicates whether $S$ is consistent with $M,B$.}

$H$ = the skeleton of $M;$ ($H=(V,E_H)$ is a PDAG) \\
\ForEach{$i$ in $\{1,2,\dots,k\}$}{
\If{$p_i \not\subset E_H$ or $d_i\in E_H$}{
\Return FALSE;
}
\Switch{$r_i$}{
\uCase{"V"}{
\uIf{$p_i\neq \emptyset$ or $d_i$ is not contained in a v-structure of $M$}{\Return FALSE;}
}
\uCase{"B"}{
\uIf{$p_i\neq \emptyset$ or $d_i\not\in B$}{\Return FALSE;}
}
\uCase{"R1"}{
Find $a,b,c\in V$ such that $a\to b-c\in H,$ $d_i=b\to c, p_i=\{a\to b\}$ and $a,c$ are not adjacent in $H$. \\
\uIf{such $a,b,c$ does not exist}{\Return FALSE;}
}
\uCase{"R2"}{
Find $a,b,c\in V$ such that $a\to b\to c\in H,a-c\in H,d_i=a\to c,p_i=\{a\to b,b\to c\}.$ \\
\uIf{such $a,b,c$ does not exist}{\Return FALSE;}
}
\uCase{"R3"}{
Find $a,b,c,d\in V$ such that $b-a-c\in H,b\to d\leftarrow c\in H,a-d\in H,d_i=a\to d,p_i=\{b\to d,c\to d\},$ and $(b,a,c)$ is not a v-structure in $M.$ \\
\uIf{such $a,b,c,d$ does not exist}{\Return FALSE;}
}
\uCase{"R4"}{
Find $a,b,c,d\in V$ such that $b-a-c\in H,c\to d\to b\in H,a-d\in H,d_i=a\to b,p_i=\{c\to d\},$ and $(b,a,c)$ is not a v-structure in $M.$ \\
\uIf{such $a,b,c,d$ does not exist}{\Return FALSE;}
}
}
Add $d_i$ to $E_H.$
}
\If{$H\neq M$}{\Return FALSE;}
\Return TRUE;
\end{algorithm}

The consistency of DEGS $S$ with respect to an MPDAG $M$ and a set of background knowledge $B$ is defined as whether $M$ can be recovered from its skeleton with its directed edges sequentially oriented by v-structure, background knowledge, or one of the four Meek's rules. The sequential orientation procedure is described by $S.$ Definition~\ref{def:consistency of DEGS} gives a strict description of this nature.

\begin{definition}
\label{def:consistency of DEGS}
(Consistency of DEGS) Let $M=(V,E_M)$ be an MPDAG with background knowledge $B\subset E_M.$ A DEGS $S=\langle(d_1,p_1,r_1),\dots,(d_k,p_k,r_k) \rangle$ is consistent with $M,B$, if Algorithm~\ref{alg:check DEGS} returns TRUE for input $M,S,B.$
\end{definition}

\begin{lemma}
\label{lemma:exist DEGS}
For any MPDAG $M=(V,E_M)$ with background knowledge $B,$ there exists at least one DEGS $S$ that is consistent with $M,B.$
\end{lemma}
\begin{proof}
From Proposition 1 in \citep{fang2022representation}, there exists an CPDAG $\mathcal{C}$ such that $[M]$ is the restricted set of $[\mathcal{C}]$ under background knowledge $B.$ As defined in~\citep{meek1995causal}, the pattern $\Pi$ of $\mathcal{C}$ is the partially directed graph which shares the same skeleton with $\mathcal{C}$ and has an directed edge if and only if that edge is in a v-structure. By Theorem 4 in \citep{meek1995causal}, $M$ can be obtained by applying Meek's rules R1-R4 and orienting edges according to $B$ in $\Pi.$ 

Let $S$ be a sequence of two parts. The first part contains all $(d,p,r)$ such that $d$ is a directed edge in $\Pi,$ $p=\emptyset$ and $r=$"V". The second part is a sequence of $(d,p,r)$ along with the procedure of applying Meek's rules R1-R4 and orienting edges according to $B$ in $\Pi.$ In each step orienting an edge $d_i$, if $d_i$ is oriented according to $B$, let $(d,p,r)=(d_i,\emptyset,"B").$ If $d_i$ is oriented by R1, then there exists $a,b,c\in V$ such that $a\to b-c$ before orienting $d_i,$ $d_i=b\to c,$ and $a,c$ are not adjacent in $M,$ then let $(d,p,r)=(d_i,\{a\to b\},"R1").$ If $d_i$ is oriented by R2, then there exists $a,b,c\in V$ such that $a\to b\to c$ and $a-c$ before orienting $d_i=a\to c$, then let $(d,p,r)=(d_i,\{a\to b,b\to c\},"R2").$ If $d_i$ is oriented by R3, then there exists $a,b,c,d\in V$ such that $b-a-c,b\to d\leftarrow c,a-d$ before orienting $d_i=a\to d$, then let $(d,p,r)=(d_i,\{b\to d,c\to d\},"R3").$ If $d_i$ is oriented by R4, then there exists $a,b,c,d\in V$ such that $b-a-c,c\to d\to b,a-d$ before orienting $d_i=a\to b,$ then let $(d,p,r)=(d_i,\{c\to d\},"R4").$ It can be verified that Algorithm~\ref{alg:check DEGS} returns TRUE by inputting $M,S,B.$
\end{proof}

\begin{lemma}
\label{lemma:complete induction}
Let $M=(V,E_M)$ be an MPDAG with background knowledge $B\subset E_M.$ Let $\mathbf{S}=\{S_1,\dots,S_q\}$ be the set of all DEGS that is consistent with $M,B.$ Let $G=(V,E_G)$ be the output PDAG of Algorithm~\ref{alg:mb-by-mb in MPDAG} with target $T$. Let $a,b\in V$ be two distinct nodes which are connected to $T$ with an undirected path in $G$. If $a\to b\in M$ but $a-b\in G,$ then for any $i=1,2,\dots,q,$ let $(d,p,r)\in S_i$ such that $d=a\to b,$ there exists $d'=a'\to b'\in p$ such that $(a',b')\neq (a,b), d'\in M,d'\not\in G,$ and $a',b'$ are connected to $T$ with an undirected path in $G.$
\end{lemma}
\begin{proof}
Since $a\to b\in M,$ by the definition of consistent DEGS there exists $(d,p,r)\in S_i$ such that $d=a\to b.$ We firstly prove that $r\not\in \{"V","B"\},$ so $p\neq \emptyset.$ If $r="V",$ there exists a v-structure in $M$ which contains $d,$ i.e. there exists $c\in V$ such that $a\to b\leftarrow c\in M$ and $a,c$ are not adjacent in $M.$ However, since $a$ and $T$ are connected by a undirected path in $G,$ by Algorithm~\ref{alg:mb-by-mb in MPDAG} $a$ is in DoneList, so the v-structures contains $a$ as a parent node is correctly learned in $G,$ so $a\to b\in G,$ which contradicts with the assumption. If $r="B",$ we know $d\in B$ so $a\to b\in G,$ which contradicts with the assumption.

Then we talk about each case of $r.$ If $r="R1",$ there exists $c\in V$ such that $c\to a\to b\in M,$ and $c,b$ are not adjacent in $M,$ and $p=\{c\to a\}.$ Since $a,b$ is in DoneList, the adjacency of $a,b$ in $M$ is the same in $G,$ so $c,b$ are not adjacent in $G,$ and $(c,b,S_{cb})\in\,$IndSet, $a\in S_{cb}.$ Moreover, $a,c$ are adjacent in $G.$ If $c\to a\in G,$ then $a\to b\in G$ by R1, which contradicts with $a-b\in G,$ so $c-a\in G.$ So $d'=c\to a\in p$ satisfies the conclusion of lemma.

If $r="R2",$ there exists $c\in V$ such that $a\to c\to b\in M,$ and $p=\{a\to c, c\to b\}.$ Since $a,b$ are in DoneList, their adjacency are correctly learned so $(a,c),(c,b)$ are adjacent pairs in $G.$ If $a\to c\to b\in G,$ then $a\to b\in G$ by R2, which contradicts with $a-b\in G,$ so one of them must be undirected in $G.$ The undirected one satisfies the conclusion of lemma.

If $r="R3",$ there exists $c,d\in V$ such that $c\to b\leftarrow d\in M,$ $(a,c),(a,d)$ are adjacent pairs in $M,$ $c,d$ are not adjacent in $M,$ $(c,a,d)$ is not a v-structure in $M,$ and $p=\{c\to b, d\to b\}.$ Since $a,b$ are in DoneList, their adjacency are correctly learned so $(a,c),(a,d),(b,c),(b,d)$ are adjacent pairs in $G.$ If $c\to b\leftarrow d\in G,$ we enumerate all possible states of edges between $(a,c)$ and $(a,d)$ in $G.$ (1) They are both undirected edges in $G.$ Then $c$ is in DoneList, so $c,d$ is not adjacent in $G$, $(c,d,S_{cd})\in\,$IndSet and $a\in S_{cd},$ so $a\to b\in G$ by R3, which contradicts with $a-b\in G.$ (2) They are both directed edges in $G.$ Then $a\to c$ or $a\to d$ in $G,$ otherwise they will form a v-structure, contradicts with that they do not form a v-structure in $M.$ Without loss of generality suppose $a\to c\in G,$ then by $c\to b\in G$ and R2 we have $a\to b\in G,$ which contradicts with $a-b\in G.$ (3) Only one of them is directed. Without loss of generality suppose the edge between $(a,c)$ is directed in $G,$ then $d$ is in DoneList, so $c,d$ is not adjacent in $G,$ $(c,d,S_{cd})\in\,$IndSet and $a\in S_{cd}.$ So $a\to c\in G,$ otherwise $c\to a\in G$ implies $a\to b\in G$ by R1, which contradicts with the supposition that $a-b\in G.$ Then $a\to c,c\to b\in G$ implies $a\to b\in G$ by R2, which contradicts with $a-b\in G.$ Therefore, one of $c\to b,d\to b$ is not in $G,$ and that one satisfies the conclusion of lemma.

If $r="R4",$ there exists $c,d\in V$ such that $c\to d\to b\in M,$ $(a,c),(a,d)$ are adjacent pairs in $M,$ $p=\{c\to d\},$ and $(b,a,c)$ is not a v-structure in $M.$ i.e. $(b,c,S_{bc})\in\,$IndSet and $a\in S_{bc}.$ Since $a,b$ are in DoneList, their adjacency are correctly learned so $(a,c),(a,d),(b,d)$ are adjacent pairs in $G.$ We first claim that the three edges between $(a,c),(a,d),(b,d)$ are not all directed in $G.$ Assume, for the sake of contradiction, that they are all directed, then $d\to b\in G$ since $d\to b\in M.$ If $a\to d\in G,$ then $a\to b\in G$ by R2, which contradicts with $a-b\in G.$ So $d\to a\in G.$ If $c\to a\in G,$ then $a\to b\in G$ by R1, which contradicts with $a-b\in G.$ So $a\to c\in G.$ Since all directed edges in $G$ are correct, we have $d\to a\to c\in M,$ which contradicts with $c\to d\in M$ since it will form a directed cycle. So at least one of the three edges between $(a,c),(a,d),(b,c)$ are undirected in $G.$ Since $a,b$ are connected with $T$ by an undirected path, at least one of $c,d$ is also connected with $T$ by an undirected path, so its adjacency is correctly learned and $(c,d)$ is adjacent in $G.$

Now for the sake of contradiction assume that $c\to d\in G.$ Since the adjacency of $b$ is learned, $(b,c,S_{bc})\in\,$IndSet and $d\in S_{bc}.$ So $d\to b\in G$ by R1. Enumerate all possible cases of edges between $(a,d),(a,c).$ (1) If $a\to d\in G,$ then from $d\to b\in G$ we have $a\to b\in G$ by R2, which contradicts with $a-b\in G.$ (2) If $d\to a\in G,$ then from $c\to d\in G$ we have $c\to a\in G$ by R2, and have $a\to b\in G$ by R1, which contradicts with $a-b\in G.$ (3) If $a\to c\in G,$ then from $c\to d\in G$ we have $a\to d\in G$ and back to the first case. (4) If $c\to a\in G,$ then we have $a\to b\in G$ by R1, which contradicts with $a-b\in G.$ (5) If $d-a-c\in G,$ then we have $a\to b\in G$ by R4, which contradicts with $a-b\in G.$ So $c-d\in G.$

Now since one of $(c,d)$ is connected with $T$ by an undirected path, both of them are connected with $T$ by an undirected path. So $c\to d$ is the directed edge that satisfies the conclusion of lemma.
\end{proof}

\begin{lemma}
\label{lemma:complete}
Suppose the joint distribution is markovian and faithful with the underlying DAG , and all independence tests are correct, then the MB-by-MB in MPDAG algorithm can correctly discover all directed edges around target $T.$
\end{lemma}
\begin{proof}
Denote $M$ to be the true MPDAG, $G$ be the output PDAG of Algorithm~\ref{alg:mb-by-mb in MPDAG}, and $B$ be the background knowledge. Let $S$ be a consistent DEGS of $M,B,$ which exists by Lemma~\ref{lemma:exist DEGS}. Assume, for the sake of contradiction, there exists an directed edge $d\in M$ which is around $T$ and $d\not\in G.$ Let $i$ be the index such that $(d_i,p_i,r_i)\in S$ and $d_i=d.$ Then by Lemma~\ref{lemma:complete induction}, there exists $d'\in p_i$ such that $d'\in M$, $d'\not\in G,$ and the head and tail node of $d'$ are connected with $T$ by an undirected path in $G.$ Let $j$ be the index such that $(d_j,p_j,r_j)\in S$ and $d_j=d'.$ By the consistency of DEGS, we know $j<i,$ and by Lemma~\ref{lemma:complete induction} there exists $d''\in p_j$ such that $d''\in M,$ $d''\not\in G,$and the head and tail node of $d''$ are connected with $T$ by an undirected path in $G.$ Repeat that procedure, it will lead to a contradiction because by the consistency of DEGS, $p_1=\emptyset.$
\end{proof}

\textbf{Restatement of Theorem~\ref{thm:mb-by-mb}}. 
\textit{Suppose the joint distribution is markovian and faithful with the underlying DAG, and all independence tests are correct, then the MB-by-MB in MPDAG algorithm can correctly discover the edges connected to the target $T.$ Further for each edge connected to $T,$ if it is directed in MPDAG, then it is correctly oriented in the output of the algorithm, otherwise it remains undirected in the output of the algorithm. The same holds for all nodes connected with $T$ by an undirected path in the MPDAG.}

\begin{proof}
The conclusion for target node $T$ holds directly from Lemma~\ref{lemma:adj}, Lemma~\ref{lemma:sound} and Lemma~\ref{lemma:complete}. For another node $T'$ which is connected with $T$ by an undirected path in the MPDAG, note that Algorithm~\ref{alg:mb-by-mb in MPDAG} gives the same output with $T'$ as the target node. So the conclusion also holds for $T'.$
\end{proof}

\subsection{Proof for Theorem~\ref{thm:algorithm faster}}
\begin{proof}
It suffices to prove that each independence test needed by Algorithm~\ref{alg:mb-by-mb in MPDAG} is also needed in the baseline method. 

Let $(V_1,V_2,\dots,V_p)$ be the nodes that are sequentially considered in Algorithm~\ref{alg:mb-by-mb in MPDAG}. That is, node $Z$ popped in Line 5 of Algorithm~\ref{alg:mb-by-mb in MPDAG} is iteratively $V_1,V_2,\dots,V_p.$ 

We first show that it is valid for the baseline method to explore these nodes in the same order. Let $T$ denote the target node, so clearly $V_1=T.$ By induction, we only need to show that after exploring $V_i$ in both algorithms, $T$ and $V_{i+1}$ are connected by an undirected path. After exploring $(V_1,V_2,\dots,V_i),$ each undirected edge in the local structure $G$ is only added by Line 8 in Algorithm~\ref{alg:mb-by-mb in MPDAG}, which is connected with $V_j$ for some $1\le j\le i.$ Therefore, this undirected edge is also added to $G$ in the baseline method. Moreover, with background knowledge $\mathcal{B}$ and applying Meek's rules, we can only add directed edges or orient undirected edges into directed edges. Therefore, any undirected edge after exploring $(V_1,\dots,V_i)$ in Algorithm~\ref{alg:mb-by-mb in MPDAG} also exists and is undirected in the baseline method. Since $T$ and $V_{i+1}$ are connected by an undirected path in Algorithm~\ref{alg:mb-by-mb in MPDAG}, they are also connected by an undirected path in the baseline method.

Then we consider conditional independence tests used by Algorithm~\ref{alg:mb-by-mb in MPDAG}. All conditional independence tests are used in learning the marginal graph over $\mathrm{MB}^+(Z)$ for each $Z\in (V_1,V_2,\dots,V_p),$ as shown in Line 7 in Algorithm~\ref{alg:mb-by-mb in MPDAG}. While learning the marginal graph, we need to check whether two nodes are d-separated by a subset of other nodes. This procedure is not affected by learned edges. Therefore, the total number of independence tests in Algorithm~\ref{alg:mb-by-mb in MPDAG} is not larger than the baseline method.
\end{proof}

\subsection{Proof for Lemma~\ref{lemma:identify critical set}}
\begin{proof}
It suffices to show that: (i) For every $\mathbf{Q}\in \mathcal{Q}$ such that $X\perp Y\mid pa(X,\mathcal{G}^*)\cup \mathbf{Q},$ we have $an(\mathbf{C},\mathcal{G}^*)\cap sib(X,\mathcal{G}^*)\subseteq \mathbf{Q},$ (ii) $an(\mathbf{C},\mathcal{G}^*)\cap sib(X,\mathcal{G}^*)\in \mathcal{Q},$ and (iii) $X\perp Y\mid pa(X,\mathcal{G}^*)\cup (an(\mathbf{C},\mathcal{G}^*)\cap sib(X,\mathcal{G}^*)).$ 

We first show (i). For the sake of contradiction, suppose that there exists $\mathbf{Q}\in \mathcal{Q}$ and $Z\in (an(\mathbf{C},\mathcal{G}^*)\cap sib(X,\mathcal{G})) \setminus \mathbf{Q}$ such that $X\perp Y\mid pa(X,\mathcal{G}^*)\cup \mathbf{Q}.$ Let $C \in \mathbf{C}$ be a descendant of $Z$ in $\mathcal{G}^*.$ By Theorem 1 in~\citet{fang2020ida}, $C$ is not in $pa(X,\mathcal{G}^*) \cup \mathbf{Q}.$ By the definition of critical set, there exists a chordless partially directed path $p=\langle X,C,\dots,Y\rangle$ in $\mathcal{G}^*.$ Let $\mathcal{G}_1$ be a DAG such that $X\to C$ in $\mathcal{G}_1.$ Then $p$ is causal in $\mathcal{G}_1.$ Since $p$ is chordless, every node except $C$ is not adjacent to $X.$ Therefore, $pa(X,\mathcal{G}^*)\cup \mathbf{Q}$ does not contain any node in $p,$ so $p$ is open given $pa(X,\mathcal{G}^*)\cup \mathbf{Q}.$ That contradicts with the fact that $X\perp Y\mid pa(X,\mathcal{G}^*)\cup \mathbf{Q}.$

We then show (ii). Since $X$ is not a definite cause of $Y$ in $\mathcal{G}^*,$ by Theorem 4.5 in~\citet{zuo2022counterfactual}, the induced subgraph of $\mathcal{G}^*$ over $\mathbf{C}$ is complete. Initially, let $\mathbf{Q}=\mathbf{C}.$ As long as there exists $W\in sib(X,\mathcal{G}^*)\setminus \mathbf{Q}$ and $R\in \mathbf{Q}$ such that $W\to R$ in $\mathcal{G}^*,$ we add $W$ into $\mathbf{Q}.$ By Lemma B.1 in~\citet{zuo2022counterfactual}, after each addition, $\mathbf{Q}$ remains to be complete. Moreover, after all additions, we have $\mathbf{Q}=an(\mathbf{C},\mathcal{G}^*)\cap sib(X,\mathcal{G}^*).$ Therefore, $an(\mathbf{C},\mathcal{G}^*)\cap sib(X,\mathcal{G}^*)$ is complete. Moreover, there does not exist $W\in sib(X,\mathcal{G}^*) \mid \mathbf{Q}$ and $R\in \mathbf{Q}$ such that $W\to R$ in $\mathcal{G}^*,$ so orienting $\mathbf{Q}\to X$ and $X\to sib(X,\mathcal{G}^*)\setminus \mathbf{Q}$ does not introduce any v-structure collided on $X$ or directed triangle containing $X.$ Therefore, $an(\mathbf{C},\mathcal{G}^*)\cap sib(X,\mathcal{G}^*)\in \mathcal{Q}.$

Finally, we show (iii). By Theorem 1 in~\citet{fang2020ida}, there exists a DAG $\mathcal{G}_2\in [\mathcal{G}^*]$ such that $pa(X,\mathcal{G})=pa(X,\mathcal{G}^*)\cup (an(\mathbf{C},\mathcal{G}^*)\cap sib(X,\mathcal{G}^*)).$ By Lemma 2 in~\citet{fang2020ida}, $X$ is not an ancestor of $Y$ in $\mathcal{G}_2.$ So by the Markov property, $X\perp Y\mid pa(X,\mathcal{G}^*)\cup (an(\mathbf{C},\mathcal{G}^*)\cap sib(X,\mathcal{G}^*)).$ 
\end{proof}

\subsection{Proof for Theorem~\ref{thm:correct alg non-an}}
We first show that when applying a new background knowledge in an MPDAG, there are two properties: (i) Orienting an edge outside a B-component does not affect the orientation inside the B-component; (ii) Orientation within a B-component can be done by applying Meek's rules within it. There has been a sufficient and necessary condition for a partially directed graph to be an MPDAG~\citep{fang2022representation}. However, to the best of our knowledge, there is no formal theories for these two properties.

The first theorem shows that for an MPDAG, orienting each of its B-components independently induces an orientation of the entire MPDAG.

\begin{theorem}
\label{thm:MEC decomposition}
Let $\mathcal{G}^*$ be an MPDAG. Let $\mathcal{G}^*_1,\mathcal{G}^*_2,\dots,\mathcal{G}^*_k$ be all B-components of $\mathcal{G}^*$ with vertices $\mathbf{V}_1,\mathbf{V}_2,\dots,\mathbf{V}_k.$ Then for all $\mathcal{G}_1\in [\mathcal{G}^*_1],\mathcal{G}_2\in [\mathcal{G}^*_2],\dots,\mathcal{G}_k\in [\mathcal{G}^*_k],$ there exists $\mathcal{G}\in [\mathcal{G}^*]$ such that the induced subgraph of $\mathcal{G}$ over $\mathbf{V}_i$ is identical to $\mathcal{G}_i$ for all $i=1,2,\dots,k.$
\end{theorem}
\begin{proof}
By Theorem 1 in~\citet{fang2022representation}, $\mathcal{G}^*_1,\mathcal{G}^*_2,\dots,\mathcal{G}^*_k$ are MPDAGs, so $[\mathcal{G}^*_1],[\mathcal{G}^*_2],\dots,[\mathcal{G}^*_k]$ are well-defined.

By the definition of B-component, each undirected edge in $\mathcal{G}^*$ lies in exactly one B-component of $\mathcal{G}^*.$ Let $\mathcal{G}$ be a directed graph constructed by letting the induced subgraph of $\mathcal{G}^*$ over $\mathbf{V}_i$ be $\mathcal{G}_i$ for all $i=1,2,\dots,k.$ Then $\mathcal{G}$ and $\mathcal{G}^*$ have the same skeleton, and all directed edges in $\mathcal{G}^*$ has the same direction in $\mathcal{G}.$ It suffices to show that $\mathcal{G}\in [\mathcal{G}^*],$ i.e. $\mathcal{G}$ has no directed cycle and no v-structure not included in $\mathcal{G}^*.$

For the sake of contradiction, we first suppose that there is a v-structure $A\to B\leftarrow C$ in $\mathcal{G}$ not included in $\mathcal{G}^*.$ Then both edges between $A,B$ and $B,C$ are undirected in $\mathcal{G}^*.$ Otherwise, without loss of generality, suppose $A\to B-C$ in $\mathcal{G}^*,$ then by Meek's rule 1 we have $A\to B\to C$ in $\mathcal{G}^*,$ leading to a contradiction. Since $A-B-C$ in $\mathcal{G}^*,$ $A,B,C$ are in the same B-component $\mathcal{G}^*_i.$ Then we have $A\to B\leftarrow C$ is a v-structure in $\mathcal{G}_i,$ which contradicts with the fact that $\mathcal{G}_i\in [\mathcal{G}^*_i].$

Then we assume that there is a directed cycle $p$ in $\mathcal{G}.$ If there is an directed edge in $p$ which is not in any $\mathcal{G}_i,$ then it is also directed in the chain skeleton of $\mathcal{G}^*.$ So $p$ is partially directed in the chain skeleton of $\mathcal{G}^*,$ which contradicts with the fact that the chain skeleton of $\mathcal{G}^*$ is a chain graph by Theorem 1 in~\citet{fang2022representation}. Since every edge between two B-components does not belong to any B-component, all nodes in $p$ is in the same B-component $\mathcal{G}^*_i.$ Then $p$ is a directed cycle in $\mathcal{G}_i,$ which contradicts with the fact that $\mathcal{G}_i\in [\mathcal{G}^*_i].$
\end{proof}

The converse of Theorem~\ref{thm:MEC decomposition} is trivial.

\begin{lemma}
\label{lemma:inverse decomposition of MEC}
Let $\mathcal{G}^*$ be an MPDAG. Let $\mathcal{G}^*_1,\mathcal{G}^*_2,\dots,\mathcal{G}^*_k$ be all B-components of $\mathcal{G}^*$ with vertices $\mathbf{V}_1,\mathbf{V}_2,\dots,\mathbf{V}_k.$ For every $\mathcal{G}\in [\mathcal{G}^*], i=1,2,\dots,k,$ let $\mathcal{G}_i$ denote the induced subgraph of $\mathcal{G}$ over $\mathbf{V}_i,$ then $\mathcal{G}_i\in [\mathcal{G}^*_i].$ 
\end{lemma}
\begin{proof}
Since $\mathcal{G}\in [\mathcal{G}^*],$ all directed edges in $\mathcal{G}^*$ have the same direction in $\mathcal{G},$ and $\mathcal{G}$ and $\mathcal{G}^*$ have the same skeleton and v-structures. So all directed edges in $\mathcal{G}^*_i$ have the same direction in $\mathcal{G}_i,$ and $\mathcal{G}_i$ and $\mathcal{G}^*_i$ have the same skeleton and v-structures. Clearly $\mathcal{G}_i$ has no directed cycle. So $\mathcal{G}_i\in [\mathcal{G}^*_i].$
\end{proof}

Combining Theorem~\ref{thm:MEC decomposition} and Lemma~\ref{lemma:inverse decomposition of MEC}, we can construct an bijection from the MEC to the Cartesian product of MECs for each B-component.

\begin{definition}
Let $\mathcal{G}^*$ be an MPDAG and $\mathcal{G}^*_1,\dots,\mathcal{G}^*_k$ be all B-components of $\mathcal{G}^*$ with vertices $\mathbf{V}_1,\dots,\mathbf{V}_k.$ The \textbf{decomposition map} of $\mathcal{G}^*$ is defined as a map $\phi_{\mathcal{G}^*}$ from $[\mathcal{G}^*]$ to $\times_{i=1}^{k}[\mathcal{G}^*_i]$ that for each $\mathcal{G}\in [\mathcal{G}^*],$ $\phi_{\mathcal{G}^*}(\mathcal{G})=(\mathcal{G}_1,\mathcal{G}_2,\dots,\mathcal{G}_k),$ where $\mathcal{G}_i$ is the induced subgraph of $\mathcal{G}$ over $\mathbf{V}_i,$ $i=1,2,\dots,k.$
\end{definition}

\begin{corollary}
Every decomposition map is a bijection.
\end{corollary}

The next theorem shows that orientation of edges outside a B-component does not affect the orientation inside the B-component.

\begin{theorem}
Let $\mathcal{G}^*_1$ be an MPDAG, $\mathcal{B}_{11},\mathcal{B}_{12}$ be two different B-components of $\mathcal{G}^*_1$ with vertice sets $\mathbf{V}_1,\mathbf{V}_2,$ respectively. Let $X-Y$ be an undirected edge in $\mathcal{B}_{12}.$ Let $\mathcal{G}^*_2$ be the MPDAG that represents the restricted MEC $[\mathcal{G}^*_1]$ restricted by a new background knowledge $X\to Y.$ Then for any $\mathcal{G}_1\in [\mathcal{G}^*_1],$ there exists $\mathcal{G}_2\in [\mathcal{G}^*_2],$ such that the induced subgraph of $\mathcal{G}_1$ and $\mathcal{G}_2$ over $\mathbf{V}_1$ are identical.
\end{theorem}
\begin{proof}
Let $\phi_{\mathcal{G}^*_1}$ be the decomposition map of $\mathcal{G}^*_1$. Let $\mathcal{B}_{11},\mathcal{B}_{12},\dots,\mathcal{B}_{1k}$ be all B-components of $\mathcal{G}^*_1.$ For any $\mathcal{G}_1\in [\mathcal{G}^*_1],$ let $(\mathcal{G}_{11},\mathcal{G}_{12},\dots,\mathcal{G}_{1k})=\phi_{\mathcal{G}^*_1}(\mathcal{G}_1),$ where $\mathcal{G}_{11}$ is the induced subgraph of $\mathcal{G}_1$ over $\mathbf{V}_1,$ and $\mathcal{G}_{12}$ is the induced subgraph of $\mathcal{G}_1$ over $\mathbf{V}_2.$ Then for all $i=1,2,\dots,k,$ $\mathcal{G}_{1i}\in [\mathcal{B}_{1i}].$ Since $X-Y$ is an undirected edge in $\mathcal{B}_{12},$ by Theorem 4 in~\citet{meek1995causal}, there exists $\mathcal{G}_{22}\in [\mathcal{B}_{12}]$ such that $X\to Y$ in $\mathcal{G}_{22}.$ Let $\mathcal{G}_2=\phi^{-1}_{\mathcal{G}^*_1}(\mathcal{G}_{11},\mathcal{G}_{22},\mathcal{G}_{13},\mathcal{G}_{14},\dots,\mathcal{G}_{1k}),$ then by Theorem~\ref{thm:MEC decomposition}, $\mathcal{G}_2\in [\mathcal{G}^*_1].$ By the definition of $\phi_{\mathcal{G}^*_1},$ we have $X\to Y$ in $\mathcal{G}_2,$ so $\mathcal{G}_2\in [\mathcal{G}^*_2].$ By construction of $\mathcal{G}_2,$ the induced subgraph of $\mathcal{G}_1$ and $\mathcal{G}_2$ over $\mathbf{V}_1$ are both $\mathcal{G}_{11}.$
\end{proof}

The next theorem shows that when adding background knowledge in a B-component, we can apply Meek's rules within the B-component to get the new MPDAG.

\begin{theorem}
\label{thm:local orientation}
Let $\mathcal{G}^*_1$ be an MPDAG and $\mathcal{B}_{11}$ is a B-component of $\mathcal{G}^*_1$ over node set $\mathbf{V}_1.$ Let $X-Y$ be an undirected edge in $\mathcal{B}_{11}.$ Let $\mathcal{B}_{21}$ be the MPDAG over $\mathbf{V}_1$ which represents the restricted Markov equivalence class of $[\mathcal{B}_{11}]$ with background knowledge $X\to Y.$ Let $\mathcal{G}^*_2$ be the PDAG constructed by replacing the induced graph of $\mathcal{G}^*_1$ over $\mathbf{V}_1$ by $\mathcal{B}_{21}.$ Then $\mathcal{G}^*_2$ is the MPDAG representing the restricted Markov equivalence class of $[\mathcal{G}^*_1]$ with background knowledge $X\to Y.$
\end{theorem}
\begin{proof}
Let $\mathcal{B}_{11},\mathcal{B}_{12},\dots,\mathcal{B}_{1k}$ be all B-components of $\mathcal{G}^*_1$ with vertice sets $\mathbf{V}_1,\mathbf{V}_2,\dots,\mathbf{V}_k.$ Let $\mathcal{S}$ denote the restricted Markov equivalence class of $[\mathcal{G}^*_1]$ with background knowledge $X\to Y.$ By the definition of MPDAG, it suffices to show that (i) $\mathcal{G}^*_2$ has the same skeleton with every DAG in $\mathcal{S}$; (ii) For every directed edge in $\mathcal{G}^*_2,$ it has the same direction for every DAG in $\mathcal{S}$; (iii) For every undirected edge $A-B$ in $\mathcal{G}^*_2,$ there exists two DAGs $\mathcal{G}_1,\mathcal{G}_2$ in $\mathcal{S}$ such that $A\to B$ in $\mathcal{G}_1$ and $A\leftarrow B$ in $\mathcal{G}_2.$

We first show (i). By the construction of $\mathcal{B}_{21},$ $\mathcal{B}_{21}$ has the same skeleton with $\mathcal{B}_{11}.$ So $\mathcal{G}^*_2$ has the same skeleton with $\mathcal{G}^*_1,$ and has the same skeleton with every DAG in $[\mathcal{G}^*_1].$ Therefore, $\mathcal{G}^*_2$ has the same skeleton with every DAG in $\mathcal{S}.$

Then we show (ii). In the following proof, let $\phi_{\mathcal{G}^*_1}$ be the decomposition map of $\mathcal{G}^*_1.$ For every directed edge $A\to B$ in $\mathcal{G}^*_2,$ either both $A$ and $B$ are in $\mathbf{V}_p$ for some $1\le p \le k, $ or $A\in \mathbf{V}_p,B\in \mathbf{V}_q$ for some $p\ne q,1\le p,q\le k.$ If $A\in \mathbf{V}_p,B\in \mathbf{V}_q$ for some $p\ne q$, by the construction of $\mathcal{G}^*_2,$ $A\to B$ is also in $\mathcal{G}^*_1,$ so $A\to B$ has the same direction in every DAG in $\mathcal{S}.$ Now suppose that both $A$ and $B$ are in $\mathbf{V}_p$ for some $1\le p \le k.$ If $p > 1,$ by construction of $\mathcal{G}^*_2,$ $A\to B$ is also in $\mathcal{G}^*_1,$ so $A\to B$ has the same direction in every DAG in $\mathcal{S}.$ If $p=1,$ then $A\to B$ in $\mathcal{B}_{21}.$ For every DAG $\mathcal{G}\in \mathcal{S},$ we have $\mathcal{G}\in [\mathcal{G}^*_1]$ and $X\to Y$ in $\mathcal{G}.$ Denote $(\mathcal{G}_{11},\mathcal{G}_{12},\dots,\mathcal{G}_{1k})=\phi_{\mathcal{G}^*_1}(\mathcal{G}).$ Then by Lemma~\ref{lemma:inverse decomposition of MEC}, $\mathcal{G}_{11}\in [\mathcal{B}_{11}].$ Since $X\to Y$ in $\mathcal{G}$ and $\mathcal{G}_{11}$ is the induced subgraph of $\mathcal{G}$ over $\mathbf{V}_1,$ we have $X\to Y$ in $\mathcal{G}_{11},$ so $\mathcal{G}_{11}\in [\mathcal{B}_{21}].$ Since $A\to B$ in $\mathcal{B}_{21},$ we have $A\to B$ in $\mathcal{G}_{11},$ so $A\to B$ in $\mathcal{G}.$

Finally we show (iii). Let $A-B$ be an undirected edge in $\mathcal{G}^*_2.$ By the construction of $\mathcal{B}_{21},$ every undirected edge in $\mathcal{B}_{21}$ is also undirected in $\mathcal{B}_{11}.$ So every undirected edge in $\mathcal{G}^*_2$ is also undirected in $\mathcal{G}^*_1.$ Therefore, $A-B$ is undirected in $\mathcal{G}^*_1.$ Let $\mathcal{G}_1,\mathcal{G}_2$ be two DAGs in $[\mathcal{G}^*_1]$ such that $A\to B$ in $\mathcal{G}_1$ and $A\leftarrow B$ in $\mathcal{G}_2.$ Denote $(\mathcal{G}_{11},\mathcal{G}_{12},\dots,\mathcal{G}_{1k})=\phi_{\mathcal{G}^*_1}(\mathcal{G}_1)$ and $(\mathcal{G}_{21},\mathcal{G}_{22},\dots,\mathcal{G}_{2k})=\phi_{\mathcal{G}^*_1}(\mathcal{G}_2).$ 
Since $A-B$ is undirected in $\mathcal{G}^*_1,$ both $A,B$ are in $\mathbf{V}_p$ for some $1\le p\le k.$ If $p > 1,$ let $\mathcal{G}_{31}$ be a DAG in $[\mathcal{B}_{21}]$ over vertice set $\mathbf{V}_1.$ Let $\mathcal{G}_3=\phi^{-1}_{\mathcal{G}^*_1}(\mathcal{G}_{31},\mathcal{G}_{12},\mathcal{G}_{13},\dots,\mathcal{G}_{1k})$ and $\mathcal{G}_4=\phi^{-1}_{\mathcal{G}^*_1}(\mathcal{G}_{31},\mathcal{G}_{22},\mathcal{G}_{23},\dots,\mathcal{G}_{2k}).$ Then by Theorem~\ref{thm:MEC decomposition}, $\mathcal{G}_3,\mathcal{G}_4\in [\mathcal{G}^*_1].$ Since the induced subgraph of both $\mathcal{G}_3$ and $\mathcal{G}_4$ over $\mathbf{V}_1$ are $\mathcal{G}_{31},$ which is in $[\mathcal{B}_{21}],$ we have $X\to Y$ in both $\mathcal{G}_3$ and $\mathcal{G}_4$, so $\mathcal{G}_3,\mathcal{G}_4\in \mathcal{S}.$ Since $p>1,$ by the construction of $\mathcal{G}_3$ and $\mathcal{G}_4$ we have $A\to B$ in $\mathcal{G}_3$ and $A\leftarrow B$ in $\mathcal{G}_4.$

If $p=1,$ then $A-B$ in $\mathcal{B}_{21}.$ So there exists $\mathcal{G}_{21},\mathcal{G}_{31}$ in $[\mathcal{B}_{21}]$ over vertice set $\mathbf{V}_1$ such that $A\to B$ in $\mathcal{G}_{21}$ and $A\leftarrow B$ in $\mathcal{G}_{31}.$ Let $\mathcal{G}_1$ be a DAG in $\mathcal{G}^*_1$ and denote $(\mathcal{G}_{11},\mathcal{G}_{12},\dots,\mathcal{G}_{1k})=\phi_{\mathcal{G}^*_1}(\mathcal{G}_1).$ Let $\mathcal{G}_2=\phi^{-1}_{\mathcal{G}^*_1}(\mathcal{G}_{21},\mathcal{G}_{12},\mathcal{G}_{13},\dots,\mathcal{G}_{1k})$ and $\mathcal{G}_3=\phi^{-1}_{\mathcal{G}^*_1}(\mathcal{G}_{31},\mathcal{G}_{12},\mathcal{G}_{13},\dots,\mathcal{G}_{1k}).$ By Theorem~\ref{thm:MEC decomposition}, we have $\mathcal{G}_2,\mathcal{G}_3\in [\mathcal{G}^*_1].$ Moreover, since $\mathcal{G}_{21},\mathcal{G}_{31}\in [\mathcal{B}_{21}],$ we have $X\to Y$ in both $\mathcal{G}_{21}$ and $\mathcal{G}_{31},$ so $X\to Y$ in both $\mathcal{G}_2$ and $\mathcal{G}_3,$ so $\mathcal{G}_2,\mathcal{G}_3\in \mathcal{S}.$ Since $A\to B$ in $\mathcal{G}_{21}$ and $A\leftarrow B$ in $\mathcal{G}_{31},$ we have $A\to B$ in $\mathcal{G}_2$ and $A\leftarrow B$ in $\mathcal{G}_3.$ 
\end{proof}

Leveraging these theorems, we can give a proof for Theorem~\ref{thm:correct alg non-an}.

\textbf{Restatement of Theorem~\ref{thm:correct alg non-an}.} \textit{Let $\mathcal{G}^*$ be the MPDAG under background knowledge $\mathcal{B}$ consisting of direct causal information $\mathcal{B}_1$ and non-ancestral information $\mathcal{B}_2.$ Let $\mathcal{G}\in [\mathcal{G}^*]$ be the true underlying DAG and $\mathcal{D}$ be i.i.d. observations generated from a distribution Markovian and faithful with respect to $\mathcal{G}.$ Let $X$ be a target node in $\mathcal{G}.$ Suppose all conditional independencies are correctly checked, and let $G$ be the output of Algorithm~\ref{alg:local struct:direct and non-an} with input $X,\mathcal{D},\mathcal{B}$.
Then for each $Z$ connected with $X$ by an undirected path in $\mathcal{G}^*,$ including $X$ itself, we have $pa(Z,G)=pa(Z,\mathcal{G}^*),ch(Z,G)=ch(Z,\mathcal{G}^*),sib(Z,G)=sib(Z,\mathcal{G}^*).$}

\begin{proof}
Let $\mathcal{G}^*_0$ be the MPDAG under observational data and background knowledge $\mathcal{B}_1,$ and $G_0$ be the variable $G$ after Step 1 in Algorithm~\ref{alg:local struct:direct and non-an}, i.e., the output of Algorithm~\ref{alg:mb-by-mb in MPDAG} under input $X,\mathcal{D},\mathcal{B}_1.$  By Theorem~\ref{thm:mb-by-mb}, for each $Z$ connected with $X$ by an undirected path in $\mathcal{G}^*_0,$ including $X$ itself, we have $pa(Z,G_0)=pa(Z,\mathcal{G}^*_0), ch(Z,G_0)=ch(Z,\mathcal{G}^*_0), sib(Z,G_0)=sib(Z,\mathcal{G}^*_0).$

Write $\mathcal{B}_2=\{(N_j,T_j)\}_{j=1}^{k},$ which implies that $N_j$ is not a cause of $T_j$ in the underlying DAG $\mathcal{G}$ for $j=1,2,\dots,k.$ For each $j=1,2,\dots,k,$ let $\mathcal{G}^*_j$ denote the MPDAG representing the restricted Markov equivalence class of $[\mathcal{G}^*_{j-1}]$ with a new background knowledge that $N_j$ is not a cause of $T_j.$ Let $G_j$ denote the PDAG $G$ after the $j$-th iteration of Step 2 to Step 12 in Algorithm~\ref{alg:local struct:direct and non-an}. By induction, it suffices to show that for each $j=1,2,\dots,k,$ if for each $Z$ connected with $X$ by an undirected path in $\mathcal{G}^*_{j-1},$ including $X$ itself, we have $pa(Z,G_{j-1})=pa(Z,\mathcal{G}^*_{j-1}), ch(Z,G_{j-1})=ch(Z,\mathcal{G}^*_{j-1}), sib(Z,G_{j-1})=sib(Z,\mathcal{G}^*_{j-1}),$ then for each $Z$ connected with $X$ by an undirected path in $\mathcal{G}^*_j,$ including $X$ itself, we have $pa(Z,G_j)=pa(Z,\mathcal{G}^*_j), ch(Z,G_j)=ch(Z,\mathcal{G}^*_j), sib(Z,G_j)=sib(Z,\mathcal{G}^*_j).$ By symmetry, we only need to consider $j=1.$

First suppose that $N_1$ is not connected with $X$ by an undirected path in $G_0.$ Then $N_1$ is not connected with $X$ by an undirected path in $\mathcal{G}^*_0,$ so $N_1$ is not in the B-component containing $X$ in $\mathcal{G}^*_0.$ Let $\mathbf{C}_{\mathcal{G}^*_0}(N_1,T_1)$ denote the critical set of $N_1$ with respect to $T_1$ in $\mathcal{G}^*_0.$ By Lemma 2 in~\citet{fang2020ida} and Lemma 4.3 in~\citet{zuo2022counterfactual}, the fact that $N_1$ is not a cause of $T_1$ is equivalent to that for every $C\in \mathbf{C}_{\mathcal{G}^*_0}(N_1,T_1),$ we have $C\to N_1$ in the true underlying DAG $\mathcal{G}.$ Therefore, $\mathcal{G}^*_1$ is the MPDAG representing the restricted Markov equivalence class $[\mathcal{G}^*_0]$ with background knowledge $C\to N_1$ for every $C\in \mathbf{C}_{\mathcal{G}^*_0}(N_1,T_1).$ Let $\mathcal{B}_{01}$ be the B-component of $\mathcal{G}^*_0$ containing $X$ over vertice set $\mathbf{V}_1,$ and $\mathcal{B}_{02}$ be the B-component of $\mathcal{G}^*_0$ containing $N_1$ over vertice set $\mathbf{V}_2.$ Let $\mathcal{B}_{12}$ be the MPDAG constructed by adding background knowledge $C\to N_1$ for every $C\in \mathbf{C}_{\mathcal{G}^*_0}(N_1,T_1)$ to $\mathcal{B}_{02}.$ Then by sequentially applying Theorem~\ref{thm:local orientation}, $\mathcal{G}^*_1$ is obtained by replacing the induced subgraph of $\mathcal{G}^*_0$ over $\mathbf{V}_2$ with $\mathcal{B}_{02}.$ Therefore, the induced subgraph of $\mathcal{G}^*_1$ over $\mathbf{V}_1$ is identical to $\mathcal{B}_{01}.$ So the B-component of $\mathcal{G}^*_1$ containing $X$ is also $\mathcal{B}_{01}.$ Since all directed edges in $\mathcal{G}^*_0$ have the same direction in $\mathcal{G}^*_1,$ for each $Z$ connected with $X$ by an undirected path in $\mathcal{G}^*_1,$ including $X$ itself, we have $pa(Z,\mathcal{G}^*_1)=pa(Z,\mathcal{G}^*_0), pa(Z,\mathcal{G}^*_1)=ch(Z,\mathcal{G}^*_0), pa(Z,\mathcal{G}^*_1)=sib(Z,\mathcal{G}^*_0).$ By Steps 3 to 11 in Algorithm~\ref{alg:local struct:direct and non-an}, we have $G_1=G_0,$ so we also have $pa(Z,G_1)=pa(Z,\mathcal{G}^*_1), ch(Z,G_1)=ch(Z,\mathcal{G}^*_1), sib(Z,G_1)=sib(Z,\mathcal{G}^*_1).$

Then suppose that $N_1$ is connected with $X$ by an undirected path in $G_0.$ Then $N_1$ is connected with $X$ by an undirected path in $\mathcal{G}^*_0,$ so $N_1$ and $X$ are in the same B-component of $\mathcal{G}^*_0$ over vertice set $\mathbf{V}_1,$ denoted by $\mathcal{B}_{01}.$ By Theorem~\ref{thm:mb-by-mb}, we have $pa(N_1,G_0)=pa(N_1,\mathcal{G}^*_0), ch(N_1,G_0)=ch(N_1,\mathcal{G}^*_0), sib(N_1,G_0)=sib(N_1,\mathcal{G}^*_0).$ Let $\mathcal{Q}$ be the set of all $\mathbf{Q}\subseteq sib(X,\mathcal{G}^*_0)$ such that orienting $\mathbf{Q}\to X$ and $X\to sib(X,\mathcal{G}^*_0)\setminus \mathbf{Q}$ does not introduce any v-structure collided on $X$ or any directed triangle containing $X.$ Then after Steps 4 to 9 in Algorithm~\ref{alg:local struct:direct and non-an}, we have $\mathrm{candC}=\cap \left\{ \mathbf{Q}\in \mathcal{Q} \mid X\perp Y \mid pa(X,\mathcal{G}^*_0)\cup \mathbf{Q} \right\}.$ By Lemma~\ref{lemma:identify critical set}, we have $\mathrm{candC}=an(\mathbf{C}_{\mathcal{G}^*_0}(N_1,T_1),\mathcal{G}^*_0)\cap sib(X,\mathcal{G}^*_0).$ Therefore, orienting $\mathrm{candC} \to N_1$ in Step 10 in Algorithm~\ref{alg:local struct:direct and non-an} is equivalent to orienting $an(\mathbf{C}_{\mathcal{G}^*_0}(N_1,T_1),\mathcal{G}^*_0)\cap sib(X,\mathcal{G}^*_0) \to N_1.$ The background knowledge that $N_1$ is not a cause of $T_1$ implies $N_1$ is not a definite cause of $T_1$ in $\mathcal{G}^*_0.$ So by Theorem 4.5 in~\citet{zuo2022counterfactual}, we have $\mathbf{C}_{\mathcal{G}^*_0}(N_1,T_1)\subseteq sib(N_1,\mathcal{G}^*_0).$ For every $C\in \mathbf{C}_{\mathcal{G}^*_0}(N_1,T_1)$ and $D\in pa(C,\mathcal{G}^*_0)\cap sib(N_1,\mathcal{G}^*_0),$ if we orient $C\to N_1$ as a new background knowledge, by $D\to C\to N_1$, $D-N_1$ and Meek's rule 2, we should orient $D\to N_1$ to get the MPDAG $\mathcal{G}^*_1.$ Inductively applying this fact we also should orient $an(\mathbf{C}_{\mathcal{G}^*_0}(N_1,T_1),\mathcal{G}^*_0)\cap sib(X,\mathcal{G}^*_0) \to N_1.$ Conversely, if we orient $an(\mathbf{C}_{\mathcal{G}^*_0}(N_1,T_1),\mathcal{G}^*_0)\cap sib(X,\mathcal{G}^*_0) \to N_1,$ then we have oriented $\mathbf{C}_{\mathcal{G}^*_0}(N_1,T_1)\to N_1$ by the fact that $\mathbf{C}_{\mathcal{G}^*_0}(N_1,T_1)\subseteq an(\mathbf{C}_{\mathcal{G}^*_0}(N_1,T_1),\mathcal{G}^*_0)$ and $\mathbf{C}_{\mathcal{G}^*_0}(N_1,T_1)\subseteq sib(X,\mathcal{G}^*_0).$ Therefore, Step 10 in Algorithm 2 is equivalent to orienting $\mathbf{C}_{\mathcal{G}^*_0}(N_1,T_1)\to N_1$ and applying Meek's rules in $G_0.$

In this case, let $\mathcal{B}_{01}$ be the B-component of $\mathcal{G}^*_0$ including $N_1$ and $X,$ with vertice set $\mathbf{V}_1.$ Let $\mathcal{B}_{01},\mathcal{B}_{02},\dots,\mathcal{B}_{0k}$ be all B-components of $\mathcal{G}^*_0,$ with vertice sets $\mathbf{V}_1,\mathbf{V}_2,\dots,\mathbf{V}_k,$ respectively. Then $\mathbf{C}_{\mathcal{G}^*_0}(N_1,T_1)\subseteq sib(X,\mathcal{G}^*_0)\subseteq \mathbf{V}_1.$ Let $\mathcal{B}_{11}$ be the MPDAG over vertice set $\mathbf{V}_1,$ which represents the restricted Markov equivalence class of $[\mathcal{B}_{01}]$ with background knowledge $\mathbf{C}_{\mathcal{G}^*_0}(N_1,T_1)\to N_1.$ By applying Theorem~\ref{thm:local orientation} inductively, we know that $\mathcal{G}^*_1$ can be constructed by replacing the induced subgraph of $\mathcal{G}^*_0$ over $\mathbf{V}_1$ by $\mathcal{B}_{11}.$ 

Let $G_{01}$ be the induced subgraph of $G_0$ over vertice set $\mathbf{V}_1,$ then $G_{01}$ is identical to $\mathcal{B}_{01}.$ By orienting $\mathbf{C}_{\mathcal{G}^*_0}(N_1,T_1)\to N_1$ and applying Meek's rules in $G_0$ to get $G_1,$ the induced subgraph of $G_1$ over $\mathbf{V}_1$ is identical to $\mathcal{B}_{11}$ by Theorem 4 in~\citet{meek1995causal}. Since all directed edges in $G_0$ have the same direction in $G_1,$ we can conclude that for every $Z$ connected with $X$ in $\mathcal{G}^*_1,$ including $X$ itself, we have $pa(Z,\mathcal{G}^*_1)=pa(Z,G_1), ch(Z,\mathcal{G}^*_1)=ch(Z,G_1), sib(Z,\mathcal{G}^*_1)=sib(Z,G_1).$

\end{proof}

\subsection{Proof for Theorem~\ref{thm:definite non-cause}}
\begin{proof}
If $X$ is a definite non-cause of $Y,$ then for any DAG $G$ represented by $M$, $X\not\in an(Y,G).$ By Theorem 1 in~\citep{fang2020ida}, there exists a DAG $G$ represented by $M$ such that $pa(X,G)=pa(X,M).$ By the Markov property we have $X\perp Y|pa(X,G),$ so $X\perp Y|pa(X,M).$

Conversely, if $X$ is not a definite non-cause of $Y,$ then there exists a DAG $G$ represented by $M$ such that $X\in an(Y,G).$ Let $\pi$ be a causal path in $G$ from $X$ to $Y,$ then each node on $\pi$ is not in $pa(X,M),$ otherwise it is also in $pa(X,G)$ and it will form a directed cycle in $G.$ So $X\not\perp Y|pa(X,M).$
\end{proof}

\subsection{Proof for Theorem~\ref{thm:explicit cause}}

\begin{figure}[ht]
    \centering
    \begin{tikzpicture}[node distance=1.5cm]
        \node at (7, 4) [draw] (lemma11) {Lemma~\ref{lemma:explicit if}};
        \node at (7, 3) [draw] (lemma12) {Lemma~\ref{lemma:explicit only if adjacent}};
        \node at (7, 2) [draw] (lemma13) {Lemma~\ref{lemma:explicit only if causal}};
        \node at (7, 1) [draw] (lemma17) {Lemma~\ref{lemma:explicit only if noncausal}};
        \node at (7, 0) [draw] (lemma18) {Lemma~\ref{lemma:explicit only if pa or sib}};
        \node at (1, 1) [draw] (lemma14) {Lemma~\ref{lemma:explicit oin has collider}};
        \node at (4, 3) [draw] (lemma15) {Lemma~\ref{lemma:explicit only if noncausal unsheilded}};
        \node at (4, 2) [draw] (lemma16) {Lemma~\ref{lemma:explicit oin pre-causal}};
        \node at (10, 2) [draw] (theorem4) {Theorem~\ref{thm:explicit cause}};
        
        \draw[->] (lemma11) -- (theorem4);
        \draw[->] (lemma12) -- (theorem4);
        \draw[->] (lemma13) -- (theorem4);
        \draw[->] (lemma17) -- (theorem4);
        \draw[->] (lemma18) -- (theorem4);
        \draw[->] (lemma14) -- (lemma15);
        \draw[->] (lemma14) -- (lemma16);
        \draw[->] (lemma14) -- (lemma17);
        \draw[->] (lemma15) -- (lemma17);
        \draw[->] (lemma16) -- (lemma17);
    \end{tikzpicture}
    \caption{Proof structure of Theorem~\ref{thm:explicit cause}.}
\label{fig:proof for thm4}
\end{figure}
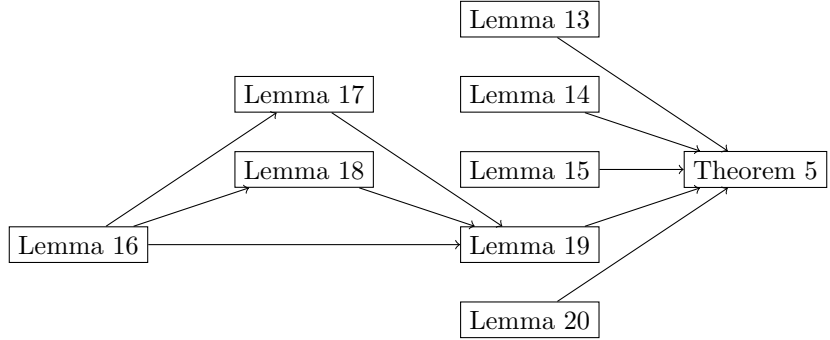

Figure~\ref{fig:proof for thm4} shows how lemmas fit together to prove Theorem~\ref{thm:explicit cause}. The first lemma shows the necessity of Theorem~\ref{thm:explicit cause}.

\begin{lemma}
\label{lemma:explicit if}
Let $X$ be an explicit cause of $Y$ in an MPDAG $\mathcal{G}^*,$ then $X \not\perp Y|pa(X,\mathcal{G}^*)\cup sib(X,\mathcal{G}^*).$
\end{lemma}
\begin{proof}
Let $\pi$ be a directed path in $\mathcal{G}^*$ from $X$ to $Y,$ and $\mathcal{G}$ be any DAG in $[\mathcal{G}^*].$ By Lemma~\ref{lemma:mpdag def}, $\pi$ is also directed in $\mathcal{G}.$ For any node $Z\neq X$ on $\pi,$ if $Z$ and $X$ are adjacent, by Lemma~\ref{lemma:partial cycle} we have $X\to Z\in \mathcal{G}^*,$ so $Z\not\in pa(X,\mathcal{G}^*)\cup sib(X,\mathcal{G}^*).$ Since there is no collider on $\pi$ in $\mathcal{G},$ $\pi$ is open given $pa(X,\mathcal{G}^*)\cup sib(X,\mathcal{G}^*).$ Therefore, $X \not\perp Y|pa(X,\mathcal{G}^*)\cup sib(X,\mathcal{G}^*).$
\end{proof}

In the following lemmas, we prove the sufficiency of Theorem~\ref{thm:explicit cause}. Therefore, we assume that $X$ is not an explicit cause of $Y$ in $\mathcal{G}^*,$ and we want to show that $X \perp Y|pa(X,\mathcal{G}^*)\cup sib(X,\mathcal{G}^*),$ which is equivalent to that any path $\pi$ between $X$ and $Y$ are blocked by $pa(X,\mathcal{G}^*)\cup sib(X,\mathcal{G}^*).$ Let $\mathcal{G}$ be any DAG represented by $\mathcal{G}^*.$
Lemma~\ref{lemma:explicit only if adjacent} considers the case that $\pi$ is an edge between $X$ and $Y.$ Lemma~\ref{lemma:explicit only if causal} considers the case that $\pi$ is causal in $\mathcal{G}$ with its second node in $ch(X,\mathcal{G}^*).$ Lemma~\ref{lemma:explicit only if noncausal} considers the case that $\pi$ is non-causal in $\mathcal{G}$ with its second node in $ch(X,\mathcal{G}^*).$ Lemma~\ref{lemma:explicit only if pa or sib} considers the case that the second node in $\pi$ is in $pa(X,\mathcal{G}^*)\cup sib(X,\mathcal{G}^*).$ Combining all this lemmas, we can conclude that sufficiency of Theorem~\ref{thm:explicit cause} holds.

\begin{lemma}
\label{lemma:explicit only if adjacent}
Let $X,Y$ be two distinct nodes that are not adjacent in an MPDAG $\mathcal{G}^*.$ Suppose $X$ is not an explicit cause of $Y$ in $\mathcal{G}^*.$ Let $\mathcal{G}$ by any DAG in $[\mathcal{G}^*].$ If $X,Y$ are adjacent, then the path $\pi$ as the edge between $X$ and $Y$ is blocked by $pa(X,\mathcal{G}^*)\cup sib(X,\mathcal{G}^*).$
\end{lemma}
\begin{proof}
Since $X$ is not an explicit cause of $Y$ in $\mathcal{G}^*,$ we have $\pi=X\leftarrow Y$ or $\pi=X-Y$ in $\mathcal{G}^*,$ so $Y\in pa(X,\mathcal{G}^*)\cup sib(X,\mathcal{G}^*)$ and $Y$ is a non-collider in $\pi,$ so $\pi$ is blocked by $pa(X,\mathcal{G}^*)\cup sib(X,\mathcal{G}^*).$ 
\end{proof}

\begin{lemma}
\label{lemma:explicit only if causal}
Let $X,Y$ be two distinct nodes that are not adjacent in an MPDAG $\mathcal{G}^*.$ Suppose $X$ is not an explicit cause of $Y$ in $\mathcal{G}^*.$ Let $\mathcal{G}$ by any DAG in $[\mathcal{G}^*].$ If there exists a causal path $\pi=\langle X=v_0,v_1,\dots,v_n,v_{n+1}=Y\rangle$ in $\mathcal{G}$ from $X$ to $Y$ such that $X\to v_1\in \mathcal{G}^*,$ then $\pi$ is blocked by $pa(X,\mathcal{G}^*)\cup sib(X,\mathcal{G}^*).$ 
\end{lemma}
\begin{proof}
If there exists a node $v_k, 1<k\le n$ such that $X,v_k$ are adjacent, then by Corollary~\ref{cor:partial cycle}, $X\to v_k\in G.$ So by Lemma~\ref{lemma:mpdag def}, $v_k\not\in pa(X,\mathcal{G}^*)$. If $v_k\in sib(X,\mathcal{G}^*),$ $\pi$ is blocked by $pa(X,M)\cup sib(X,\mathcal{G}^*)$ since $v_k$ is not a collider on $\pi.$ Therefore, we only need to consider the case that for all such $k$ that $X$ and $v_k$ are adjacent, we have $v_k\in ch(X,\mathcal{G}^*).$ Let $1\le m\le n$ be the largest index that $v_m$ and $X$ are adjacent, i.e. $v_m\in ch(X,\mathcal{G}^*).$ Denote $Y=v_{n+1}.$ For any $i,j\in \{m,m+1,\dots,n,n+1\}$ and $i+1<j,$ if $v_i,v_j$ are adjacent in $\mathcal{G},$ by Corollary~\ref{cor:partial cycle} we have $v_i\to v_j\in G$. Then the path $\pi(v_m,v_i)\oplus \langle v_i,v_j\rangle \oplus \pi(v_j,Y)$ is also a causal path in $\mathcal{G}.$ Repeat this procedure until this causal path is unshielded, and denote it by $p.$ Note that all nodes on $p$ is not adjacent to $X,$ so $\langle X,v_m\rangle \oplus p$ is unshielded, hence of definite status in $\mathcal{G}^*,$ and it is causal in $G,$ So it is b-possibly causal in $\mathcal{G}^*.$ So since $v_m\in ch(X,\mathcal{G}^*)$ and by Lemma~\ref{lemma: perkovic}, it is causal in $\mathcal{G}^*.$ That contradicts with that $X$ is not an explicit cause of $Y.$ Therefore, $\pi$ is blocked by $pa(X,\mathcal{G}^*)\cup sib(X,\mathcal{G}^*).$
\end{proof}

Before proving the case that $\pi$ is non-causal in $\mathcal{G},$ we need some technical lemmas. The first lemma is trivial, which shows that in this case, $\pi$ has at least one collider.

\begin{lemma}
\label{lemma:explicit oin has collider}
Let $X,Y$ be two distinct nodes that are not adjacent in an MPDAG $\mathcal{G}^*.$ Suppose $X$ is not an explicit cause of $Y$ in $\mathcal{G}^*.$ Let $\mathcal{G}$ by any DAG in $[\mathcal{G}^*].$ Suppose $\pi=\langle X=v_0,v_1,\dots,v_n,v_{n+1}=Y\rangle$ is a non-causal path in $\mathcal{G}$ from $X$ to $Y$ such that $X\to v_1\in \mathcal{G}^*,$ then there exists $1\le p \le n$ such that $v_{p-1}\to v_p\leftarrow v_{p+1}$ is a collider on $\pi.$
\end{lemma}
\begin{proof}
If $v_i\to v_{i+1}$ in $\mathcal{G}$ for all $0\le i\le n,$ then $\pi$ is a causal path in $\mathcal{G},$ which contradicts with the fact that $\pi$ is non-causal. Let $p$ be the smallest index such that $v_p \leftarrow v_{p+1}$ in $\mathcal{G}.$ Clearly $p>0$ since $X=v_0\to v_1$ in $\mathcal{G}.$ Then we have $v_{p-1}\to v_p\leftarrow v_{p+1}$ in $\mathcal{G}.$ 
\end{proof}

The next technical lemma shows that if we want to prove that $\pi$ is blocked by $pa(X,\mathcal{G}^*)\cup sib(X,\mathcal{G}^*)$ in non-causal case, we only need to consider the case that the first collider on $\pi$ in $\mathcal{G}$ is unshielded. 

\begin{lemma}
\label{lemma:explicit only if noncausal unsheilded}
Let $X,Y$ be two distinct nodes that are not adjacent in an MPDAG $\mathcal{G}^*.$ Suppose $X$ is not an explicit cause of $Y$ in $\mathcal{G}^*.$ Let $\mathcal{G}$ by any DAG in $[\mathcal{G}^*].$ Suppose $\pi=\langle X=v_0,v_1,\dots,v_n,v_{n+1}=Y\rangle$ is a non-causal path in $\mathcal{G}$ from $X$ to $Y$ such that $X\to v_1\in \mathcal{G}^*.$ Let $v_{p-1}\to v_p\leftarrow v_{p+1}$ be the collider on $\pi$ which is closest to $X.$ If $v_{p-1}$ and $v_{p+1}$ are adjacent, and $\pi$ is open given $pa(X,\mathcal{G}^*)\cup sib(X,\mathcal{G}^*),$ then $\pi'=\pi(X,v_{p-1}) \oplus \langle v_{p-1},v_{p+1} \rangle \oplus \pi(v_{p+1},Y)$ is also a non-causal path in $\mathcal{G}$ from $X$ to $Y$ with its second node in $ch(X,\mathcal{G}^*),$ and $\pi'$ is also open given $pa(X,\mathcal{G}^*)\cup sib(X,\mathcal{G}^*).$
\end{lemma}
\begin{proof}
We will discuss for two cases: $v_{p-1}\to v_{p+1}\in G$ and $v_{p-1}\leftarrow v_{p+1}\in G.$

Suppose that $v_{p-1}\to v_{p+1}\in G.$ We first show that $\pi'$ is also a non-causal path in $\mathcal{G}$ from $X$ to $Y$ with its second node in $ch(X,\mathcal{G}^*).$ If $X=v_{p-1},$ then $\pi$ is blocked by $v_{p+1}$ if $v_{p+1}\in pa(X,\mathcal{G}^*)\cup sib(X,\mathcal{G}^*),$ so $v_{p+1}\in ch(X,\mathcal{G}^*).$ If $X\neq v_{p-1},$ then $\pi'$ also begins with $X\to v_1.$ So $\pi'$ is also a path from $X$ to $Y$ with the second node in $ch(X,\mathcal{G}^*).$
If $\pi'$ is causal in $G$, by Lemma~\ref{lemma:explicit only if causal} we know $\pi'$ is blocked by $pa(X,\mathcal{G}^*)\cup sib(X,\mathcal{G}^*).$ Since all nodes in $\pi'$ are non-collider and $v_{p-1}\to v_p\in G,$ they are also non-collider on $\pi.$ So $\pi$ is also blocked by $pa(X,\mathcal{G}^*)\cup sib(X,\mathcal{G}^*),$ which conflicts with given conditions. Therefore, $\pi'$ is non-causal in $\mathcal{G}.$

Since $\pi'$ is non-causal in $\mathcal{G},$ we know that $v_{p+1}\neq Y.$ We talk about the edge between $v_{p+1}$ and $v_{p+2}$ in $\mathcal{G}.$ If $v_{p+1}\to v_{p+2}\in \mathcal{G},$ then all non-colliders on $\pi'$ are also non-colliders on $\pi,$ and all colliders on $\pi'$ are also colliders on $\pi.$ Since $\pi$ is open given $pa(X,\mathcal{G}^*)\cup sib(X,\mathcal{G}^*),$ $\pi'$ is also open given $pa(X,\mathcal{G}^*)\cup sib(X,\mathcal{G}^*).$ If $v_{p+2}\to v_{p+1}\in G,$ then $v_{p+1}$ is a collider on $\pi'.$ Since $v_p$ is a collider on $\pi$ and $\pi$ is open given $pa(X,\mathcal{G}^*)\cup sib(X,\mathcal{G}^*)$, we have $v_p\in an(pa(X,\mathcal{G}^*)\cup sib(X,\mathcal{G}^*), \mathcal{G}),$ so by $v_p\leftarrow v_{p+1}$ we also have $v_{p+1}\in an(pa(X,\mathcal{G}^*)\cup sib(X,\mathcal{G}^*), \mathcal{G}).$ Since other nodes except $v_{p+1}$ on $\pi'$ is a collider if and only if it is a collider on $\pi,$ we conclude that $\pi'$ is open given $pa(X,\mathcal{G}^*)\cup sib(X,\mathcal{G}^*).$

Now suppose that $v_{p-1}\leftarrow v_{p+1}\in G.$ If $X=v_{p-1},$ we have $v_{p+1}\in pa(X,\mathcal{G}^*)\cup sib(X,\mathcal{G}^*),$ then $\pi$ is blocked by $pa(X,\mathcal{G}^*)\cup sib(X,\mathcal{G}^*),$ so $X\neq v_{p-1}.$ Therefore, $\pi'$ is also a non-causal path from $X$ to $Y$ with its second node in $ch(X,\mathcal{G}^*).$ Since $v_p$ is a collider on $\pi$ and $\pi$ is open given $pa(X,\mathcal{G}^*)\cup sib(X,\mathcal{G}^*)$, we have $v_p\in an(pa(X,\mathcal{G}^*)\cup sib(X,\mathcal{G}^*), \mathcal{G}),$ so by $v_{p-1}\to v_p$ in $\mathcal{G},$ we also have $v_{p-1}\in an(pa(X,\mathcal{G}^*)\cup sib(X,\mathcal{G}^*), \mathcal{G}).$ Since other nodes except $v_{p-1}$ on $\pi'$ is a collider if and only if it is a collider on $\pi,$ we conclude that $\pi'$ is open given $pa(X,\mathcal{G}^*)\cup sib(X,\mathcal{G}^*).$
\end{proof}

Until now, we have shown that if the first collider on $\pi$ is shielded, there exists a shorter path $\pi'$ such that $\pi'$ is blocked implies $\pi$ is blocked, and $\pi'$ also have the second node in $ch(X,\mathcal{G}^*)$. Since $\pi'$ is blocked if it is causal in $\mathcal{G}$, we only need to consider the case that it is non-causal in $\mathcal{G}.$ If the first collider in $\pi'$ is shielded, we can repeat this procedure and get a path $\pi''$ shorter than $\pi',$ and we just need to prove $\pi''$ is blocked. At last, we just need to consider a path with its first collider unshielded.

The last technical lemma shows that we only need to consider the case that $\pi(X,v_p)$ is causal in $\mathcal{G}^*.$

\begin{lemma}
\label{lemma:explicit oin pre-causal}
Let $X,Y$ be two distinct nodes that are not adjacent in an MPDAG $\mathcal{G}^*.$ Suppose $X$ is not an explicit cause of $Y$ in $\mathcal{G}^*.$ Let $\mathcal{G}$ by any DAG in $[\mathcal{G}^*].$ Suppose $\pi=\langle X=v_0,v_1,\dots,v_n,v_{n+1}=Y\rangle$ is a non-causal path in $\mathcal{G}$ from $X$ to $Y$ such that $X\to v_1\in \mathcal{G}^*.$ Let $v_{p-1}\to v_p\leftarrow v_{p+1}$ be the collider on $\pi$ which is closest to $X.$ Suppose that $v_{p-1}$ and $v_{p+1}$ are not adjacent, and $\pi$ is open given $pa(X,\mathcal{G}^*)\cup sib(X,\mathcal{G}^*).$ Let $1\le k\le p$ be the largest index such that $X$ and $v_k$ are adjacent, let $s$ be a shortest subpath of $\pi(v_k,v_p)$ which is causal in $\mathcal{G},$ and let $t=\langle X,v_k\rangle \oplus s \oplus \pi(v_p,Y).$ Then $t$ is also a non-causal path in $\mathcal{G}$ with its second node in $ch(X,\mathcal{G}^*),$ and $t$ is open given $pa(X,\mathcal{G}^*)\cup sib(X,\mathcal{G}^*).$
\end{lemma}
\begin{proof}
Since $v_{p-1}$ and $v_{p+1}$ are not adjacent, $v_{p-1}\to v_p\leftarrow v_{p+1}$ forms a v-structure, so $v_{p-1}\to v_p\leftarrow v_{p+1}\in \mathcal{G}^*.$ 

If $X$ and $v_p$ are adjacent, we have $X\to v_p\in \mathcal{G}$ by Corollary~\ref{cor:partial cycle}, so $v_p\not\in pa(X,\mathcal{G}^*).$ If $X-v_p\in \mathcal{G}^*,$ since $v_{p+1}\to v_p\in \mathcal{G}^*$ we know $X$ and $v_{p+1}$ are adjacent, otherwise the edge $X-v_p$ is oriented. If $X\to v_{p+1}\in \mathcal{G}^*,$ by Meek's rule R2 we have $X\to v_p\in \mathcal{G}^*,$ which leads to a contradiction. So $v_{p+1}\in pa(X,\mathcal{G}^*)\cup sib(X,\mathcal{G}^*),$ and then $\pi$ is blocked by $pa(X,\mathcal{G}^*)\cup sib(X, \mathcal{G}^*)$ since $v_{p+1}$ is a non-collider on $\pi.$ It contradicts with the assumption. Therefore, if $X$ and $v_p$ are adjacent, we can conclude that $X\to v_p\in \mathcal{G}^*.$ For $1< i< p,$ $v_i$ is a non-collider on $\pi.$ Since $\pi$ is open given $pa(X,\mathcal{G}^*)\cup sib(X,\mathcal{G}^*),$ we know $v_i\not\in pa(X,\mathcal{G}^*)\cup sib(X,\mathcal{G}^*).$  Therefore, for any $1<i\le p$, if $X$ and $v_i$ are adjacent, we have $X\to v_i\in \mathcal{G}^*.$ 

Since $X$ and $v_k$ are adjacent, we have $X\to v_k\in \mathcal{G}^*.$ By construction of $s$ and $t,$ we know $\langle X,v_k\rangle \oplus s$ is a chordless path in $\mathcal{G}^*,$ and is causal in $\mathcal{G},$ and $v_k\in ch(X,\mathcal{G}^*).$ So by Lemma~\ref{lemma: perkovic}, $\langle X,v_k\rangle \oplus s$ is causal in $\mathcal{G}^*.$ Since $\pi$ is open given $pa(X,\mathcal{G}^*)\cup sib(X,\mathcal{G}^*),$ and any node on $t$ is a collider if and only if it is also a collider on $\pi,$ $t$ is open given $pa(X,\mathcal{G}^*)\cup sib(X,\mathcal{G}^*).$ 
\end{proof}

Now we finally come to the lemma that shows sufficiency of Theorem~\ref{thm:explicit cause} holds for the case that $\pi$ is non-causal in $\mathcal{G}$ with its second node in $ch(X,\mathcal{G}^*).$ 

\begin{lemma}
\label{lemma:explicit only if noncausal}
Let $X,Y$ be two distinct nodes that are not adjacent in an MPDAG $\mathcal{G}^*.$ Suppose $X$ is not an explicit cause of $Y$ in $\mathcal{G}^*.$ Let $\mathcal{G}$ by any DAG in $[\mathcal{G}^*].$ If there exists a non-causal path $\pi=\langle X=v_0,v_1,\dots,v_n,v_{n+1}=Y\rangle$ in $\mathcal{G}$ from $X$ to $Y$ such that $X\to v_1\in \mathcal{G}^*,$ then $\pi$ is blocked by $pa(X,\mathcal{G}^*)\cup sib(X,\mathcal{G}^*).$ 
\end{lemma}
\begin{proof}
By Lemma~\ref{lemma:explicit oin has collider}, there exists $1\le p \le n$ such that $v_{p-1}\to v_p\leftarrow v_{p+1}$ is a collider on $\pi.$ Let $v_{p-1}\to v_p\leftarrow v_{p+1}$ be the collider on $\pi$ which is closest to $X.$ By Lemma~\ref{lemma:explicit only if noncausal unsheilded}, without loss of generality we can assume that $v_{p-1}$ and $v_{p+1}$ are not adjacent, so $v_{p-1}\to v_p\leftarrow v_{p+1}$ forms a v-structure and we have $v_{p-1}\to v_p\leftarrow v_{p+1}\in \mathcal{G}^*.$ By Lemma~\ref{lemma:explicit oin pre-causal}, without loss of generality we can assume that $\pi(X,v_p)$ is causal in $\mathcal{G}^*.$

Now consider the collider $v_p.$ We only need to consider the case that $v_p\in an(pa(X,\mathcal{G}^*)\cup sib(X,\mathcal{G}^*), \mathcal{G}),$ otherwise $\pi$ is blocked by $pa(X,\mathcal{G}^*)\cup sib(X,\mathcal{G}^*)$. Let $s=\langle v_p,b_1,b_2,\dots,b_r\rangle$ be the shortest causal path in $\mathcal{G}$ from $v_p$ to $pa(X,\mathcal{G}^*)\cup sib(X,\mathcal{G}^*),$ and $b_r\in pa(X,\mathcal{G}^*)\cup sib(X,\mathcal{G}^*).$ Moreover, since $\pi(X,v_p)\oplus s$ is a causal path from $X$ to $b_r$ in $\mathcal{G},$ we have $X\to b_r\in  \mathcal{G}$ by Corollary~\ref{cor:partial cycle}. Therefore $b_r\not\in pa(X,\mathcal{G}^*)$ and hence $b_r\in sib(X,\mathcal{G}^*).$

Denote $b_0=v_p.$ Let $0\le k< r$ be the largest index such that $v_{p-1}\to b_k\leftarrow v_{p+1}\in \mathcal{G}^*.$ Note that if $v_{p-1}\to b_r\in \mathcal{G}^*$ then by Lemma~\ref{lemma:partial cycle} we have $X\to b_r\in \mathcal{G}^*,$ which leads to a contradiction with $b_r\in sib(X,\mathcal{G}^*).$ So $v_{p-1}\to b_{k+1}\leftarrow v_{p+1}$ does not exist in $\mathcal{G}^*.$ 

If $b_k\to b_{k+1}\in \mathcal{G}^*,$ since $s(b_k,b_r)$ is the shortest causal path in $\mathcal{G}$ from $b_k$ to $b_r,$ by Lemma~\ref{lemma: perkovic} we know $s(b_k,b_r)$ is causal in $\mathcal{G}^*.$ Then $\pi(X,v_{p-1})\oplus \langle v_{p-1},b_k\rangle \oplus s(b_k,b_r)$ is causal in $\mathcal{G}^*,$ and by Lemma~\ref{lemma:partial cycle} we have $X\to b_r\in \mathcal{G}^*,$ which leads to a contradiction with $b_r\in sib(X,\mathcal{G}^*).$ So $b_k-b_{k+1}\in \mathcal{G}^*.$

Since $v_{p-1}\to b_k\leftarrow v_{p+1}$ and $b_k-b_{k+1}\in \mathcal{G}^*,$ we know $(v_{p-1},b_{k+1}),(v_{p+1},b_{k+1})$ are adjacent pairs. Since $v_{p-1}\to b_k\to b_{k+1}\in \mathcal{G}$ and $v_{p+1}\to b_k\to b_{k+1}\in \mathcal{G},$ we have $v_{p-1}\to b_{k+1}\in \mathcal{G}$ and $v_{p+1}\to b_{k+1}\in \mathcal{G},$ so $v_{p-1}\not\in ch(b_{k+1},\mathcal{G}^*)$ and $v_{p+1}\not\in ch(b_{k+1},\mathcal{G}^*).$ By the selection of $b_k$ we know $v_{p-1}\to b_{k+1}\leftarrow v_{p+1}$ does not exist in $\mathcal{G}^*.$ If $v_{p-1}-b_{k+1}-v_{p+1}\in \mathcal{G}^*,$ it contradicts with Meek's rule R3 since $v_{p-1}$ and $v_{p+1}$ are not adjacent. If $v_{p-1}\to b_{k+1}-v_{p+1}\in \mathcal{G}^*$ or $v_{p-1}-b_{k+1}\leftarrow v_{p+1}\in \mathcal{G}^*,$ it contradicts with Meek's rule R1 since $v_{p-1}$ and $v_{p+1}$ are not adjacent. Since all cases lead to a contradiction, we have ended the proof.
\end{proof}

\begin{lemma}
\label{lemma:explicit only if pa or sib}
Let $X,Y$ be two distinct nodes that are not adjacent in an MPDAG $\mathcal{G}^*.$ Suppose $X$ is not an explicit cause of $Y$ in $\mathcal{G}^*.$ Let $\mathcal{G}$ by any DAG in $[\mathcal{G}^*].$ If there exists a path $\pi=\langle X=v_0,v_1,\dots,v_n,v_{n+1}=Y\rangle$ in $\mathcal{G}$ from $X$ to $Y$ such that $v_1\in pa(X,\mathcal{G}^*)\cup sib(X,\mathcal{G}^*),$ then $\pi$ is blocked by $pa(X,\mathcal{G}^*)\cup sib(X,\mathcal{G}^*).$ 
\end{lemma}
\begin{proof}
If $v_1\in pa(X,\mathcal{G}^*),$ then $v_1\to X\in \mathcal{G}$ so $v_1$ is not a collider on $\pi$ in $\mathcal{G}.$ So $\pi$ is blocked by $pa(X,\mathcal{G}^*)\cup sib(X,\mathcal{G}^*).$

If $v_1\in sib(X,\mathcal{G}^*),$ we just need to consider the case that $v_1$ is a collider on $\pi$, otherwise $\pi$ is blocked by $pa(X,\mathcal{G}^*)\cup sib(X,\mathcal{G}^*).$  Then $X$ and $v_2$ are adjacent since $(X,v_1,v_2)$ is a collider in $\mathcal{G}$ and $X-v_1\in \mathcal{G}^*.$ If $v_2\in pa(X,\mathcal{G}^*)\cup sib(X,\mathcal{G}^*),$ then $\pi$ is blocked by $pa(X,\mathcal{G}^*)\cup sib(X,\mathcal{G}^*)$ since $v_2$ is a non-collider on $\pi.$ Suppose that $v_2\in ch(X,\mathcal{G}^*).$ Let $s=\langle X,v_2\rangle \oplus \pi(v_2,Y).$ Since $v_1\in sib(X,\mathcal{G}^*), v_2\in ch(X,\mathcal{G}^*)$ and $v_2\to v_1\in \mathcal{G},$ $v_2$ cannot block $s$ whether it is a collider on $s$ or it is a non-collider on $s.$ For other nodes except $v_2$ on $s,$ they are collider on $s$ if and only if they are collider on $\pi.$ Since $s$ is a path from $X$ to $Y$ with $v_2\in ch(X,\mathcal{G}^*),$ from (2) we know $s$ is blocked by $pa(X,\mathcal{G}^*)\cup sib(X,\mathcal{G}^*),$ so $\pi$ is also blocked by $pa(X,\mathcal{G}^*)\cup sib(X,\mathcal{G}^*).$
\end{proof}

\subsection{Proof for Theorem~\ref{thm:implicit cause}}
To prove Theorem~\ref{thm:implicit cause}, we need the following lemma:
\begin{lemma}
\label{lemma: maxclique is pa}
Suppose that $\mathcal{G}^*$ is an MPDAG, $X$ is a node in $\mathcal{G}^*,$ and $\mathbf{Q}$ is a maximal clique in the induced subgraph of $\mathcal{G}^*$ over $sib(X,\mathcal{G}^*).$ Then there exists a DAG $G$ represented by $\mathcal{G}^*,$ such that $pa(X,G)=pa(X,\mathcal{G}^*)\cup \mathbf{Q}$ and $ch(X,G)=ch(X,\mathcal{G}^*)\cup sib(X,\mathcal{G}^*) \setminus \mathbf{Q}.$
\end{lemma}
\begin{proof}
By Theorem 1 in~\citep{fang2020ida}, we just need to prove that there does not exist $S\in \mathbf{Q}$ and $C\in \mathrm{adj}(X,\mathcal{G}^*)\setminus (\mathbf{Q}\cup pa(X,\mathcal{G}^*))$ such that $C\to S.$ If such $S,C$ exists, by Lemma B.1 in~\citep{zuo2022counterfactual}, $\mathbf{Q}\cup \{C\}$ induces a complete subgraph of $\mathcal{G}^*$ and $C\in sib(X,\mathcal{G}^*)\setminus \mathbf{Q}.$ That means $\mathbf{Q}\cup \{C\}$ induces a complete subgraph in the induced subgraph of $\mathcal{G}^*$ over $sib(X,\mathcal{G}^*),$ which contradicts with that $\mathbf{Q}$ is maximal.
\end{proof}

Then we prove Theorem~\ref{thm:implicit cause}:
\begin{proof}
Let $\mathbf{C}$ be the critical set of $X$ with respect to $Y$ in $\mathcal{G}^*.$ Suppose that $X$ is an implicit cause of $Y,$ then by Theorem~\ref{thm:explicit cause}, $X\perp Y|pa(X,\mathcal{G}^*)\cup sib(X,\mathcal{G}^*).$ For any $\mathbf{Q}\in \mathcal{Q},$ from Theorem 4.5 in~\citep{zuo2022counterfactual} we know that $\mathbf{C}\setminus \mathbf{Q}\neq \emptyset,$ otherwise $\mathbf{C}$ induces a complete graph. So there is a chordless partially directed path in $\mathcal{G}^*$ from $X$ to $Y,$ denoted by $\pi=\langle X,v_1,v_2,\dots,v_n,Y\rangle,$ such that $v_1\not\in \mathbf{Q}.$ Let $G$ be a DAG represented by $\mathcal{G}^*$ such that $X\to v_1\in G.$ There is no collider on $\pi$ in $G,$ otherwise since $\pi$ is chordless, that collider is a v-structure in $\mathcal{G}^*,$ which contradicts with that $\pi$ is partially directed in $\mathcal{G}^*.$ Therefore, $\pi$ is causal in $G.$ Since $\pi$ is chordless, $v_2,v_3,\dots,v_n,Y$ are not adjacent with $X,$ so they are not in $pa(X,\mathcal{G}^*)\cup \mathbf{Q}.$ $v_1\not\in pa(X,\mathcal{G}^*)$ since $\pi$ is partially directed, and $v_1\not\in \mathbf{Q}$ by the selection of $\pi.$ So $\pi$ is not blocked by $pa(X,\mathcal{G}^*)\cup \mathbf{Q}$ in $G.$ Therefore, $X\not\perp Y|pa(X,\mathcal{G}^*)\cup \mathbf{Q}.$

Conversely, suppose $X\perp Y| pa(X,\mathcal{G}^*)\cup sib(X,\mathcal{G}^*)$ and $X\not\perp Y|pa(X,\mathcal{G}^*)\cup \mathbf{Q}$ for any $\mathbf{Q}\in \mathcal{Q}$ holds. Then by Theorem~\ref{thm:explicit cause}, $X$ is not an explicit cause of $Y.$ Assume, for the sake of contradiction, that $X$ is not an implicit cause of $Y.$ Then $X$ is not a definite cause of $Y.$ By Theorem 4.5 in~\citep{zuo2022counterfactual}, $\mathbf{C}\cup ch(X,\mathcal{G}^*)=\emptyset,$ and either $\mathbf{C}=\emptyset$ or $\mathbf{C}$ induces an complete subgraph of $G.$ Both cases implies that there exists $\mathbf{Q}\in\mathcal{Q}$ such that $\mathbf{C}\subseteq \mathbf{Q}.$ We will show that $X\perp Y|pa(X,\mathcal{G}^*)\cup \mathbf{Q},$ which leads to a contradiction. By Lemma~\ref{lemma: maxclique is pa}, there exists a DAG $G$ represented by $\mathcal{G}^*$ such that $pa(X,G)=pa(X,\mathcal{G}^*)\cup \mathbf{Q}$ and $ch(X,G)=ch(X,\mathcal{G}^*)\cup sib(X,\mathcal{G}^*)\setminus \mathbf{Q}.$ So $\mathbf{C}\subset pa(X,G).$ By Lemma 2 in~\citep{fang2020ida}, $X\not\in an(Y,G).$ So by the Markov property, $X\perp Y|pa(X,\mathcal{G}^*)\cup \mathbf{Q}.$ That contradicts with the assumption. 

\end{proof}

\section{Examples}
\label{append:example}

\begin{figure}[t]
\centering
\subfloat[$\mathcal{G}$]
{
\label{fig_eg:subfig1}\includegraphics[width=0.25\textwidth]{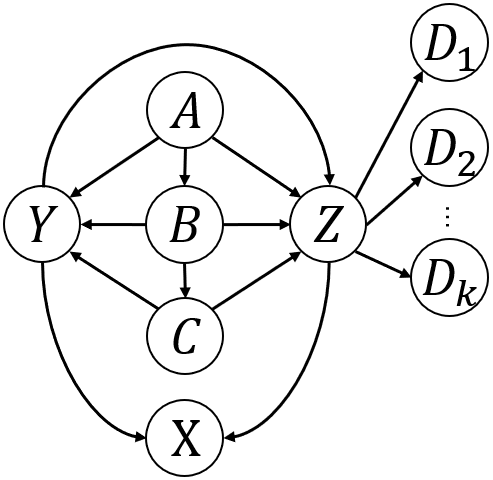}
}
\subfloat[$G_{(X)}$]
{
\label{fig_eg:subfig2}\includegraphics[width=0.25\textwidth]{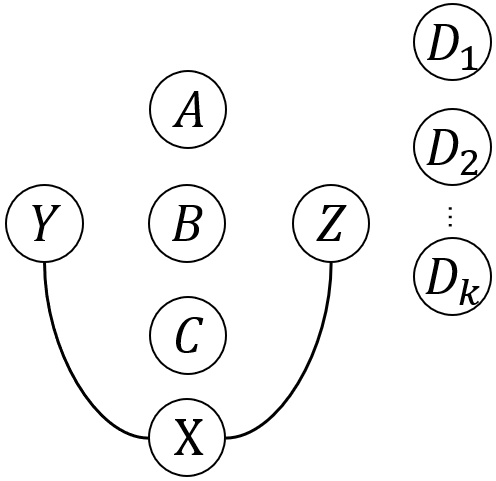}
}
\subfloat[$G_{(X,Y)}$]
{
\label{fig_eg:subfig3}\includegraphics[width=0.25\textwidth]{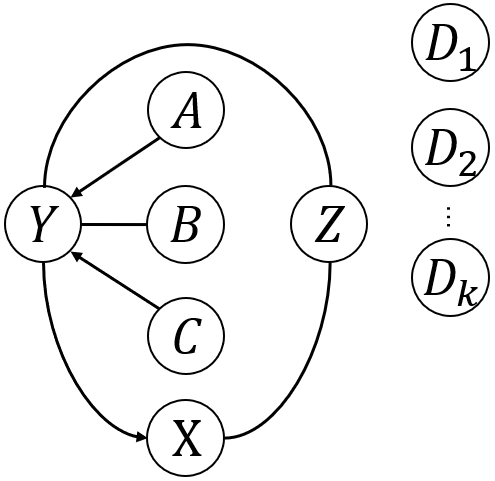}
}
\\
\subfloat[$G_{(X,Y,Z)}$]
{
\label{fig_eg:subfig4}\includegraphics[width=0.25\textwidth]{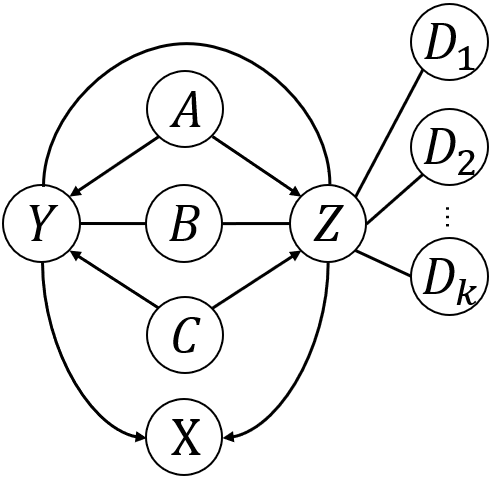}
}
\subfloat[$G^{A\to Y}_{(X)}$]
{
\label{fig_eg:subfig5}\includegraphics[width=0.25\textwidth]{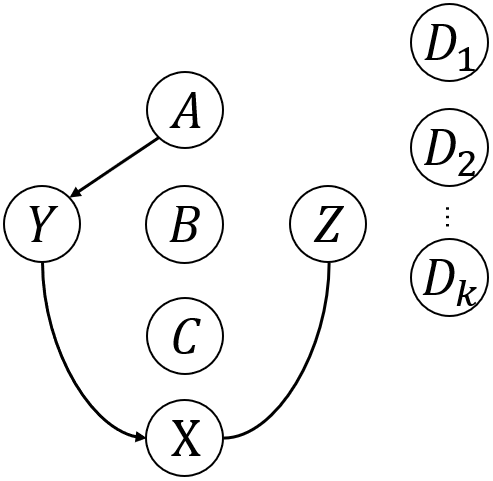}
}
\subfloat[$G^{A\to Y}_{(X,Z)}$]
{
\label{fig_eg:subfig6}\includegraphics[width=0.25\textwidth]{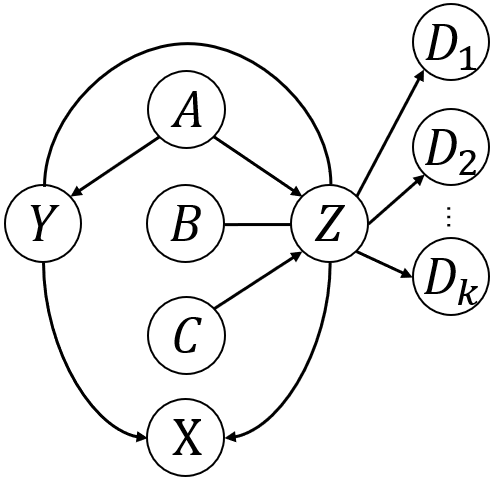}
}
\caption{An example for learning local structure. (a) The original DAG. (b)-(f) Learned structure $G$ defined at Step $3$ in Algorithm~\ref{alg:mb-by-mb in MPDAG}. $G_{D}^{\mathcal{B}}$ denote the learned graph after running Step $4$-$16$ in Algorithm~\ref{alg:mb-by-mb in MPDAG} for each node in $D$, with background knowledge $\mathcal{B}$.} 
\label{fig:eg}
\end{figure}

In this section, we use a more complicated example to demonstrate that how Algorithm~\ref{alg:mb-by-mb in MPDAG} is performed, and in which situation Algorithm~\ref{alg:mb-by-mb in MPDAG} needs more conditional independence tests than the baseline method. 

Suppose that the true DAG $\mathcal{G}$ is plotted by Figure~\ref{fig_eg:subfig1}. We want to learn the local structure around $X.$ Firstly, assume that there is no prior knowledge. In this way, Algorithm~\ref{alg:mb-by-mb in MPDAG} coincides with the baseline method. As following Step $4$-$16$ in Algorithm~\ref{alg:mb-by-mb in MPDAG}, we first find the MB of $X,$ which is $\mathrm{MB}(X)=\{X,Y,Z\}.$ Since any two nodes in $\{X,Y,Z\}$ are not independent given any subset of $\{X,Y,Z\}$ except themselves, we conclude that $X$ and $Y,Z$ are adjacent in $\mathcal{G}.$ Moreover, we do not learn any v-structure in the marginal graph since there is no independence. Therefore, we draw undirected edges $Y-X-Z$ in the learned local structure $G,$ giving Figure~\ref{fig_eg:subfig2}.

After that, we have $\mathrm{WaitList}=\{Y,Z\}$ since they are connected to $X$ by an undirected path in $G.$ Now we find the MB of $Y,$ which is $\mathrm{MB}(Y)=\{A,B,C,X,Z\}.$ By learning the marginal graph over $Y\cup \mathrm{MB}(Y),$ we find that $Y$ and $A,B,C,X,Z$ are adjacent, and we also find that $A\perp C|B,$ which implies that $A\to Y\leftarrow C$ is a v-structure in $\mathcal{G}.$ When finding $\mathrm{MB}(X)$ is previous steps, we have found that $X\perp A|Y,Z.$ Therefore, we orient $Y\to X$ in Step $10$ in Algorithm~\ref{alg:mb-by-mb in MPDAG}. Since there is no more v-structure including $Y,$ we put $Y$ into DoneList and update $\mathrm{WaitList}=\{Z\},$ giving Figure~\ref{fig_eg:subfig3}.

Finally, we find the MB of $Z,$ which contains all nodes except $Z.$ By learning the marginal graph, we first find that all other nodes are adjacent with $Z.$ As we found that $A\perp C|B$ previously, we directly know that $A\to Z\leftarrow C$ is a v-structure in $\mathcal{G}.$ When finding $\mathrm{MB}(X)$ is previous steps, we have found that $X\perp A|Y,Z.$ Therefore, we orient $Z\to X$ in Step $10$ in Algorithm~\ref{alg:mb-by-mb in MPDAG}. After that, there is no undirected edge connected with $X,$ so the algorithm returns and gives Figure~\ref{fig_eg:subfig4}. We find that the local structure around $X$ is $Y\to X\leftarrow Z.$

Now suppose that $A\to Y$ is known as prior knowledge. In Algorithm~\ref{alg:mb-by-mb in MPDAG}, we first add this edge to an empty graph $G.$ Then we find the MB of $X,$ which is $\mathrm{MB}(X)=\{Y,Z\}.$ By learning the marginal graph over $\{X,Y,Z\},$ we find that $X$ and $Y,Z$ are adjacent. We also find that $X\perp A|Y,Z.$ Therefore, we orient $Y\to X$ in Step $10$ in Algorithm~\ref{alg:mb-by-mb in MPDAG}. Therefore, we have Figure~\ref{fig_eg:subfig5} as the current local structure.

Then we have $\mathrm{WaitList}=\{Z\}.$ We find that $\mathrm{MB}(Z)=\mathbf{V}\setminus \{Z\},$ so we need to learn the entire graph to find $adj(Z,\mathcal{G})$ and v-structures including $Z.$ The first step is to find $adj(Z,\mathcal{G}),$ which iterate over $T\in \mathrm{V}\setminus \{Z\}$ and tests whether $Z\perp T\mid \mathbf{S}$ for each $\mathbf{S}\subseteq \mathrm{V}\setminus \{Z,T\}.$ This step is also done when there is no prior knowledge. However, we do not know $A\perp C\mid B$ now, so we need to do the second step, that is, to test other conditional independencies over $\mathbf{V},$ until we find that $A\perp C\mid B,$ orient the v-structure $A\to Z\leftarrow C,$ and orient $Z\to X$ and $Z\to D_i$ for $1\le i\le k$ in Step $10$ in Algorithm~\ref{alg:mb-by-mb in MPDAG}. Suppose that $k$ is large, then we need more conditional independence tests to learn the local structure $Y\to X\leftarrow Z,$ comparing with the scenario without prior knowledge.

Although this example shows that in extreme situation we may use more conditional independence tests to learn the local structure with prior knowledge, experiments in Section~\ref{exp:local structure} show that the presence of prior knowledge reduce the number of conditional independence test needed in average.

\section{Algorithm for local structure learning with all types of background knowledge}
\label{append:algorithm}

\begin{algorithm}[t]
\LinesNumbered
\caption{Learning the local structure around $X$ given background knowledge consisting of direct causal information $\mathcal{B}_1$, non-ancestral information $\mathcal{B}_2$ and ancestral information $\mathcal{B}_3$.}
\label{alg:local struct:all types}
\SetKwInOut{Input}{Input}
\SetKwInOut{Output}{Output}
\small

\Input{Target node $X$, observational data $\mathcal{D}$, background knowledge $\mathcal{B}=\mathcal{B}_1\cup \mathcal{B}_2\cup \mathcal{B}_3$.}
\Output{A PDAG $G$ containing the local structure of $X$ and a set of direct causal clauses within $G$.}
Let $G$ be the output of MB-by-MB algorithm~\citep{wang2014discovering} under input $X,\mathcal{D},\mathcal{B}_1$.

Let $\mathrm{DCC}=\emptyset$.

\ForEach{$(F_j,T_j)\in \mathcal{B}_3$}{
\If{$F_j$ and $X$ are connected by an undirected path in $G$}{
Let $\mathbf{candC} = sib(N_j,G), \mathrm{hasCand}=\mathrm{FALSE}.$

Let $\mathcal{Q}$ be the set of all maximal cliques in the induced subgraph of $G$ over $sib(N_j,G)$.

Let $\mathcal{Q}_1$ be the set of all $\mathbf{Q}\subseteq sib(F_j,G)$ such that letting $\mathbf{Q}\to F_j$ and $F_j\to sib(F_j,G)\setminus \mathbf{Q}$ does not introduce any v-structure collided on $F_j$.

\ForEach{$\mathbf{Q}\in \mathcal{Q}$}{
\If{$F_j \perp T_j \mid pa(F_j,G)\cup \mathbf{Q}$}{
Update $\mathbf{candC}:=\mathbf{candC}\cap \mathbf{Q}.$

Update $\mathrm{hasCand}:=\mathrm{TRUE}.$
}
}
\If{$\mathrm{hasCand}$}{
\ForEach{$\mathbf{Q}\in \mathcal{Q}_1\setminus \mathcal{Q}$}{
\If{$F_j \perp T_j \mid pa(F_j,G)\cup \mathbf{Q}$}{
Update $\mathbf{candC}:=\mathbf{candC}\cap \mathbf{Q}.$
}
}
Add $F_j \overset{or}{\to} \mathbf{candC}$ into $\mathrm{DCC}$.
}
}
}

\ForEach{$(F_i,T_i)\in \mathcal{B}_1$}{
\If{$F_j$ and $X$ are connected by an undirected path in $G$}{
Orient $F_i\to T_i$ and apply Meek's rules on $G$.
}
}

Run Step 2 to 12 of Algorithm~\ref{alg:local struct:direct and non-an}.

\Return{G, DCC}

\end{algorithm}

By Theorem~\ref{thm:correct alg non-an}, background knowledge with direct causal information and non-ancestral information can be utilized for learning local structure of any MPDAG by Algorithm~\ref{alg:local struct:direct and non-an}. Ancestral information that $F_j$ is a cause of $T_j$ is equivalent to that there exists $C\in \mathbf{C}_{\mathcal{G}^*}(F_j,T_j)$ such that $F_j\to C$ in the true underlying DAG $\mathcal{G},$ where $\mathcal{G}^*$ is an MPDAG representing a restricted Markov equivalence class $[\mathcal{G}^*]$ containing $\mathcal{G},$ and $\mathbf{C}_{\mathcal{G}^*}(F_j,T_j)$ is the critical set of $F_j$ with respect to $T_j$ in $\mathcal{G}^*.$ With theories of direct causal clause (DCC,~\citep{fang2022representation}), we rewrite this fact as $F_j\overset{or}{\to} \mathbf{C}_{\mathcal{G}^*}(F_j,T_j).$

\section{Additional experimental results}
\label{append:experiment}

\begin{figure}[t!]
	\centering
	
	\subfloat[$n=100$, $N=100$ \label{fig:kappa_100_100}]{
		\begin{minipage}[t]{0.22\textwidth}
			\centering
			\includegraphics[width=\textwidth]{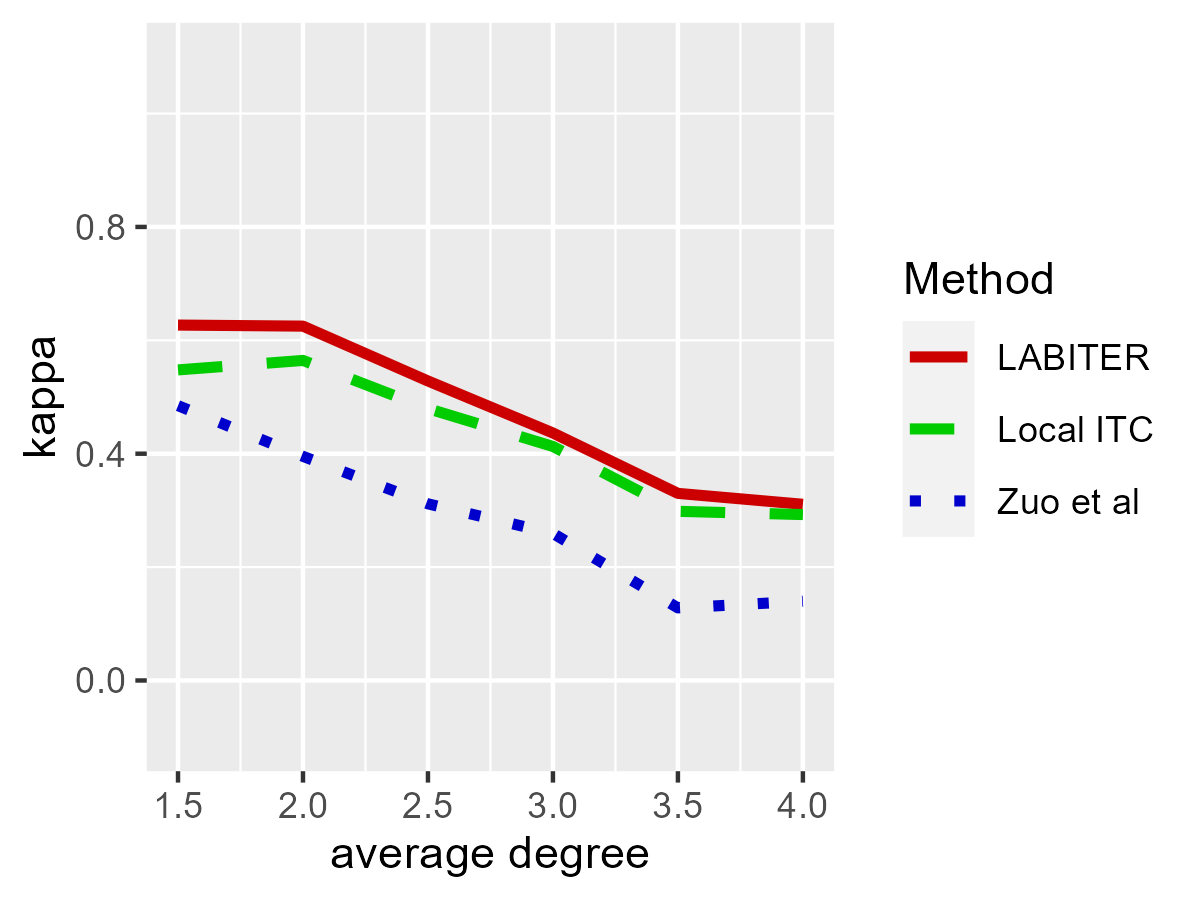}
		\end{minipage}%
	}%
	\hspace{0.01\textwidth}
	\subfloat[$n=100$, $N=200$  \label{fig:kappa_100_200}]{
		\begin{minipage}[t]{0.22\textwidth}
			\centering
			\includegraphics[width=\textwidth]{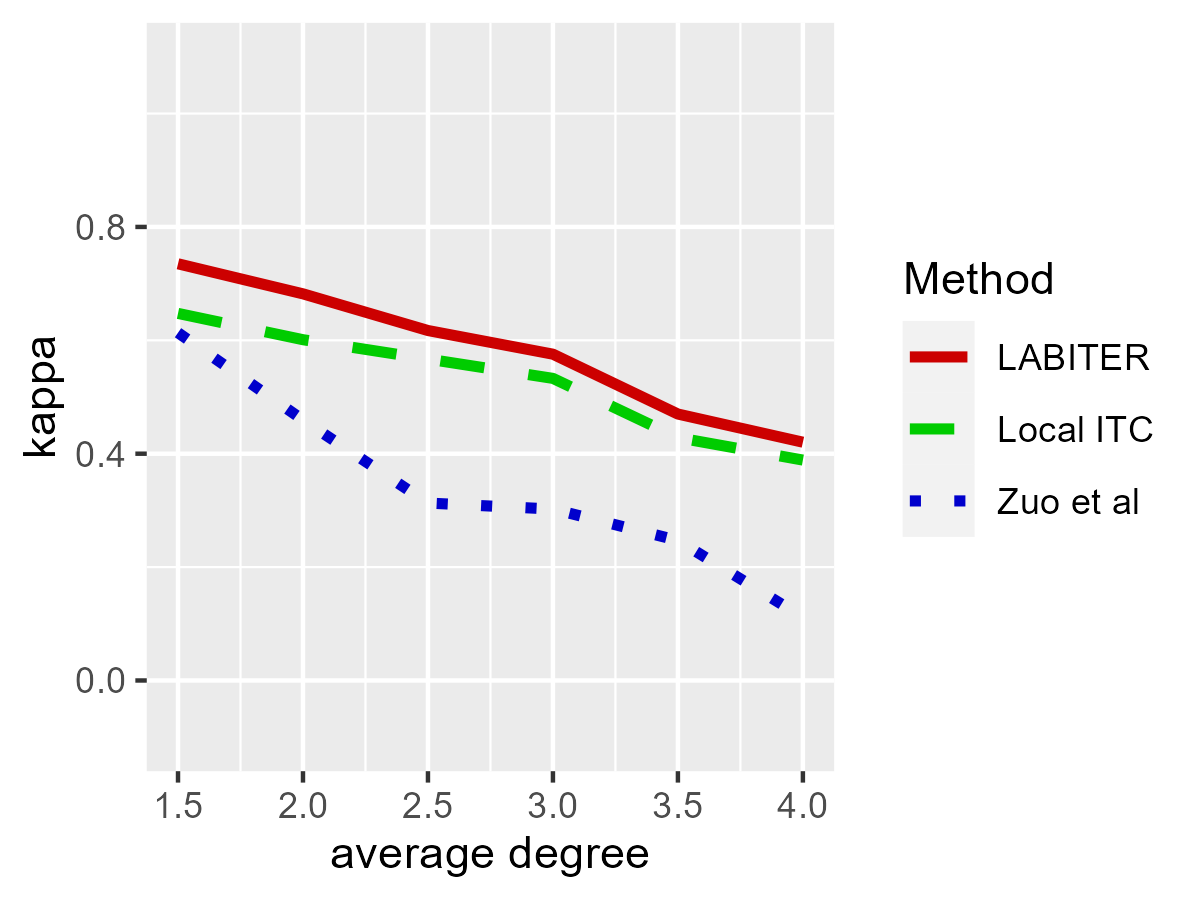}
		\end{minipage}%
	}%
	\hspace{0.01\textwidth}
	\subfloat[$n=100$, $N=500$ \label{fig:kappa_100_500}]{
		\begin{minipage}[t]{0.22\textwidth}
			\centering
			\includegraphics[width=\textwidth]{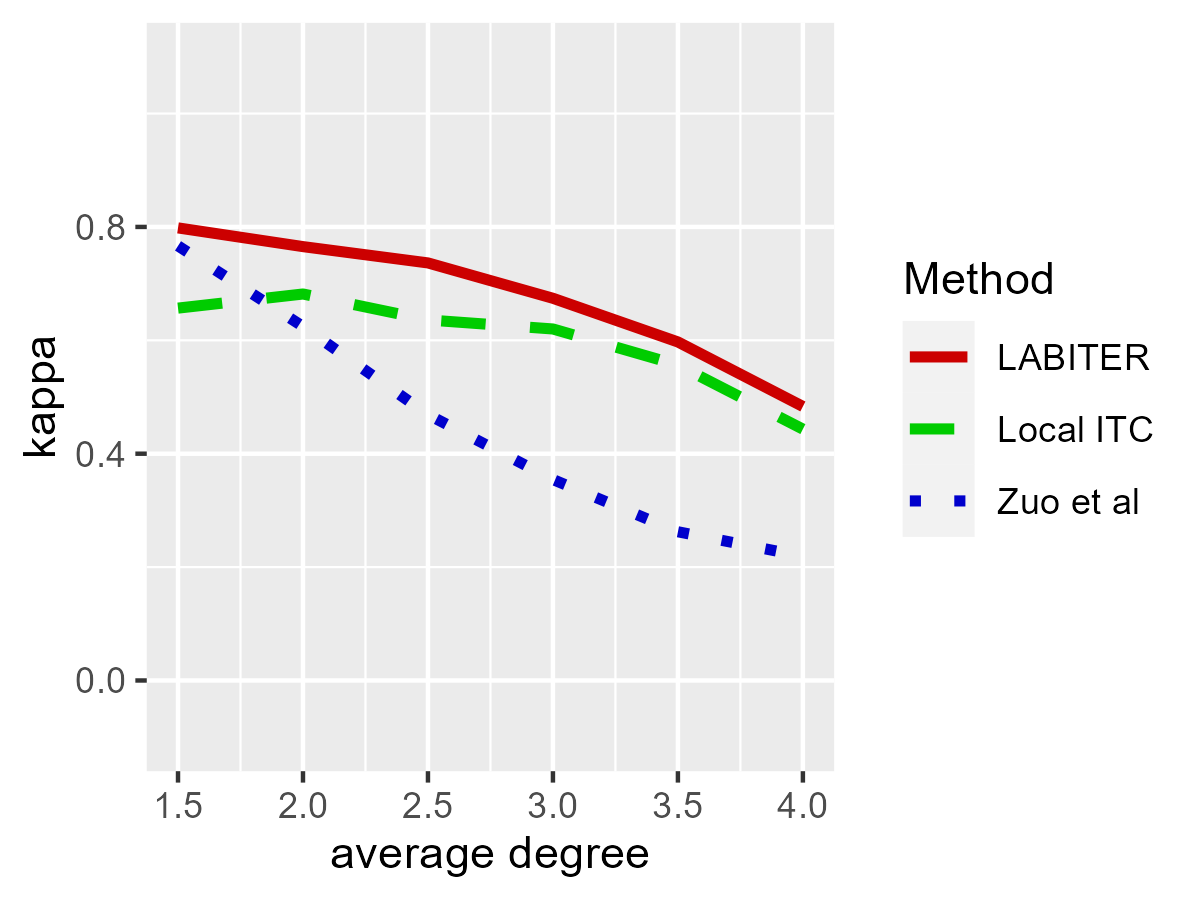}
		\end{minipage}%
	}%
	\hspace{0.01\textwidth}
	\subfloat[$n=100$, $N=1000$ \label{fig:kappa_100_1000}]{
		\begin{minipage}[t]{0.22\textwidth}
			\centering
			\includegraphics[width=\textwidth]{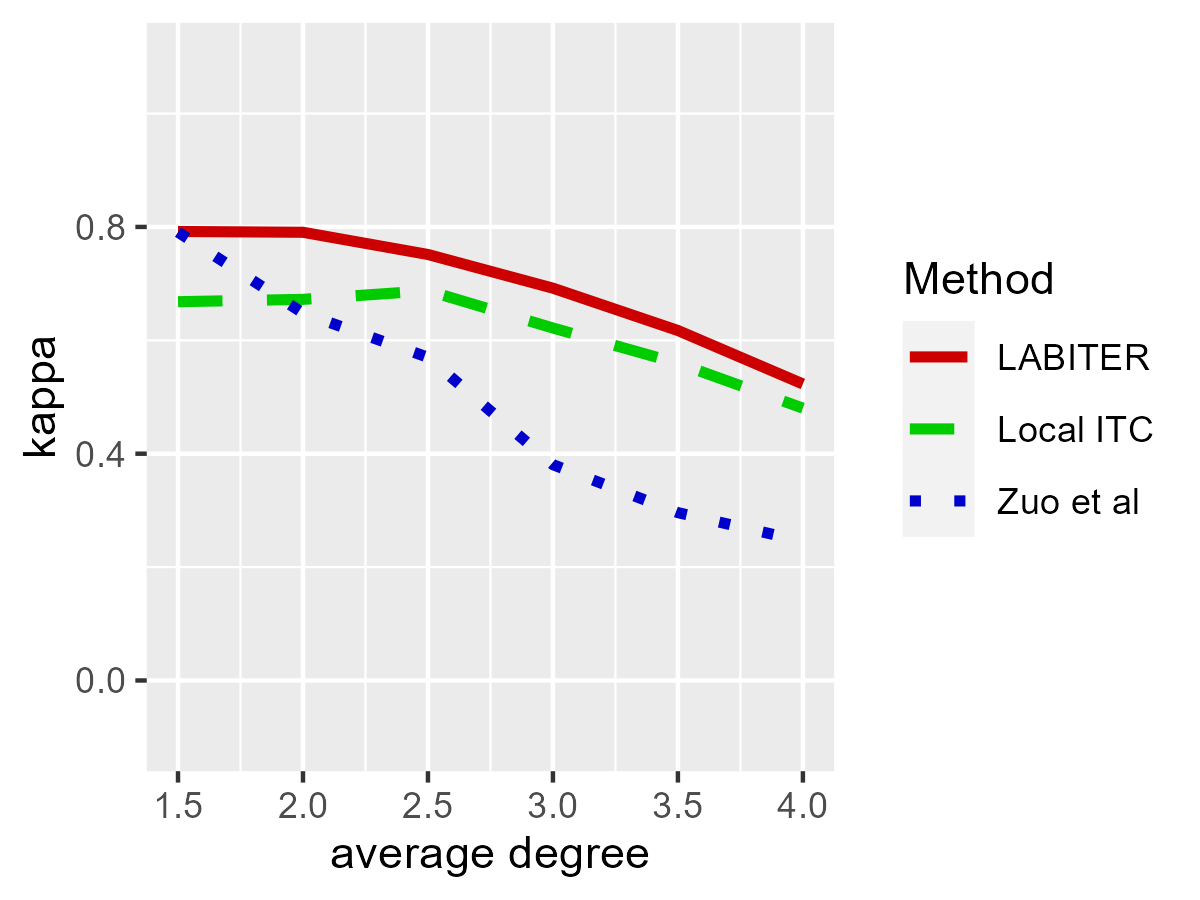}
		\end{minipage}%
	}%

    \subfloat[$n=100$, $N=100$ \label{fig:time_100_100}]{
		\begin{minipage}[t]{0.22\textwidth}
			\centering
			\includegraphics[width=\textwidth]{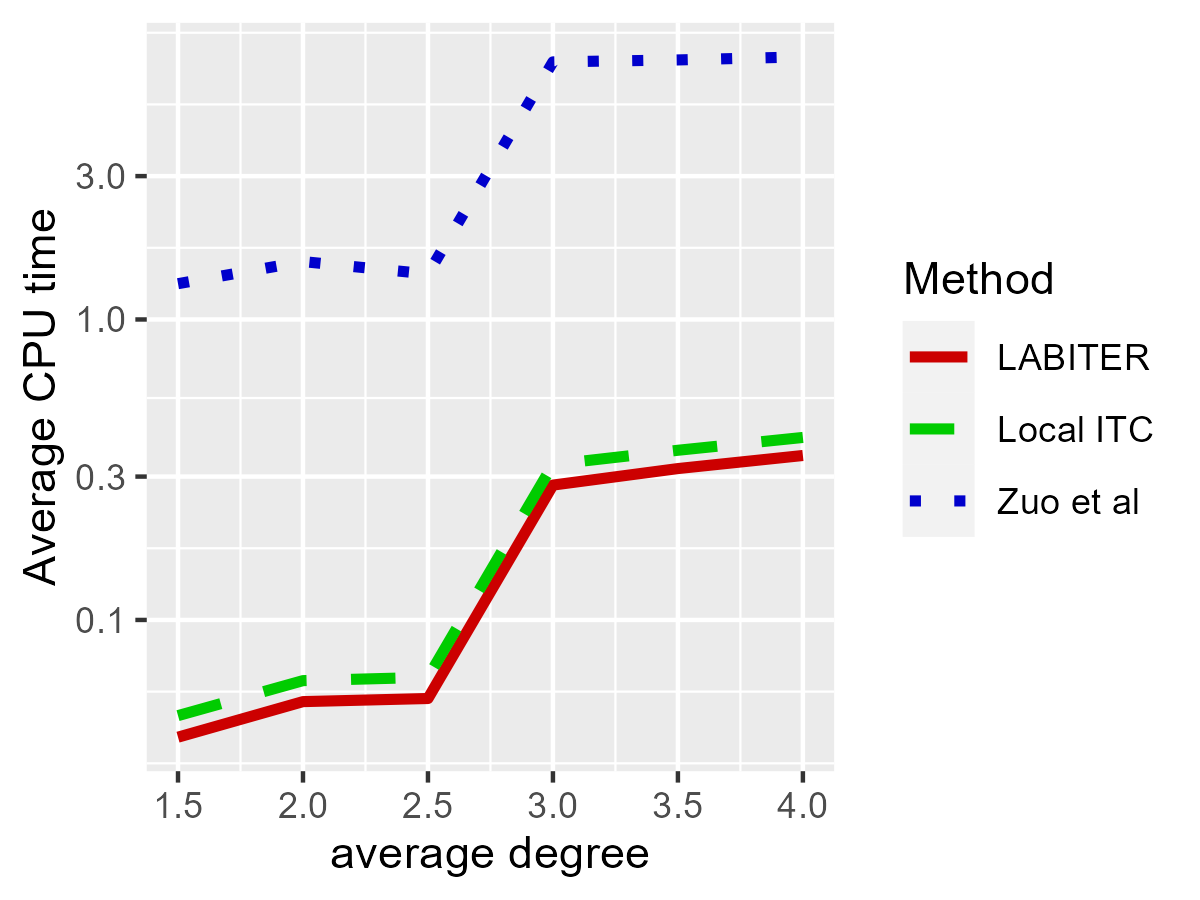}
		\end{minipage}%
	}%
	\hspace{0.01\textwidth}
	\subfloat[$n=100$, $N=200$  \label{fig:time_100_200}]{
		\begin{minipage}[t]{0.22\textwidth}
			\centering
			\includegraphics[width=\textwidth]{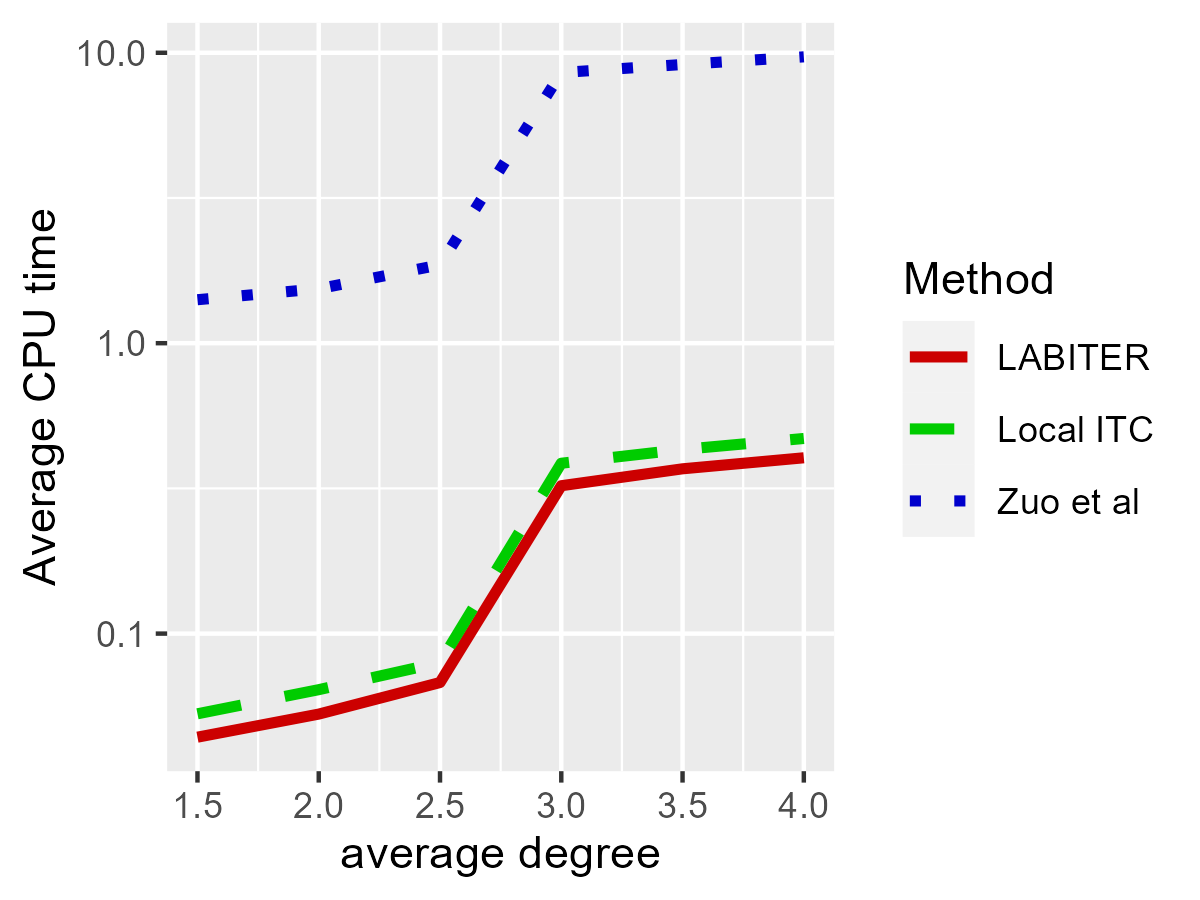}
		\end{minipage}%
	}%
	\hspace{0.01\textwidth}
	\subfloat[$n=100$, $N=500$ \label{fig:time_100_500}]{
		\begin{minipage}[t]{0.22\textwidth}
			\centering
			\includegraphics[width=\textwidth]{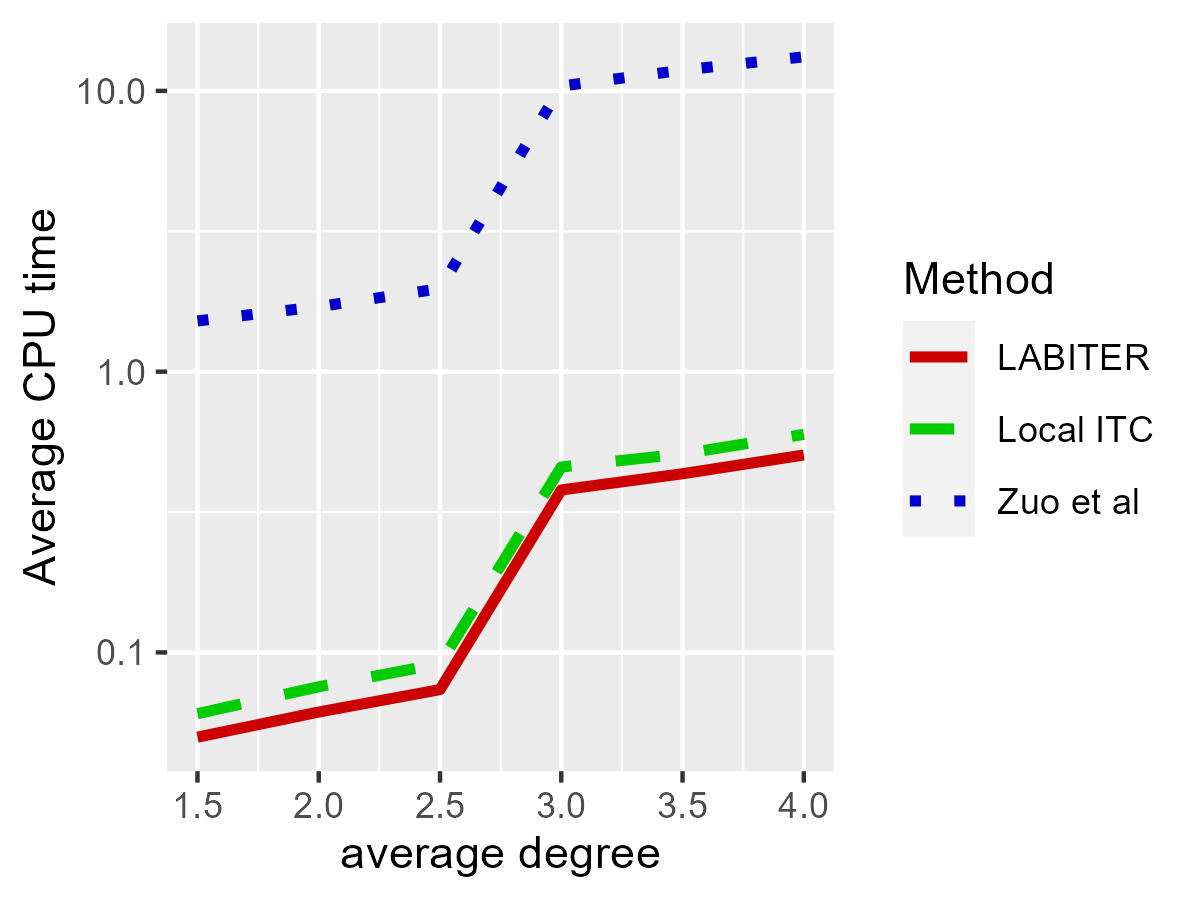}
		\end{minipage}%
	}%
	\hspace{0.01\textwidth}
	\subfloat[$n=100$, $N=1000$ \label{fig:time_100_1000}]{
		\begin{minipage}[t]{0.22\textwidth}
			\centering
			\includegraphics[width=\textwidth]{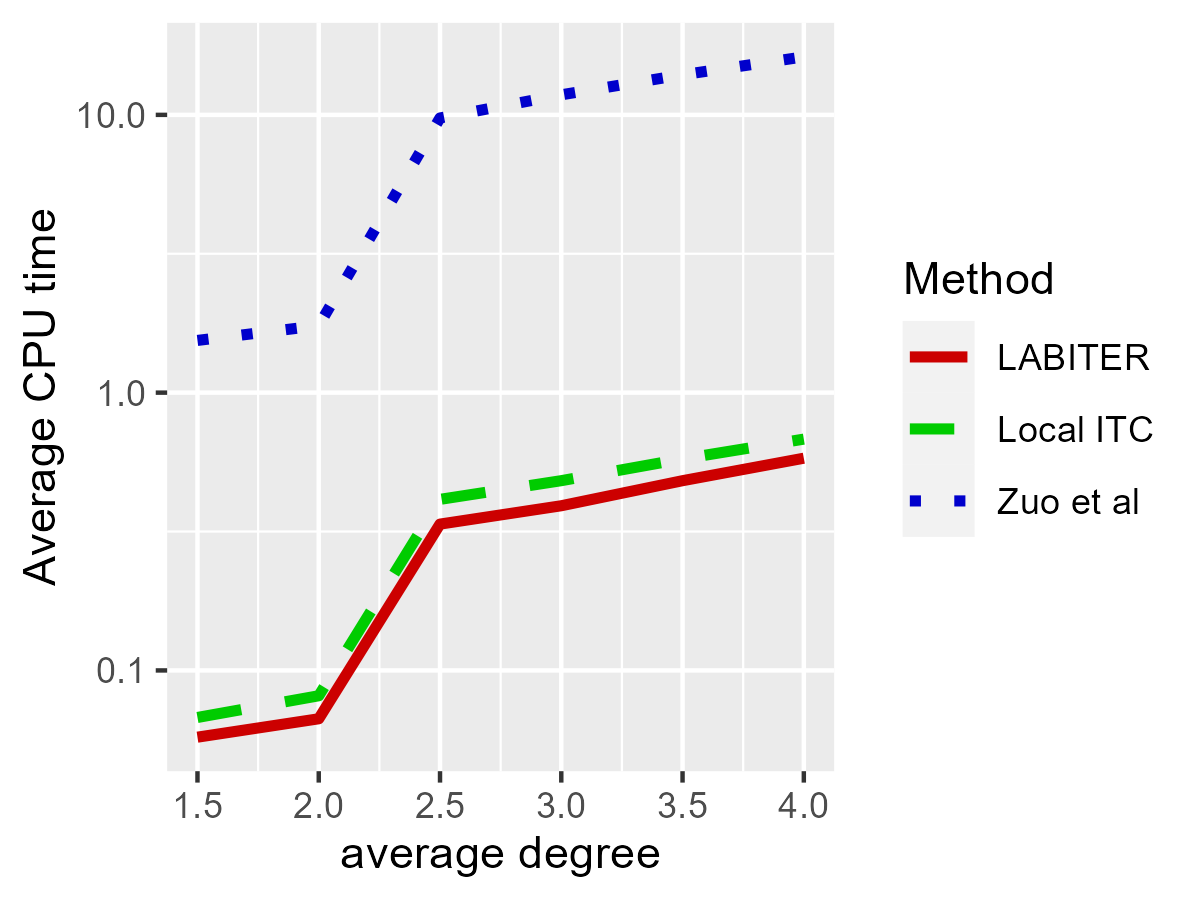}
		\end{minipage}%
	}%
	
	\caption{(a)-(d) The Kappa coefficients of different methods on random graphs with $n$ nodes and $N$ samples are drawn. (e)-(h) The average CPU time of different methods. }
	\label{fig:true:kappa_n100}
\end{figure}

\end{document}